%% file: iclr2027_conference.tex
\pdfoutput=1
\documentclass{article} 
\usepackage{iclr2027_conference,times}
\usepackage{xcolor}

\input{math_commands.tex}

\usepackage{hyperref}
\usepackage{url}
\usepackage{amsmath}
\usepackage{amssymb}
\usepackage{amsthm}
\newtheorem{theorem}{Theorem}

\newtheorem{corollary}{Corollary}

\newtheorem{assumption}{Assumption}
\usepackage{booktabs}
\usepackage{graphicx}
\usepackage{float}
\usepackage{placeins}
\usepackage{subfig}
\usepackage{tabularx}
\usepackage{multirow}
\usepackage{algorithm}
\usepackage{algorithmic}
\usepackage{tikz}
\usepackage{enumitem}
\usepackage{longtable}
\usepackage{wrapfig}

\definecolor{keystepfill}{HTML}{E1E9FB}
\newcommand{\keystep}[1]{{\setlength{\fboxsep}{1.5pt}\colorbox{keystepfill}{\parbox[t]{\dimexpr\linewidth-2\fboxsep\relax}{\strut#1\strut}}\par}}
\newcommand{\codeus}{\char95{}}
\newcommand{\codeplain}[1]{\noindent\makebox[\linewidth][l]{\hspace*{2pt}\strut#1}\par}
\newcommand{\codekey}[1]{\noindent\colorbox{keystepfill}{\makebox[\linewidth][l]{\hspace*{2pt}\strut#1}}\par}
\newcommand{\codeblockstart}{\par\smallskip\begingroup\footnotesize\ttfamily\setlength{\fboxsep}{0pt}\setlength{\parskip}{0pt}\setlength{\lineskiplimit}{-\maxdimen}}
\newcommand{\codeblockend}{\endgroup\smallskip}
\newcommand{\na}{--}

\usetikzlibrary{arrows.meta,positioning,fit,backgrounds,calc}

\title{Network World Models as Environments for\\Algorithm Design on Complex Systems}

\author{
\textbf{Rishab Alagharu$^{1,*}$} \quad
\textbf{Hongji Pu$^{2,*}$} \quad
\textbf{Zeeshan Memon$^{1,*}$} \\
\textbf{Xinyuan Song$^{1}$} \quad
\textbf{Yuntong Hu$^{1}$} \quad
\textbf{Liang Zhao$^{1,\dagger}$} \\[3pt]
$^{1}$Emory University
\quad
$^{2}$University of Illinois Urbana-Champaign \\[3pt]
\texttt{27ralagharu@woodward.edu},
\texttt{hongjip2@illinois.edu} \\
\texttt{\{zeeshan.memon, xinyuan.song, yuntong.hu, liang.zhao\}@emory.edu} \\[3pt]
$^{*}$Equal contribution.
\qquad
$^{\dagger}$Corresponding author.
}

\iclrfinalcopy

\begin{document}

\setcounter{footnote}{1}
\maketitle
\lhead{Preprint}
\vspace{-0.2in}

\begin{abstract}
\input{sections/0_abstract}
\end{abstract}

\input{figures/headline_figure}

\input{sections/1_introduction}
\input{sections/2_related_work}
\input{sections/3_methodology}
\input{sections/4_experiments}
\input{sections/5_conclusion}

\subsection*{AI use statement}

\input{sections/ai_use_statement}

\subsection*{Ethics statement}

\input{sections/ethics_statement}

\subsection*{Reproducibility statement}

\input{sections/reproducibility_statement}



\bibliography{iclr2027_conference}
\bibliographystyle{iclr2027_conference}

\appendix
\input{sections/appendix}

\end{document}

%% file: math_commands.tex
\usepackage{amsmath,amsfonts,bm}

\def\eqref#1{equation~\ref{#1}}

\def\1{\bm{1}}

\def\ro{{\textnormal{o}}}

\def\rr{{\textnormal{r}}}

\DeclareMathAlphabet{\mathsfit}{\encodingdefault}{\sfdefault}{m}{sl}
\SetMathAlphabet{\mathsfit}{bold}{\encodingdefault}{\sfdefault}{bx}{n}

\newcommand{\R}{\mathbb{R}}



%% file: sections/0_abstract.tex
World models, which simulate an environment and predict how it changes under actions, are increasingly used in real-world applications such as robotics. Complex systems call for the same tool because the effect of an action is not immediate. Seeding nodes for a campaign, or immunizing nodes against an epidemic, changes little on its own; what matters is the outcome that unfolds over the steps that follow. Designing an algorithm that selects such actions to maximize expected performance on a task is inherently iterative, and every candidate must be scored by the outcome it produces. Obtaining that outcome has relied on simulation, whose cost becomes a bottleneck when candidates are evaluated over many sampled trajectories. We propose an action-conditioned Network World Model that learns a network's diffusion dynamics under interventions over time, applies each action to the network, and predicts the outcome that follows. It serves as a fast evaluator inside an algorithm design loop in which a coding agent designs and refines executable algorithms using feedback from full rollouts, action-level credit, and counterfactual probes over alternative interventions. Across eight network tasks and five diffusion models, the designed algorithms match or exceed the strongest reported baseline in 138 of 141 settings while enabling up to 14.5 times faster rollouts than Monte Carlo simulation. Code will be released upon acceptance





%% file: figures/headline_figure.tex
\begin{figure}[H]
\definecolor{skyblue}{RGB}{66,133,244}
\definecolor{amber}{RGB}{242,169,0}
\definecolor{leaf}{RGB}{52,168,83}
\definecolor{heat}{RGB}{234,67,53}
\newsavebox{\headlinekey}
\sbox{\headlinekey}{\begin{tikzpicture}[key/.style={font=\sffamily\fontsize{6.0pt}{6.8pt}\selectfont, text=black!62, inner sep=1pt}]
\def\keygap{0.24}
\coordinate (k0) at (0,0);
\draw[leaf!75!black, line width=1.8pt] (k0) -- ++(0.38,0);
\draw[white, line width=0.7pt, fill=leaf!75!black] ($(k0)+(0.19,0)$) circle (0.1);
\node[key, anchor=west] (k1) at ($(k0)+(0.42,0)$) {\textbf{ours}};
\draw[black!70, line width=1pt, dash pattern=on 2.6pt off 1.7pt] ($(k1.east)+(\keygap,0)$) -- ++(0.4,0);
\node[key, anchor=west] (k2) at ($(k1.east)+(\keygap+0.44,0)$) {strongest baseline};
\fill[leaf, opacity=0.5] ($(k2.east)+(\keygap,-0.11)$) rectangle ++(0.3,0.22);
\node[key, anchor=west, inner xsep=0pt] (k3) at ($(k2.east)+(\keygap+0.34,0)$) {our gain};
\end{tikzpicture}}
\centering
\begin{tikzpicture}[
  x=1cm, y=1cm,
  every node/.style={font=\sffamily},
  gnode/.style={circle, draw=black!40, line width=0.35pt, fill=white, minimum size=2.2mm, inner sep=0pt},
  gedge/.style={draw=black!30, line width=0.5pt, line cap=round},
  title/.style={font=\sffamily\bfseries\fontsize{9pt}{10pt}\selectfont, anchor=west, inner sep=0pt},
  body/.style={font=\sffamily\fontsize{7.2pt}{8.6pt}\selectfont, text=black!72, inner sep=1pt, align=center},
  small/.style={font=\sffamily\fontsize{6.6pt}{7.6pt}\selectfont, text=black!62, inner sep=1pt},
  steplabel/.style={font=\sffamily\fontsize{7pt}{7.6pt}\selectfont, text=skyblue!75!black, inner sep=1pt},
  spoke/.style={font=\sffamily\fontsize{5.3pt}{6.1pt}\selectfont, text=black!80, inner sep=1.5pt},
  code/.style={draw=black!15, fill=white, rounded corners=2pt, inner sep=4.5pt, align=left,
               font=\ttfamily\fontsize{6.7pt}{9.4pt}\selectfont, text=black!82},
  flow/.style={-{Stealth[length=2.6mm, width=2.2mm]}, line width=1.5pt, line cap=round},
  chip/.style={rounded corners=2.5pt, inner sep=2.6pt, font=\sffamily\bfseries\fontsize{7.2pt}{8pt}\selectfont},
]

\newcommand{\panel}[7]{
  \fill[#5, rounded corners=6pt] (#1,#2) rectangle (#3,#4);
  \node[title, anchor=base west, text=#6] at (#1+0.8,#4-0.47) {#7};
}

\newcommand{\robot}[3]{
  \begin{scope}[shift={(#1,#2-0.025)}, line cap=round, line join=round]
    \draw[#3, line width=0.8pt] (0,0.14) -- (0,0.21);
    \fill[#3] (0,0.235) circle (0.035);
    \fill[#3] (-0.25,-0.07) rectangle (-0.2,0.05);
    \fill[#3] (0.2,-0.07) rectangle (0.25,0.05);
    \draw[#3, line width=0.85pt, fill=white, rounded corners=1.8pt] (-0.2,-0.17) rectangle (0.2,0.14);
    \fill[#3] (-0.085,0.025) circle (0.042);
    \fill[#3] (0.085,0.025) circle (0.042);
    \draw[#3, line width=0.8pt] (-0.075,-0.09) -- (0.075,-0.09);
  \end{scope}
}
\newcommand{\neuralnet}[3]{
  \begin{scope}[shift={(#1,#2)}, line cap=round]
    \foreach \ya in {0.1, -0.1} \foreach \yb in {0.18, 0, -0.18}
      \draw[#3, line width=0.45pt] (-0.2,\ya) -- (0,\yb);
    \foreach \yb in {0.18, 0, -0.18} \foreach \yc in {0.1, -0.1}
      \draw[#3, line width=0.45pt] (0,\yb) -- (0.2,\yc);
    \foreach \px/\py in {-0.2/0.1, -0.2/-0.1, 0/0.18, 0/0, 0/-0.18, 0.2/0.1, 0.2/-0.1}
      \draw[#3, line width=0.7pt, fill=white] (\px,\py) circle (0.05);
  \end{scope}
}

\newcommand{\minigraph}[5]{
  \foreach \n/\px/\py in {n1/0.33/0.89, n2/0.80/0.80, n3/0.15/0.66, n4/0.57/0.61,
                          n5/0.84/0.45, n6/0.34/0.39, n7/0.69/0.21, n8/0.20/0.12}
    \coordinate (#1-\n) at (#2+\px*#4,#3+\py*#5);
  \foreach \u/\v in {n1/n2, n1/n3, n1/n4, n2/n5, n3/n6, n4/n6, n4/n5, n6/n7, n5/n7, n7/n8}
    \draw[gedge] (#1-\u) -- (#1-\v);
  \foreach \n in {n1, n2, n3, n4, n5, n6, n7, n8}
    \node[gnode] at (#1-\n) {};
}
\newcommand{\paint}[2]{\foreach \n/\shade in {#2} \node[gnode, fill=\shade, draw=black!35] at (#1-\n) {};}

\panel{0}{0.95}{3.0}{5.9}{amber!13}{amber!68!black}{Coding agent}
\robot{0.43}{5.54}{amber!68!black}
\node[code, anchor=north] at (1.5,5.0)
  {\textcolor{amber!70!black}{\textbf{def}} select(G, k):\\\ \ S = hubs(G, k-1)\\\ \ S += [bridge(G)]\\\ \ \textcolor{amber!70!black}{\textbf{return}} S};
\node[body] at (1.5,3.2) {designs an algorithm};
\foreach \score/\px/\shade in {8.1/0.5/white, 8.9/1.5/white, 9.4/2.5/amber!38}
  \node[chip, draw=amber!75!black, line width=0.5pt, fill=\shade, text=amber!55!black, minimum width=0.6cm] at (\px,2.14) {\score};
\foreach \px in {0.87, 1.87}
  \draw[-{Stealth[length=1.7mm]}, draw=amber!75!black, line width=0.7pt] (\px,2.14) -- (\px+0.26,2.14);
\node[small] at (1.5,1.7) {reward per generation};

\begin{scope}[xshift=-0.15cm]
\panel{4.0}{0.95}{8.15}{5.9}{skyblue!11}{skyblue!80!black}{Network World Model}
\neuralnet{4.41}{5.54}{skyblue!80!black}
\foreach \frame/\fx in {f0/4.27, f1/5.507, f2/6.744} {
  \fill[white, rounded corners=3pt] (\fx,3.02) rectangle (\fx+1.137,4.76);
  \minigraph{\frame}{\fx}{3.02}{1.137}{1.74}
}
\paint{f0}{n4/skyblue, n6/skyblue}
\paint{f1}{n4/heat, n6/heat, n1/heat!45, n5/heat!30, n3/heat!50, n7/heat!35}
\paint{f2}{n4/heat, n6/heat, n1/heat!80, n5/heat!65, n3/heat!85, n7/heat!70, n2/heat!45, n8/heat!35}
\foreach \label/\px in {$t=0$/4.84, $t=1$/6.075, $t=H$/7.31}
  \node[steplabel] at (\px,4.95) {\label};
\node[steplabel] at (6.69,4.95) {$\cdots$};
\draw[-{Stealth[length=1.8mm]}, draw=skyblue!60, line width=0.8pt] (4.35,2.74) -- (7.8,2.74);
\node[body] at (6.075,2.42) {predicts how choices spread};
\node[gnode, fill=skyblue, draw=black!35] at (4.5,1.82) {};
\node[small, anchor=west] at (4.68,1.82) {seeds chosen by the algorithm};
\foreach \shade/\px in {heat!30/4.5, heat!65/4.76, heat/5.02}
  \node[gnode, fill=\shade, draw=black!35] at (\px,1.4) {};
\node[small, anchor=west] at (5.2,1.4) {infection probability};
\end{scope}

\draw[flow, draw=amber!80!black] (3.0,4.1) -- (3.85,4.1);
\draw[flow, draw=skyblue!90] (3.85,3.72) -- (3.0,3.72);

\fill[leaf!11, rounded corners=6pt] (8.3,0.95) rectangle (13.95,5.9);

\coordinate (C) at (11.125,3.56);
\def\R{1.56}\def\vcentre{0.9}\def\vspan{0.3}\def\vmax{1.2}
\newcommand{\octagon}[1]{($(C)+(0:#1)$) -- ($(C)+(45:#1)$) -- ($(C)+(90:#1)$) -- ($(C)+(135:#1)$) --
  ($(C)+(180:#1)$) -- ($(C)+(225:#1)$) -- ($(C)+(270:#1)$) -- ($(C)+(315:#1)$) -- cycle}
\pgfmathsetmacro{\rone}{\R*(1.0-\vcentre)/\vspan}

\foreach \name/\angle/\ours in {im/90/1.210, aim/45/1.038, cnd/0/1.102, ec/-45/1.538,
                                ib/-90/1.034, cr/-135/1.513, sl/180/1.259, cp/135/1.212} {
  \pgfmathsetmacro{\ro}{\R*(min(\ours,\vmax)-\vcentre)/\vspan}
  \coordinate (ou-\name) at ($(C)+(\angle:\ro)$);
}

\fill[white] \octagon{\R};
\fill[leaf, opacity=0.5, line join=round]
  (ou-im) -- (ou-aim) -- (ou-cnd) -- (ou-ec) -- (ou-ib) -- (ou-cr) -- (ou-sl) -- (ou-cp) -- cycle;
\fill[white] \octagon{\rone};
\foreach \v in {1.1} {
  \pgfmathsetmacro{\rr}{\R*(\v-\vcentre)/\vspan}
  \draw[black!10, line width=0.4pt] \octagon{\rr};
}
\foreach \angle in {0, 45, ..., 315}
  \draw[black!11, line width=0.4pt] (C) -- ($(C)+(\angle:\R)$);
\draw[black!18, line width=0.5pt] \octagon{\R};
\draw[black!70, line width=1pt, dash pattern=on 2.6pt off 1.7pt, line join=round] \octagon{\rone};
\draw[leaf!75!black, line width=1.8pt, line join=round]
  (ou-im) -- (ou-aim) -- (ou-cnd) -- (ou-ec) -- (ou-ib) -- (ou-cr) -- (ou-sl) -- (ou-cp) -- cycle;
\foreach \name in {im, aim, cnd, ib, ec, cr, sl, cp}
  \draw[white, line width=0.7pt, fill=leaf!75!black] (ou-\name) circle (0.11);

\foreach \v in {1.0, 1.2} {
  \pgfmathsetmacro{\rr}{0.924*\R*(\v-\vcentre)/\vspan}
  \node[font=\sffamily\fontsize{5.4pt}{6pt}\selectfont, text=black!50, fill=white, inner sep=0.5pt] at ($(C)+(22.5:\rr)$) {\v};
}

\node[spoke, anchor=south, align=center] at ($(C)+(90:\R+0.05)$) {Influence\\maximization};
\node[spoke, anchor=south west, align=left] at ($(C)+(45:\R+0.08)$) {Adaptive\\influence\\maximization};
\node[spoke, anchor=west, align=left] at ($(C)+(0:\R+0.07)$) {Critical\\node\\detection};
\node[spoke, anchor=north west, align=left] at ($(C)+(-45:\R+0.08)$) {Epidemic\\control};
\node[spoke, anchor=north, align=center] at ($(C)+(-90:\R+0.05)$) {Influence\\blocking};
\node[spoke, anchor=north east, align=right] at ($(C)+(-135:\R+0.08)$) {Cascade\\reconstruction};
\node[spoke, anchor=east, align=right] at ($(C)+(180:\R+0.07)$) {Source\\localization};
\node[spoke, anchor=south east, align=right] at ($(C)+(135:\R+0.08)$) {Cascade\\prediction};

\node[inner sep=0pt] at (11.125,1.2) {\usebox{\headlinekey}};
\end{tikzpicture}
\caption{A coding agent writes algorithms that choose interventions on a network, and a Network World Model rolls out how each choice unfolds, returning a reward, feedback, and probe answers the agent uses to improve the algorithm (left). Across eight tasks on complex networks, the agent-designed algorithms outperform the strongest baselines, averaged over budgets (right).}
\label{fig:headline}
\vspace{-2pt}
\end{figure}
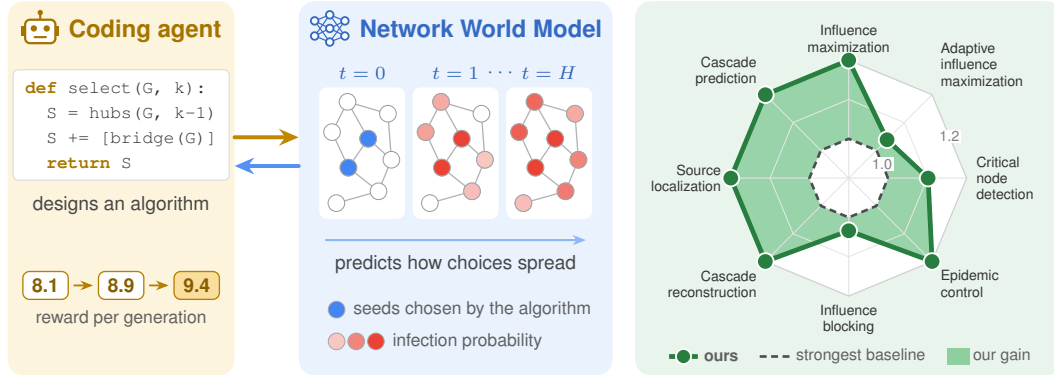

%% file: sections/1_introduction.tex
\section{Introduction}

Complex systems such as social, contact and information networks are dynamic in two ways: their state evolves on its own, and it also changes under actions taken on the system~\citep{pastor2015epidemic}. Choosing which actions to take to maximize task performance is a combinatorial problem that appears across domains, from selecting the users who will carry adoption farthest~\citep{kempe2003maximizing} to choosing who to vaccinate before an outbreak spreads~\citep{pastor2002immunization,cohen2003efficient} or seeding a correction ahead of a rumor~\citep{budak2011limiting}. One of the major roadblocks in designing algorithms for these choices lies in evaluation: how do we know what an action will lead to before taking it? Classically, the answer is simulation, called inside the search once per candidate~\citep{kempe2003maximizing,tang2014influence,han2018efficient}, since computing the expected outcome is \#P-hard~\citep{chen2010scalable}. Learning-based methods handle this problem by fitting a mapping from a network and a chosen set to the quantity of interest~\citep{xia2021deepis,ko2020monstor}, or by predicting the solution itself~\citep{fan2020finding,ling2023deep}, but each is fit to a particular task and objective. This creates a trade-off between generality and cost: simulation supports arbitrary candidate interventions but requires repeated process execution, whereas learned mappings are inexpensive but typically specialized to a particular task and objective.

Designing effective algorithms for these problems often requires substantial heuristic design and empirical tuning. LLM-based program search can automate part of this process by generating, evaluating, and refining algorithmic code~\citep{romera2024mathematical,mankowitz2023faster,novikov2025alphaevolve,liu2024evolution,ye2024reevo,van2024llamea,zheng2025monte}. Such a search process needs an environment in which the consequences of candidate actions can be evaluated. In other domains, world models serve this role by learning environment dynamics in which an agent can act~\citep{ha2018world,hafner2025mastering,schrittwieser2020mastering,bruce2024genie}. Existing world models largely operate on images, video, point clouds, or low-dimensional control states. Yet network systems require modeling both the relational structure through which interactions propagate and how interventions modify the system and its subsequent dynamics. This raises the central question of this work: can we learn such an action-conditioned world model and use it as a reliable and efficient environment for algorithm design?

Complex systems pose three challenges for world modeling and algorithm design. First, actions may modify node states, edges, or contact weights, changing both the immediate system configuration and the subsequent diffusion dynamics. Conventional world models evolve states within fixed observation spaces~\citep{hafner2025mastering,schrittwieser2020mastering}, while learned diffusion surrogates often fix the intervention type, such as a seed set, and predict a specific objective~\citep{xia2021deepis,ko2020monstor,ling2023deep}. Second, long-horizon evaluation introduces compounding transition error, so accurate next-state prediction does not necessarily preserve the relative performance of candidate algorithms over repeated rollouts. Third, evaluating candidate algorithms with Monte Carlo simulation is costly, while aggregate task scores provide limited guidance for revision, and evaluating alternatives, such as swapping a seed or blocking a different edge, incurs simulation cost.

To address these challenges, we propose a Network World Model and an LLM-based algorithm design framework (Figure~\ref{fig:headline}). The model is a graph neural simulator that applies interventions exactly to node states, edges, and contact weights, then predicts subsequent diffusion through a transition head constrained by the process's known rules. We develop a coding agent search loop that uses fast model rollouts to score candidate algorithms, estimate the contribution of individual actions through ablation, and evaluate alternatives such as node swaps, providing targeted feedback for algorithm revision. We further provide a theoretical analysis that bounds algorithm selection regret in terms of transition error, rollout horizon, and search suboptimality, and shows how exact intervention handling tightens this bound. Across eight tasks and five diffusion models, the framework matches or exceeds the strongest reported baseline in 138 of 141 settings while enabling up to 14.5$\times$ faster rollouts than Monte Carlo simulation. \textbf{Our main contributions are:} \textbf{(1)} an action-conditioned Network World Model for complex network dynamics; \textbf{(2)} an LLM-based framework for network algorithm design using the learned model; and \textbf{(3)} theoretical guarantees for world-model-based algorithm selection, together with extensive evaluation across eight tasks and five diffusion models.

%% file: sections/2_related_work.tex
\section{Related Work}

\paragraph{Algorithms for complex network dynamics.}
Classical methods for influence, containment, and inverse diffusion problems typically rely on repeated simulation or task specific structural objectives~\citep{kempe2003maximizing,pastor2002immunization,budak2011limiting,shah2011rumors}. Learned alternatives reduce this cost by approximating the forward process, directly predicting solutions, or learning generative priors over diffusion trajectories, but are generally specialized to a particular task, objective, or intervention form~\citep{xia2021deepis,fan2020finding,ling2022source,yan2023diffusion,memon2025deep}.

\paragraph{World models.}
World models learn a transition function and plan or learn a policy against it in place of the environment~\citep{ha2018world}, either by compressing observations and rolling them forward in a latent space~\citep{hafner2019learning,hafner2019dream} or by training the model to support search~\citep{schrittwieser2020mastering}. Predominantly, they are learned over video frames or low-dimensional control states, where the observation keeps a fixed shape, however the agent acts. On graphs, the pieces exist separately: learned solvers predict the next state by message passing but take no action~\citep{battaglia2016interaction,sanchez2020learning,pfaff2020learning}; the estimators above take an action but return only a final number; and graph world models act on multimodal agent states, where an action names a task and not a change to the network~\citep{feng2025graph}. A complex system requires both a latent representation structured by the graph and dynamics that respond to an action which changes that structure~\citep{memon2026toward}.

\paragraph{Algorithm discovery with language models.}
Language models can write a program, score it, and revise it, and this work has converged on an evolutionary loop over a population of candidates. FunSearch pairs a model with a program database and selects on the exact objective~\citep{romera2024mathematical}; EoH, ReEvo, LLaMEA and MCTS-AHD organize the same loop as an evolution over heuristics, adding reflection, tree search, or mutation operators~\citep{liu2024evolution,ye2024reevo,van2024llamea,zheng2025monte}; AlphaEvolve scales it to whole programs~\citep{novikov2025alphaevolve}. Reinforcement learning is the other learned route on graphs, training a policy that emits a solution instead of a procedure that produces one~\citep{khalil2017learning,manchanda2020gcomb}. Across these systems, the evaluator is taken as given, since the objective is cheap and the loop's cost is the model's. Graph diffusion breaks that assumption: scoring one candidate means thousands of simulations. We supply the evaluator and make it return diagnostics and counterfactual answers.

%% file: sections/3_methodology.tex
\section{Methodology}

\subsection{Problem Formulation}
\label{sec:problem}

Let $\mathcal{G}$ be a complex network with node set $\mathbb{V}$, such as a social network or a power grid, carrying a state $\bm{s}_t$ at discrete time $t$, and let $a_t$ be an action applied to $\mathcal{G}$ at time $t$. The network evolves as
\begin{equation}
\mathbf{s}_{t+1}\sim P(\cdot\mid \bm{s}_t, a_t),
\label{eq:kernel-a}
\end{equation}
where $P$ gives the distribution of the next state given the current state and the action applied to it. For the processes we study, the distribution of the outcome of the transition $P$ after $H$ steps is not simple: computing expected spread is \#P-hard under IC~\citep{chen2010scalable}. The score of a candidate must therefore be estimated from many sampled trajectories, which is what makes it expensive. Actions are produced by an algorithm $\pi$, and let $H$ be the horizon of the task, the number of timesteps after which its outcome is read. Applying the actions returned by $\pi$ to $\mathcal{G}$ and advancing Eq.~(\ref{eq:kernel-a}) for $H$ steps yields the outcome of the plan $\bm{a}$, which receives a score $R(\bm{a})$, the metric the task defines for that outcome. The actions must satisfy constraints $\mathbb{F}$, for example how many actions may be applied, which operations are allowed, and which elements of the network must be left intact. A task instance is formally represented by $\mathcal{T}=(\mathcal{G}, P, \mathbb{F}, H, \bm{s}_0)$ with initial state $\bm{s}_0$, and the plan $\pi$ produces is $\bm{a}=(a_0,\dots, a_{H-1})$.

The goal is to produce an algorithm that returns actions satisfying constraints and attaining the highest expected score at the horizon. The score is obtained by executing actions on the complex network and letting it evolve, which is required for every candidate algorithm and for queries about actions not yet applied. The problem therefore splits into two parts: learning a world model of the network dynamics (Section~\ref{sec:world-model}) and searching for algorithms against that world model (Section~\ref{sec:design}). We train a world model $f_{\bm{\theta}}$ on transitions $(\bm{s}_t,a_t,\bm{s}_{t+1})$ observed on the network by minimizing
\begin{equation}
\mathcal{L}(\bm{\theta})=\mathbb{E}\bigl[-\log f_{\bm{\theta}}(\bm{s}_{t+1}\mid \mathcal{G}_t,\bm{s}_t,a_t)\bigr].
\label{eq:learn-a}
\end{equation}
Algorithm design solves
\begin{equation}
\pi^\star\in\operatorname*{arg\,max}_{\pi} J_P(\pi),
\qquad
J_P(\pi)=\mathbb{E}\bigl[R(\pi(\mathcal{T}))\bigr],
\qquad\text{subject to}\qquad
\bm{a}=\pi(\mathcal{T})\in\mathbb{F},
\label{eq:objective-a}
\end{equation}
where $J_P$ is the expected score under the trusted transition kernel $P$, and the expectation is over the transition randomness in Eq.~(\ref{eq:kernel-a}). This requires $f_{\bm{\theta}}$ to capture the effect of changing actions, remain accurate over the horizon $H$, and generalize across task instances sharing the same dynamics.

\subsection{Network World Model of Complex Networks}
\label{sec:world-model}

The Network World Model $f_{\bm{\theta}}$ predicts the state of the network one step ahead, $\bm{s}_{t+1}$, from the current state $\bm{s}_t$ and the action $a_t$ applied to it. An action changes the network in two ways: directly, through an immediate effect that the action determines and that is known exactly, and indirectly, through the diffusion that follows. The functional form of the one-step diffusion is given by the process, but its parameters are in general not observed and must be inferred from recorded transitions. The direct effect is applied, and the dynamics are learned,
\begin{equation}
(\mathcal{G}_t^{+},\bm{s}_t^{+})=T_{\mathrm{exo}}(\mathcal{G}_t,\bm{s}_t,a_t),
\qquad
\hat{\bm{s}}_{t+1}
=
f_{\bm{\theta}}
(\mathcal{G}_t^{+},\bm{s}_t^{+},\bm{\eta}(a_t)),
\label{eq:wm-factorized}
\end{equation}
which is the factorized form of the model in Eq.~(\ref{eq:learn-a}): $T_{\mathrm{exo}}$ applies the action by modifying the network, its parameters, or the state as the action specifies, $(\mathcal{G}_t^{+},\bm{s}_t^{+})$ is the resulting post-intervention network and state, before the diffusion takes its step, and $\bm{\eta}(a_t)$ are features describing the operations the action applies.
Theoretical analysis in Appendices~\ref{app:exact-intervention-proof} and~\ref{app:exact-intervention-regret-proof} shows that applying the intervention exactly removes intervention semantics error from the rollout bound and consequently tightens the downstream algorithm selection guarantee.

\paragraph{Action conditioning.}
The action reaches the model through two channels: the post intervention network and state, into which $T_{\mathrm{exo}}$ has already written it, and a set of features $\bm{\eta}(a_t)$ carried alongside it. The second channel is needed because the state records what changed but not which operation caused that change, and distinct operations can produce similar post intervention states. The features record, for each modified element, the operation type and, when parameterized, its magnitude. They decompose into node level features $\bm{\eta}_v(a_t)$ and edge level features $\bm{\eta}_{uv}(a_t)$, concatenated to the node input of the message passing network and its edge map,
\begin{equation}
\bm{h}_v^{(0)}=\phi\bigl([\,\bm{s}_{t,v}^{+},\;\bm{\eta}_v(a_t)\,]\bigr),
\qquad
e_{uv}=\psi\bigl([\,\bm{h}_u,\;\bm{h}_v,\;w_{uv},\;\bm{\eta}_{uv}(a_t)\,]\bigr),
\label{eq:conditioning}
\end{equation}
Here, $\bm{s}_{t,v}^{+}$ is the post intervention state of node $v$, $\bm{h}_v^{(0)}$ its initial representation, $\bm{h}_u$ and $\bm{h}_v$ the representations returned by message passing, $w_{uv}$ the edge weight, and $\phi$ and $\psi$ the node and edge maps. For an edited edge, $\bm{\eta}_{uv}(a_t)$ records whether the edge was modified and its new weight; node operations are encoded analogously. Message passing runs on $\mathcal{G}_t^{+}$, so structural edits alter aggregation: removed edges carry no messages, while downweighted edges contribute proportionally less.

\paragraph{Transition head.}
The model of Eq.~(\ref{eq:wm-factorized}) runs in two stages. Message passing advances the system one step from the post-intervention network and state, so $\bm{h}_v$ represents the state of node $v$ at $t + 1$, and a readout $g$ reports that state. The transition constrains $g$ in two ways. First, $T_{\mathrm{exo}}$ has already set the nodes affected by the action, so $g$ is applied only to the remainder. Writing $\mathbb{U}_t$ for that set,
\begin{equation}
\hat{\bm{s}}_{t + 1,v}=g(\bm{h}_v)\ \ \text{for } v\in\mathbb{U}_t,
\qquad
\hat{\bm{s}}_{t + 1,v}=\bm{s}_{t,v}^{+}\ \ \text{for } v\notin\mathbb{U}_t.
\label{eq:head-general}
\end{equation}
Second, when the process form is known, $g$ is constrained to that form, leaving only its parameters to be learned. For a cascade process, for example, $g(\bm{h}_v)=1-\prod_{u\rightarrow v}(1-q_{uv}z_u)$, where $q_{uv}=e_{uv}$ is the learned transmission probability and $z_u$ indicates whether $u$ lies on the post intervention frontier. Training applies Eq.~(\ref{eq:learn-a}) one step at a time and node by node, using cross-entropy between $g(\bm{h}_v)$ and the conditional next state target over $\mathbb{U}_t$.

\paragraph{Rollout.}
A candidate plan $\bm{a}$ is scored by running the model forward to the horizon. A trajectory starts at $\bm{s}_0$ and repeats the same three operations at each step: the action scheduled for that step is applied by $T_{\mathrm{exo}}$, the model returns a probability for each node in $\mathbb{U}_t$, and a state is drawn from those probabilities and carried into the next step,
\begin{equation}
(\mathcal{G}_t^{+},\hat{\bm{s}}_t^{+(j)})=T_{\mathrm{exo}}(\mathcal{G}_t,\hat{\bm{s}}_t^{(j)},a_t),
\qquad
\hat{\bm{s}}_{t+1}^{(j)}
\sim
f_{\bm{\theta}}
(\mathcal{G}_t^{+},\hat{\bm{s}}_t^{+(j)},\bm{\eta}(a_t)),
\qquad t=0,\ldots,H-1.
\label{eq:wm-rollout}
\end{equation}
The score of trajectory $j$ is read from $\hat{\bm{s}}_H^{(j)}$. Drawing a state instead of carrying the probabilities forward keeps each rollout consistent with the underlying process: a vector of marginals holds every node partially active at once and discards the dependence between them, which is what the score at the horizon is built from. The candidate and current best algorithms are advanced under common random numbers so their score difference is not dominated by process variance. This is what the search calls in place of a Monte Carlo simulator. The features $\bm{\eta}$, the network, and the readout depend on the dynamic process alone, so one $\bm{\theta}$ can be applied across tasks with the same dynamic process. An evaluator $\mathcal{E}$ scores a plan $\bm{a}$ by rolling it out on $n$ sampled trajectories, under a seed $\xi$ that fixes their randomness, and averaging the per-trajectory scores,
\begin{equation}
\hat R_{\mathcal{E}}(\bm{a};\xi)=\frac{1}{n}\sum_{i=1}^{n} R_i(\bm{a};\xi),
\qquad
\hat R_{\mathcal{E}}(\pi;\xi)=\hat R_{\mathcal{E}}(\pi(\mathcal{T});\xi),
\label{eq:evaluator}
\end{equation}
where $R_i(\bm{a};\xi)$ is the task's score read at the horizon from trajectory $i$, that is, from $\hat{\bm{s}}_H^{(i)}$ under the world model. An algorithm is scored through the plan it produces, so $\hat R_{\mathcal{E}}(\pi;\xi)$ is the evaluator's estimate of $J_{\widehat P_{\bm{\theta}}}(\pi)$, the counterpart under the learned dynamics of the objective $J_P(\pi)$ in Eq.~(\ref{eq:objective-a}). To simplify notation, we write $R_i(\pi;\xi)=R_i(\pi(\mathcal{T});\xi)$.

\subsection{Network Algorithm Design Agent}
\label{sec:design}

We develop an LLM-based network algorithm design framework around the Network World Model, in which a coding agent searches over executable algorithms for the objective in Eq.~(\ref{eq:objective-a}). Section~\ref{sec:design-loop} describes each round of this search. Section~\ref{sec:feedback} describes the feedback the world model returns beyond the task score, which tells the agent how its candidate algorithm should be revised. Section~\ref{sec:paired} describes how a candidate algorithm is compared with the best under sampling noise.

\subsubsection{Algorithm Design Loop}
\label{sec:design-loop}

In each round of the algorithm design loop, the coding agent generates an executable candidate algorithm $\pi$. The agent is provided with a library of classical algorithms for the task as part of its context (Appendix~\ref{app:library}), which it may reuse, combine, or extend when generating the candidate. The candidate produces the actions $\bm{a}$ for the task instance, and the Network World Model rolls these actions forward over $n$ sampled trajectories, scoring the candidate and the current best algorithm $\hat{\pi}$ on the same trajectories under $\hat R_{\mathcal{E}}$. The candidate replaces $\hat{\pi}$ when its gain exceeds the variability of the rollout estimates, and the algorithm returned at the end of the search is the estimate of $\pi^\star$ in Eq.~(\ref{eq:objective-a}). Together with the task score, the Network World Model returns the feedback described in Section~\ref{sec:feedback}, which is provided to the agent before the next candidate is generated.

Candidate generation is governed by a set of search operators divided between exploration and exploitation. Exploration is favored earlier in the search, while later iterations increasingly refine the current best algorithm using the evidence accumulated during previous evaluations. The full operator definitions and scheduling rules are provided in Appendix~\ref{app:operators}. Each candidate is validated against the task constraints before evaluation.

\subsubsection{Rollout Feedback from the Network World Model}
\label{sec:feedback}

The task score states how well a candidate performs but not how to revise it, and a revision to a plan can take three forms: retiming its actions, removing or keeping actions in the plan, and adding or swapping in actions outside it. The Network World Model returns one view for each: rollout diagnostics describe how the process evolved under the plan, action-level credit measures what each action contributed, and counterfactual probes evaluate actions the plan did not take. Diagnostics and credit are computed after every evaluation from the rollouts that produce the score, and the agent poses probes before writing the next candidate, targeting the actions the other two single out.

\paragraph{Rollout diagnostics.} The coding agent is given a trace of how the actions produced by the candidate algorithm unfold over the sampled trajectories. For spreading tasks, these diagnostics include the estimated task score and its standard error, the cumulative number of nodes reached after each step, the number reached during that step, and the step at which the process stops spreading. Task-specific diagnostics are used for the inverse and forecasting tasks (Appendix~\ref{app:evaluators}). These diagnostics come from the same rollouts used to evaluate the actions produced by the candidate algorithm and therefore incur no additional rollout cost. They provide intermediate information about how the network state evolves under those actions, as revealed by the rollout trajectories.

\paragraph{Action-level credit.} Credit attributes the task score to the individual actions of the plan. Let $a_1,\ldots,a_k$ denote the $k$ attributable actions contained in the plan $\bm{a}$ produced by the candidate algorithm, and let $\bm{a}_{-j}$ denote the same plan with action $a_j$ removed. The credit assigned to $a_j$ is the change in the evaluator score when that action is omitted under the same rollout seed $\xi$,
\begin{equation}
c_j=\hat R_{\mathcal{E}}(\bm{a};\xi)-\hat R_{\mathcal{E}}(\bm{a}_{-j};\xi).
\label{eq:credit}
\end{equation}
A larger $c_j$ indicates a greater marginal contribution of $a_j$ to the evaluated outcome, whereas a value near zero indicates that removing the action has little effect on the task score. Computing credit for all $k$ actions requires $k+1$ plan evaluations, which are batched through the Network World Model over the same $H$-step rollout horizon.

\paragraph{Counterfactual probes.} Probes evaluate actions outside the current plan. Before generating the next candidate algorithm, the coding agent is given a probe turn in which it can query the Network World Model about counterfactual modifications to the current action set. Each query is evaluated through a world model rollout. For example, one of the probes uses gradients from a mean-field rollout to identify unselected actions with high marginal potential and selected actions with low marginal value. More generally, the probes evaluate alternative actions and interrogate how those alternatives change the resulting network dynamics. The full set of probe operations is provided in Table~\ref{tab:probes} of Appendix~\ref{app:probes}. Together, these feedback signals inform the generation of the next candidate, whose performance is then compared with the current best algorithm.

\subsubsection{Search and Candidate Selection}
\label{sec:paired}

\paragraph{Candidate algorithm selection.} The task scores of the candidate algorithm $\pi$ and the current best algorithm $\hat{\pi}$ are estimated from sampled trajectories and are therefore subject to rollout variability. To reduce the effect of sampling noise on the comparison, a fresh seed $\xi$ is drawn in each round, and both algorithms are evaluated on the same $n$ sampled trajectories. Using shared rollout realizations makes the paired score difference reflect differences between the two algorithms, with variation from independently sampled trajectories removed from the comparison. For trajectory $i$, the paired difference, the mean improvement, and its standard error are
\begin{equation}
\Delta_i = R_i(\pi;\xi)-R_i(\hat{\pi};\xi),
\qquad
\Delta=\frac{1}{n}\sum_{i=1}^{n}\Delta_i,
\qquad
b=\frac{\operatorname{sd}(\Delta_1,\ldots,\Delta_n)}{\sqrt{n}}.
\label{eq:band}
\end{equation}
The candidate replaces $\hat{\pi}$ when $\Delta>b$ and is otherwise rejected, so only improvements larger than the estimated variability of the paired rollout differences are accepted. The current best algorithm is re-evaluated under the fresh seed in every round, preventing its score from being carried forward from a favorable rollout realization.

\subsection{Reliability of World-Model-Based Algorithm Selection}
\label{sec:theory-regret}

The Network World Model approximates the transition kernel $P$, so errors in the learned kernel $\widehat P_{\bm{\theta}}$ may affect which algorithm is selected. We quantify this through algorithm selection regret. Let $\Pi$ be the set of algorithms considered by the search, with expected scores $J_P(\pi)$ and $J_{\widehat P_{\bm{\theta}}}(\pi)$ under the trusted and learned dynamics. For $\pi^\star\in\operatorname*{arg\,max}_{\pi\in\Pi}J_P(\pi)$ and the returned algorithm $\widehat{\pi}$,
\begin{equation}
\operatorname{Reg}_P(\widehat{\pi};\Pi)
=
J_P(\pi^\star)-J_P(\widehat{\pi}).
\label{eq:selection-regret-definition}
\end{equation}
Because the search runs for finitely many rounds, let $\eta_{\mathrm{search}}\geq0$ denote its learned objective suboptimality, such that $J_{\widehat P_{\bm{\theta}}}(\widehat{\pi})\geq\sup_{\pi\in\Pi}J_{\widehat P_{\bm{\theta}}}(\pi)-\eta_{\mathrm{search}}$.

\begin{theorem}[World-model error implies bounded algorithm-selection regret]
\label{thm:selection-regret}
Suppose the one-step transition error of the Network World Model is at most $\epsilon$ over the state-action support reachable by algorithms in $\Pi$ within horizon $H$, the task score range is bounded by $B_R$, and the returned algorithm is $\eta_{\mathrm{search}}$ suboptimal under the learned evaluator. Then
\begin{equation}
\operatorname{Reg}_P(\widehat{\pi};\Pi)
\leq2B_RH\epsilon+\eta_{\mathrm{search}}.
\label{eq:selection-regret-bound}
\end{equation}
\end{theorem}
The theorem shows that exact reproduction of every rollout is not required for reliable algorithm selection: the downstream loss is controlled by the accumulated transition error and the search error. The full assumptions and the proof are provided in Appendix~\ref{app:regret-proof}. Further consequences for pairwise ranking preservation and exact recovery of the optimal algorithm under a sufficient selection margin are provided in Appendices~\ref{app:ranking-proof} and~\ref{app:recovery-proof}.

%% file: sections/4_experiments.tex
\section{Experiments}
\label{sec:experiments}

\begin{table}[!t]
\centering
\caption{Performance of designed algorithms on four tasks and two datasets per task. Each cell reports results under two diffusion models: IC\,/\,LT, except for influence blocking, which reports IC\,/\,CLT. A dash indicates an OOM error or timeout.}
\label{tab:results_main}
\resizebox{\textwidth}{!}{%
\begin{tabular}{@{}l*{8}{c}@{}}
\toprule
\multicolumn{9}{@{}l}{\textbf{Influence maximization} (IC\,/\,LT): spread, \% of nodes activated ($\uparrow$)} \\
\addlinespace[2pt]
 & \multicolumn{4}{c}{Network Science (1,589 nodes)} & \multicolumn{4}{c}{Digg (116,893 nodes)} \\
\cmidrule(lr){2-5}\cmidrule(lr){6-9}
Method & 1\% & 5\% & 10\% & 20\% & 1\% & 5\% & 10\% & 20\% \\
\midrule
IMM                               & 8.7\,/\,10.9 & 24.8\,/\,30.2 & 38.0\,/\,45.1 & 59.2\,/\,68.5 & 25.9\,/\,47.0 & 36.0\,/\,60.9 & 44.2\,/\,69.7 & 55.7\,/\,79.6 \\
OPIM                              & 8.8\,/\,10.8 & 24.2\,/\,29.5 & 37.7\,/\,45.3 & 58.3\,/\,67.4 & 27.4\,/\,50.5 & 41.6\,/\,68.4 & 51.8\,/\,77.6 & 62.7\,/\,85.5 \\
SubSIM                            & 8.7\,/\,10.7 & 24.3\,/\,29.1 & 37.3\,/\,45.0 & 57.8\,/\,68.5 & 27.3\,/\,50.1 & 41.7\,/\,68.4 & 51.8\,/\,77.7 & 62.7\,/\,85.4 \\
DeepIM                            & 4.8\,/\,5.3 & 15.3\,/\,20.0 & 27.7\,/\,32.9 & 45.8\,/\,53.0 & \na\,/\,\na & \na\,/\,\na & \na\,/\,\na & \na\,/\,\na \\
DegreeDiscount                    & 8.3\,/\,10.7 & 23.7\,/\,29.7 & 36.3\,/\,44.5 & 55.0\,/\,65.5 & 26.0\,/\,49.6 & 37.0\,/\,62.7 & 45.8\,/\,70.0 & 60.0\,/\,81.6 \\
\midrule
\textbf{Network World Model} & \textbf{8.9}\,/\,\textbf{11.3} & \textbf{25.3}\,/\,\textbf{30.9} & \textbf{39.1}\,/\,\textbf{46.8} & \textbf{60.1}\,/\,\textbf{69.3} & \textbf{30.4}\,/\,\textbf{56.8} & \textbf{45.8}\,/\,\textbf{73.2} & \textbf{56.5}\,/\,\textbf{83.1} & \textbf{69.4}\,/\,\textbf{93.5} \\
\bottomrule
\addlinespace[1em]
\toprule
\multicolumn{9}{@{}l}{\textbf{Adaptive influence maximization} (IC\,/\,LT): spread, \% of nodes activated ($\uparrow$)} \\
\addlinespace[2pt]
 & \multicolumn{4}{c}{Network Science (1,589 nodes)} & \multicolumn{4}{c}{Digg (116,893 nodes)} \\
\cmidrule(lr){2-5}\cmidrule(lr){6-9}
Method & 1\% & 5\% & 10\% & 20\% & 1\% & 5\% & 10\% & 20\% \\
\midrule
EPIC                  & 8.4\,/\,10.8 & 23.1\,/\,28.7 & 34.6\,/\,41.5 & 52.9\,/\,62.3 & 27.5\,/\,48.4 & 40.0\,/\,65.8 & 49.0\,/\,74.0 & 60.7\,/\,82.4 \\
Adaptive DegreeDiscount           & 8.2\,/\,10.8 & 23.6\,/\,30.2 & 36.5\,/\,45.8 & 56.9\,/\,69.3 & 29.0\,/\,54.2 & 43.5\,/\,68.9 & 54.8\,/\,79.4 & 69.3\,/\,93.6 \\
IMM                               & 8.7\,/\,10.9 & 24.8\,/\,30.2 & 38.0\,/\,45.1 & 59.2\,/\,68.5 & 25.9\,/\,47.0 & 36.0\,/\,60.9 & 44.2\,/\,69.7 & 55.7\,/\,79.6 \\
Static-Split                      & 8.3\,/\,10.7 & 24.1\,/\,30.3 & 37.5\,/\,45.9 & 58.1\,/\,69.3 & 26.4\,/\,49.6 & 38.7\,/\,64.5 & 49.0\,/\,73.8 & 65.1\,/\,87.5 \\
\midrule
\textbf{Network World Model} & \textbf{8.9}\,/\,\textbf{11.3} & \textbf{25.2}\,/\,\textbf{30.9} & \textbf{39.1}\,/\,\textbf{46.8} & \textbf{60.0}\,/\,\textbf{69.4} & \textbf{30.4}\,/\,\textbf{56.8} & \textbf{45.6}\,/\,\textbf{73.2} & \textbf{56.4}\,/\,\textbf{83.0} & \textbf{71.9}\,/\,\textbf{94.0} \\
\bottomrule
\addlinespace[1em]
\toprule
\multicolumn{9}{@{}l}{\textbf{Critical node detection} (IC\,/\,LT): remaining spread, \% of nodes infected ($\downarrow$)} \\
\addlinespace[2pt]
 & \multicolumn{4}{c}{Power Grid (4,941 nodes)} & \multicolumn{4}{c}{PGP (10,680 nodes)} \\
\cmidrule(lr){2-5}\cmidrule(lr){6-9}
Method & 1\% & 5\% & 10\% & 20\% & 1\% & 5\% & 10\% & 20\% \\
\midrule
HDA                               & 26.5\,/\,28.9 & 24.0\,/\,27.1 & 21.8\,/\,25.4 & 18.6\,/\,22.3 & 26.4\,/\,31.9 & 23.3\,/\,28.5 & 20.8\,/\,25.6 & 18.0\,/\,21.6 \\
BPD+R                             & 26.7\,/\,29.1 & 24.1\,/\,27.5 & 21.9\,/\,25.8 & 18.5\,/\,22.4 & 26.4\,/\,32.0 & \na\,/\,\na & 21.2\,/\,25.8 & 18.0\,/\,21.6 \\
CI+R                              & 26.7\,/\,29.0 & 24.1\,/\,27.3 & 21.9\,/\,25.7 & 18.5\,/\,22.3 & 26.6\,/\,32.0 & 23.4\,/\,28.5 & 21.2\,/\,25.9 & 18.1\,/\,21.6 \\
EI                                & 26.8\,/\,29.2 & 24.5\,/\,27.9 & 22.1\,/\,26.0 & 19.1\,/\,23.2 & 26.5\,/\,32.0 & 23.5\,/\,28.8 & 21.1\,/\,25.9 & 18.2\,/\,22.2 \\
Frontier                          & 26.3\,/\,28.6 & 22.7\,/\,25.0 & 18.7\,/\,20.4 & 10.7\,/\,\textbf{10.7} & 26.7\,/\,32.0 & 23.4\,/\,27.9 & 20.2\,/\,23.6 & 14.2\,/\,14.5 \\
\midrule
\textbf{Network World Model} & \textbf{26.2}\,/\,\textbf{28.2} & \textbf{22.0}\,/\,\textbf{24.1} & \textbf{17.3}\,/\,\textbf{19.6} & \textbf{10.2}\,/\,\textbf{10.7} & \textbf{25.8}\,/\,\textbf{31.1} & \textbf{21.3}\,/\,\textbf{26.7} & \textbf{16.5}\,/\,\textbf{21.5} & \textbf{11.0}\,/\,\textbf{13.0} \\
\bottomrule
\addlinespace[1em]
\toprule
\multicolumn{9}{@{}l}{\textbf{Influence blocking} (IC\,/\,CLT): rumor cascade size (nodes) ($\downarrow$)} \\
\addlinespace[2pt]
 & \multicolumn{4}{c}{Email-EU (1,005 nodes)} & \multicolumn{4}{c}{Gnutella24 (26,518 nodes)} \\
\cmidrule(lr){2-5}\cmidrule(lr){6-9}
Method & 10 & 20 & 40 & 50 & 10 & 20 & 40 & 50 \\
\midrule
RPS                               & 60.2\,/\,\textbf{70.9} & 47.4\,/\,56.7 & 35.0\,/\,42.8 & 31.9\,/\,39.1 & 1520.6\,/\,1625.3 & 1311.6\,/\,1380.6 & 1029.6\,/\,1076.5 & 939.8\,/\,980.4 \\
Reverse                           & 61.0\,/\,73.1 & 49.2\,/\,57.9 & 37.2\,/\,44.8 & 34.5\,/\,42.1 & 1817.6\,/\,1961.8 & 1713.1\,/\,1856.4 & 1462.0\,/\,1587.7 & 1413.1\,/\,1535.7 \\
Proximity                         & 63.1\,/\,74.9 & 52.4\,/\,61.2 & 39.6\,/\,47.7 & 36.3\,/\,44.2 & 2338.9\,/\,2531.6 & 2289.9\,/\,2481.1 & 2196.5\,/\,2379.0 & 2141.1\,/\,2319.7 \\
GreedyReplace                     & 92.2\,/\,162.3 & 77.2\,/\,163.5 & 60.1\,/\,157.3 & 54.8\,/\,154.9 & 1551.8\,/\,1640.5 & 1364.9\,/\,1449.2 & 1088.2\,/\,1211.8 & 990.4\,/\,1072.4 \\
SandIMIN                          & 92.5\,/\,160.0 & 77.4\,/\,158.7 & 63.7\,/\,157.8 & 54.4\,/\,150.9 & 1564.9\,/\,1633.3 & 1391.2\,/\,1456.3 & 1106.4\,/\,1173.9 & 1017.8\,/\,1096.9 \\
\midrule
\textbf{Network World Model} & \textbf{59.4}\,/\,71.7 & \textbf{46.3}\,/\,\textbf{55.4} & \textbf{34.5}\,/\,\textbf{42.2} & \textbf{31.0}\,/\,\textbf{38.7} & \textbf{1503.3}\,/\,\textbf{1582.3} & \textbf{1275.0}\,/\,\textbf{1333.8} & \textbf{985.6}\,/\,\textbf{1023.0} & \textbf{896.0}\,/\,\textbf{924.7} \\
\bottomrule
\end{tabular}}
\vspace{-2pt}
\end{table}

\subsection{Setup}
\label{sec:setup}

\paragraph{Complex network tasks.} We evaluate eight tasks spanning three problem families: five intervention tasks, \emph{influence maximization}, \emph{adaptive influence maximization}, \emph{critical node detection}, \emph{influence blocking}, and \emph{epidemic control}; two inverse tasks, \emph{source localization} and \emph{cascade reconstruction}; and one forecasting task, \emph{cascade prediction}. Appendix~\ref{app:tasks} provides detailed task definitions.

\paragraph{Evaluation protocol.} Each task is evaluated on networks ranging from $10^3$ to more than $10^5$ nodes (Appendix~\ref{app:datasets}) under the appropriate diffusion models: IC, LT, CLT, SIR, and SIS. We compare against five task-specific baselines, including classical algorithms and neural methods, listed in Appendix~\ref{app:baselines}. The design loop runs for $G = 10$ rounds with $n = 200$ sampled rollouts per evaluation, horizon $H = 10$, and six probes per round. GPT-5.6 Sol~\citep{openai2026gpt56sol} is used as the coding agent for the main experiments, with the coding model comparison reported in Figure~\ref{fig:coding_model}. The final performance of both the designed algorithms and the baselines is evaluated independently using the trusted Monte Carlo simulator with $200$ samples (Appendix~\ref{app:evaluators}). Seeding, containment, and epidemic tasks use budgets of $1\%$, $5\%$, $10\%$, and $20\%$ of $N$, while influence blocking uses $k \in \{10, 20, 30, 40, 50\}$; recovery and forecasting tasks have no intervention budget. Performance is reported using the task-specific metric for each designed algorithm, while the Network World Model is evaluated using one-step $\Delta$F1, Brier score, and rollout count bias. Appendix~\ref{app:tasks} and Appendix~\ref{app:wm-study} provide the remaining implementation and experimental details.

\subsection{Performance of Designed Algorithms}
\label{sec:exp-design}

\begin{wrapfigure}{r}{0.37\textwidth}
\vspace{-4pt}
\centering
\includegraphics[width=0.35\textwidth]{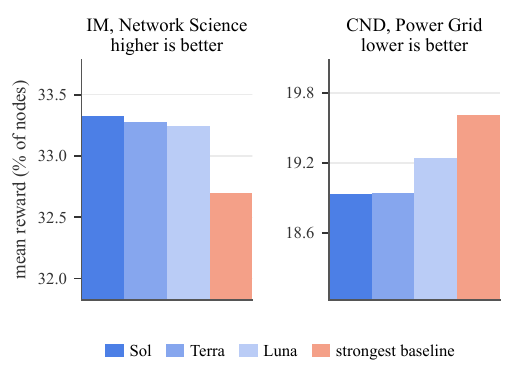}
\vspace{-8pt}
\caption{Task score of the returned algorithm under three coding models, averaged over budgets.}
\label{fig:coding_model}
\vspace{-14pt}
\end{wrapfigure}

The designed algorithms are evaluated with the trusted Monte Carlo 
simulator, independently of the Network World Model used during search, to assess their performance under the true environment dynamics. Across the eight tasks summarized in Figure~\ref{fig:headline}, the designed algorithms improve on the strongest existing method in all eight tasks and trail it in only 3 of the 141 settings. Table~\ref{tab:results_main} reports representative results, while complete results across all eight tasks and evaluated datasets are provided in Appendix~\ref{app:full-results}. Across the 64 configurations in Table~\ref{tab:results_main}, the designed algorithms beat or tie the strongest listed baseline in 63, except in one influence-blocking setting. The gains are observed under both IC and LT dynamics and hold across intervention budgets (Figure~\ref{fig:budget_curves}, Appendix~\ref{app:full-results}). We additionally compare the framework with four published LLM algorithm discovery systems using the same coding model. The proposed framework performs best at every evaluated budget on both influence maximization and critical node detection; full results are reported in Table~\ref{tab:discovery} of Appendix~\ref{app:discovery}.

\textbf{Choice of coding model.} We show that the gains do not depend on which LLM writes the code. Figure~\ref{fig:coding_model} repeats the search on Network Science and Power Grid under three models, GPT-5.6 Sol, Terra, and Luna. Averaged over the four budgets, the three finish within $0.1$ points of one another on Network Science and within $0.3$ points on Power Grid, and each beats the strongest baseline. The similar performance across coding models suggests that the gains are not specific to the choice of coding model. We provide the full numbers across all three models in Appendix~\ref{app:coding-model}.

\subsection{Cross-Network Generalization and Scaling}
\label{sec:exp-generalizability}
We test whether an algorithm designed on one network transfers to other networks of the same task without modification. Figure~\ref{fig:transfer} replays the critical node detection algorithm designed on PGP on Power Grid and Gnutella31, each against the strongest baseline of the target network. The transferred algorithms beat the strongest baseline on every target by 1.8 to 5.1\% for influence maximization and 4.0 to 11.7\% for critical node detection. Figure~\ref{fig:scale} complements this transfer result by showing influence maximization performance across increasingly large networks, where the designed algorithm remains ahead of the strongest baseline from Network Science through Digg, with the improvement reaching 7 to 12\% on Digg compared with 1 to 4\% on Network Science. Across the completed transfers, performance remains within $0.1$ percentage points of algorithms designed directly on the target network at every budget (Table~\ref{tab:transfer}, Appendix~\ref{app:transfer}).
\begin{figure}[h]
\centering
\captionsetup[subfloat]{position=bottom,font=scriptsize,farskip=0pt,captionskip=2pt}
\subfloat[Critical node detection, designed on PGP]{\includegraphics[width=0.36\textwidth]{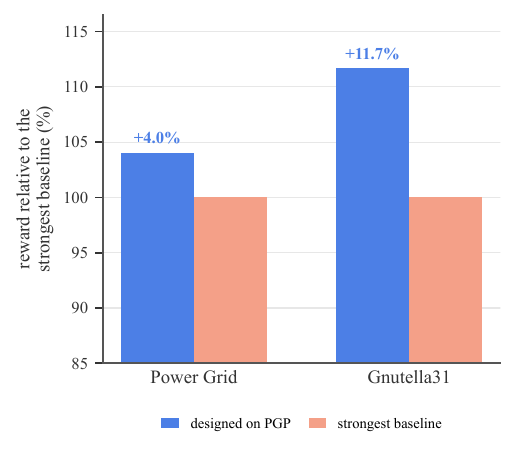}\label{fig:transfer}}\hspace{0.04\textwidth}
\subfloat[Influence maximization across network sizes]{\includegraphics[width=0.40\textwidth]{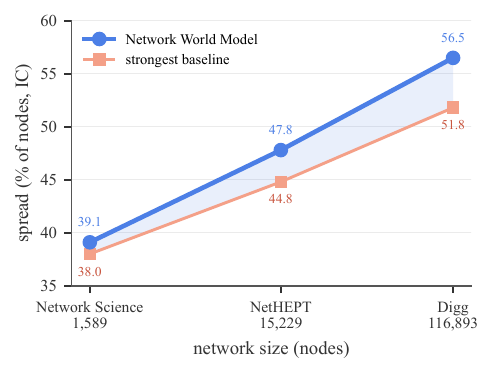}\label{fig:scale}}
\caption{Generalization across networks and scales. (a) Critical node detection transfer from PGP to Power Grid and Gnutella31, sign-corrected so higher is better. (b) Influence maximization performance across increasing network sizes under IC at the 10\% budget.}
\label{fig:generalization}
\end{figure}

\subsection{Network World Model Accuracy and Efficiency}
\label{sec:exp-wm}

We evaluate whether the Network World Model provides the three properties algorithm design requires: accurate rollouts, reliable candidate selection, and efficient evaluation. Table~\ref{tab:wm_main} reports prediction quality for five diffusion models. One-step $\Delta$F1 ranges from $0.87$ to $0.99$ for IC, LT, and CLT and from $0.85$ to $0.86$ for SIR and SIS, with Brier scores $\leq 0.0006$. The model thus accurately identifies which nodes change state with well-calibrated transition probabilities. Free-running rollouts remain within $2.2\%$ of the trusted simulator at the horizon except under LT, where the model undercounts by $6.0\%$. Rollout bias also remains stable over the horizon (Appendix~\ref{app:wm-study}).
\begin{table}[t]
\centering
\small
\begin{minipage}[t]{0.46\textwidth}
\centering
\caption{Prediction and rollout quality of the Network World Model across five diffusion models, on Network Science (IC, LT), Email-EU (CLT) and Oregon1 (SIR, SIS).}
\label{tab:wm_main}
\setlength{\tabcolsep}{4pt}
\begin{tabular}{@{}lccc@{}}
\toprule
Dynamics & $\Delta$F1 ($\uparrow$) & Brier ($\downarrow$) & Bias (\%) \\
\midrule
IC  & 0.8738 & 0.0003 & $+0.1$ \\
LT  & 0.9926 & 0.0003 & $-6.0$ \\
CLT & 0.9880 & 0.0001 & $-2.2$ \\
SIR & 0.8565 & 0.0006 & $+0.5$ \\
SIS & 0.8461 & 0.0006 & $-0.1$ \\
\bottomrule
\end{tabular}
\end{minipage}\hfill
\begin{minipage}[t]{0.51\textwidth}
\centering
\caption{Decision value on 20 held-out graphs per family. Regret is the percent difference in reward between the algorithms chosen by the world model and by Monte Carlo simulation.}
\label{tab:wm_ranking}
\setlength{\tabcolsep}{4pt}
\begin{tabular}{@{}lccc@{}}
\toprule
Graph family & \shortstack{Pref.\ accuracy\\{}[95\% CI]} & \shortstack{WM\\regret} & \shortstack{Random\\regret} \\
\midrule
SBM (in dist.) & 0.902 [0.890, 0.914] & 0.11 & 10.27 \\
BA (in dist.)  & 0.861 [0.818, 0.902] & 0.13 & 12.03 \\
WS (shift)     & 0.860 [0.832, 0.887] & 1.18 & 10.07 \\
\bottomrule
\end{tabular}
\end{minipage}
\vspace{-3pt}
\end{table}

For algorithm search, the core requirement is preserving the ordering of candidate algorithms. Table~\ref{tab:wm_ranking} compares the same candidate pools scored by the Network World Model and the trusted simulator using preference accuracy and selection regret. On held-out SBM and BA networks, the model agrees with the trusted simulator on $90.2\%$ and $86.1\%$ of candidate pair comparisons, with selection regret of 
$0.11\%$ and $0.13\%$. On Watts-Strogatz networks, used as a topology shift test, 
\begin{wrapfigure}{l}{0.40\textwidth}
\vspace{-12pt}
\centering
\includegraphics[width=0.40\textwidth]{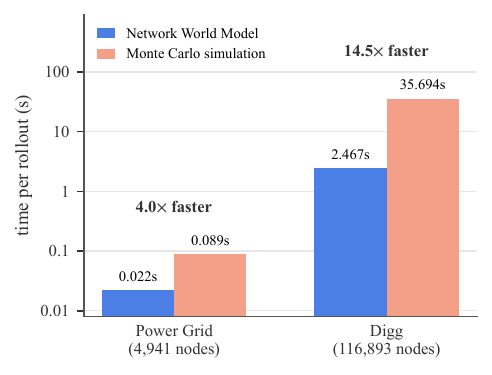}
\vspace{-8pt}
\caption{Wall-clock time of one rollout sample under the Network World Model and Monte Carlo simulation, log scale.}
\label{fig:cost}
\vspace{-24pt}
\end{wrapfigure}
preference accuracy remains $86.0\%$ and selection regret is
$1.18\%$. The resulting low regret shows that the model can reliably 
support candidate selection even when pairwise agreement is not perfect.

The practical utility of the Network World Model is the reduction in repeated evaluation cost. Figure~\ref{fig:cost} compares the time of one rollout against direct simulation: a world model rollout is $4.0$ times faster on Power Grid and $14.5$ times faster on Digg. The model incurs a one-time offline cost of $38{,}800$ simulator episodes to build the training set, whereas direct simulator-based search requires $7{,}800$ episodes for each new search, giving a break even point of about five searches. Reusing the learned dynamics therefore reduces both online evaluation time and cumulative trusted simulator use across repeated searches.

%% file: sections/5_conclusion.tex
\vspace{-2pt}
\section{Conclusion}
\vspace{-2pt}
We introduce a Network World Model as an environment for algorithm design on complex networks, together with an LLM-based search framework that uses model rollouts and targeted feedback to refine executable algorithms. Across eight tasks and five diffusion models, the designed algorithms match or exceed the strongest reported baseline in 138 of 141 settings, transfer across networks with little performance loss, and can be evaluated faster than with Monte Carlo simulation. Overall, the results show that learned network dynamics can support reliable algorithm selection and extend beyond prediction to the iterative design and evaluation of algorithms for complex networks across a range of tasks.

%% file: sections/ai_use_statement.tex
We used generative AI tools to edit the writing, including grammar checking and polishing the text for readability. In addition, large language models are part of the proposed method: the coding agent in the algorithm design loop is an LLM (GPT-5.6 Sol, Terra, or Luna) that writes and revises candidate algorithms. In the paper, we have reviewed all AI-assisted work. We read every suggested edit and accepted or rejected it by hand, and no edit changed the paper's technical content, claims, or results. We take responsibility for the final content of this work, including text, claims, or artifacts produced with the aid of generative AI.

%% file: sections/ethics_statement.tex
Our work does not raise ethical concerns beyond those in the established network diffusion literature it builds on. It involves no human subjects and no new data collection: all networks and cascade corpora are public, all datasets and baseline methods are used under their original terms, and we use their structure and adoption times, never personal attributes or content. All interventions are evaluated in simulation, and none is deployed on a real platform or population.

%% file: sections/reproducibility_statement.tex
All datasets we evaluate are public, and we describe their sources in the appendix. The appendix specifies the tasks, baselines, evaluators, and experimental protocol; the world model setup; and the design loop, including its operators, probes, and implementation details. It also includes the complete prompts given to the coding agent. The coding agent itself uses closed-source LLMs accessed through an API, which we chose because they are the current state of the art. We provide examples of some of the best-performing designed algorithms with their source code in the appendix.

%% file: sections/appendix.tex
\clearpage
\appendix

\section{Theoretical Analysis}
\label{app:theory-proofs}

\subsection{Proof of Theorem~\ref{thm:selection-regret}: Algorithm-Selection Regret}
\label{app:regret-proof}

Let $P(\cdot\mid\bm{s}, a)$ denote the trusted transition kernel in Eq.~(\ref{eq:kernel-a}), and let $\widehat P_{\bm{\theta}}(\cdot\mid\bm{s}, a)$ denote the transition kernel induced by the Network World Model in Eq.~(\ref{eq:wm-factorized}). For an executable algorithm $\pi$, let $J_P(\pi)$ and $J_{\widehat P_{\bm{\theta}}}(\pi)$ denote its expected reward under the trusted and learned dynamics, respectively. Let $\Pi$ be the executable algorithm class considered by the search, and let $\pi^\star\in\operatorname*{arg\,max}_{\pi\in\Pi}J_P(\pi)$ be the best algorithm in this class under the trusted dynamics. The algorithm-selection regret $\operatorname{Reg}_P(\widehat{\pi};\Pi)$ of the returned algorithm is defined in Eq.~(\ref{eq:selection-regret-definition}). The Network World Model is trained on one-step transitions through Eq.~(\ref{eq:learn-a}) to approximate the conditional transition law of the trusted dynamics. Its approximation quality is therefore characterized by the one-step transition error over the state-action support reachable during algorithm search.

\begin{assumption}[Uniform transition accuracy on the search support]
\label{ass:wm-transition}
Let $\mathcal{Z}_{\Pi,H}$ be the set of state-action pairs reachable within horizon $H$ by algorithms in $\Pi$. There exists $\epsilon\in[0,1]$ such that
\begin{equation}
\sup_{(\bm{s},a)\in\mathcal{Z}_{\Pi,H}}D_{\mathrm{TV}}\left(P(\cdot\mid\bm{s},a),\widehat P_{\bm{\theta}}(\cdot\mid\bm{s},a)\right)\leq\epsilon.
\label{eq:wm-transition-error}
\end{equation}
\end{assumption}

\begin{assumption}[Bounded task reward]
\label{ass:bounded-reward}
The terminal reward has bounded range over all feasible trajectories: $\sup_{\tau}R(\tau)-\inf_{\tau}R(\tau)\leq B_R<\infty$.
\end{assumption}

\begin{assumption}[Approximate world-model search optimality]
\label{ass:search-optimality}
Let $\widehat{\pi}$ be the algorithm returned by the coding-agent search using the Network World Model evaluator, and let $\eta_{\mathrm{search}}\geq0$ be its suboptimality under the learned evaluator as defined in Section~\ref{sec:theory-regret}, so that
\begin{equation}
J_{\widehat P_{\bm{\theta}}}(\widehat{\pi})\geq\sup_{\pi\in\Pi}J_{\widehat P_{\bm{\theta}}}(\pi)-\eta_{\mathrm{search}}.
\label{eq:search-error}
\end{equation}
\end{assumption}

Theorem~\ref{thm:selection-regret} is the statement of Eq.~(\ref{eq:selection-regret-bound}) under Assumptions~\ref{ass:wm-transition} to~\ref{ass:search-optimality}.

\begin{proof}
Fix any $\pi\in\Pi$. By Assumption~\ref{ass:wm-transition} and maximal coupling, the trusted and learned trajectories remain coupled at each transition with probability at least $1-\epsilon$. Therefore,
\begin{align}
D_{\mathrm{TV}}\left(\mathbb{P}_{P}^{\pi},\mathbb{P}_{\widehat P_{\bm{\theta}}}^{\pi}\right)
&\leq1-(1-\epsilon)^H \nonumber\\
&\leq H\epsilon.
\label{eq:trajectory-tv-bound}
\end{align}

By Assumption~\ref{ass:bounded-reward} and Eq.~(\ref{eq:trajectory-tv-bound}),
\begin{align}
\left|J_P(\pi)-J_{\widehat P_{\bm{\theta}}}(\pi)\right|
&\leq B_RD_{\mathrm{TV}}\left(\mathbb{P}_{P}^{\pi},\mathbb{P}_{\widehat P_{\bm{\theta}}}^{\pi}\right) \nonumber\\
&\leq B_R\left(1-(1-\epsilon)^H\right) \nonumber\\
&\leq B_RH\epsilon.
\label{eq:uniform-value-bound}
\end{align}

Applying Eq.~(\ref{eq:uniform-value-bound}) to $\pi^\star$ and $\widehat{\pi}$ and using Assumption~\ref{ass:search-optimality},
\begin{align}
J_P(\pi^\star)
&\leq J_{\widehat P_{\bm{\theta}}}(\pi^\star)+B_R\left(1-(1-\epsilon)^H\right) \nonumber\\
&\leq \sup_{\pi\in\Pi}J_{\widehat P_{\bm{\theta}}}(\pi)+B_R\left(1-(1-\epsilon)^H\right) \nonumber\\
&\leq J_{\widehat P_{\bm{\theta}}}(\widehat{\pi})+\eta_{\mathrm{search}}+B_R\left(1-(1-\epsilon)^H\right) \nonumber\\
&\leq J_P(\widehat{\pi})+2B_R\left(1-(1-\epsilon)^H\right)+\eta_{\mathrm{search}}.
\label{eq:regret-proof-chain}
\end{align}

Combining Eq.~(\ref{eq:regret-proof-chain}) with the regret definition in Eq.~(\ref{eq:selection-regret-definition}),
\begin{align}
\operatorname{Reg}_P(\widehat{\pi};\Pi)
&\leq2B_R\left(1-(1-\epsilon)^H\right)+\eta_{\mathrm{search}} \nonumber\\
&\leq2B_RH\epsilon+\eta_{\mathrm{search}},
\label{eq:regret-final-chain}
\end{align}
which proves Eq.~(\ref{eq:selection-regret-bound}).
\end{proof}

\subsection{Pairwise Ranking Preservation}
\label{app:ranking-proof}

Write $\Delta_H(\epsilon)=B_R\left(1-(1-\epsilon)^H\right)$ for the worst-case distortion the Network World Model induces in the value of an algorithm over horizon $H$, which is the bound established in Eq.~(\ref{eq:uniform-value-bound}).

\begin{corollary}[Pairwise ranking preservation]
\label{cor:ranking-preservation}
Under Assumptions~\ref{ass:wm-transition} and~\ref{ass:bounded-reward}, let $\pi_i,\pi_j\in\Pi$. If
\begin{equation}
J_P(\pi_i)-J_P(\pi_j)>2\Delta_H(\epsilon),
\label{eq:ranking-margin}
\end{equation}
then $J_{\widehat P_{\bm{\theta}}}(\pi_i)>J_{\widehat P_{\bm{\theta}}}(\pi_j)$.
\end{corollary}

Algorithms separated by more than the world-model distortion therefore preserve their true ordering under the learned evaluator. As training reduces $\epsilon$ and hence $\Delta_H(\epsilon)$, the guarantee extends to increasingly close algorithm pairs.

\begin{proof}
Applying Eq.~(\ref{eq:uniform-value-bound}) to $\pi_i$ and $\pi_j$ gives
\begin{align}
J_{\widehat P_{\bm{\theta}}}(\pi_i)-J_{\widehat P_{\bm{\theta}}}(\pi_j)
&\geq J_P(\pi_i)-\Delta_H(\epsilon)-J_P(\pi_j)-\Delta_H(\epsilon) \nonumber\\
&=J_P(\pi_i)-J_P(\pi_j)-2\Delta_H(\epsilon) \nonumber\\
&>0,
\label{eq:app-ranking-proof}
\end{align}
where the last inequality follows from the margin condition in Eq.~(\ref{eq:ranking-margin}). Hence $J_{\widehat P_{\bm{\theta}}}(\pi_i)>J_{\widehat P_{\bm{\theta}}}(\pi_j)$, proving Corollary~\ref{cor:ranking-preservation}.
\end{proof}

\subsection{Exact Recovery Under a Selection Margin}
\label{app:recovery-proof}

\begin{corollary}[Exact recovery under a selection margin]
\label{cor:exact-recovery}
Suppose $\Pi$ is finite and $\pi^\star$ is its unique best algorithm under the trusted dynamics. Define its selection margin as
\begin{equation}
\Gamma=J_P(\pi^\star)-\max_{\pi\in\Pi\setminus\{\pi^\star\}}J_P(\pi).
\label{eq:selection-margin}
\end{equation}
If $\Gamma>2\Delta_H(\epsilon)$ and $\eta_{\mathrm{search}}=0$, then
\begin{equation}
\operatorname*{arg\,max}_{\pi\in\Pi}J_{\widehat P_{\bm{\theta}}}(\pi)=\pi^\star.
\label{eq:exact-recovery}
\end{equation}
\end{corollary}

This strengthens ranking preservation to exact algorithm recovery: when the best algorithm is separated from all alternatives by more than the world-model distortion and the learned objective is optimized exactly, the Network World Model selects the same optimum as the trusted dynamics.

\begin{proof}
For every $\pi\in\Pi\setminus\{\pi^\star\}$, the definition of the selection margin in Eq.~(\ref{eq:selection-margin}) gives
\begin{align}
J_P(\pi^\star)-J_P(\pi)
&\geq\Gamma \nonumber\\
&>2\Delta_H(\epsilon).
\label{eq:app-recovery-margin}
\end{align}
By Corollary~\ref{cor:ranking-preservation}, Eq.~(\ref{eq:app-recovery-margin}) implies
\begin{equation}
J_{\widehat P_{\bm{\theta}}}(\pi^\star)>J_{\widehat P_{\bm{\theta}}}(\pi)
\qquad
\text{for every }\pi\in\Pi\setminus\{\pi^\star\}.
\label{eq:app-recovery-ranking}
\end{equation}
Therefore, $\pi^\star$ is the unique maximizer of the learned objective, and Eq.~(\ref{eq:exact-recovery}) follows.
\end{proof}

\subsection{Exact Intervention Factorization}
\label{app:exact-intervention-proof}

Theorem~\ref{thm:selection-regret} shows that algorithm-selection reliability depends on transition accuracy. Transition error is therefore separated into endogenous dynamics error and intervention-semantics error. Let $\bm{x}=(\mathcal{G},\bm{s})$ and let $K$ denote the true endogenous dynamics after intervention. The exact intervention factorization is compared with a generic surrogate that approximates the intervention:
\begin{equation}
\widehat P_{\bm{\theta}}^a(\cdot\mid\bm{x})=\widehat K_{\bm{\theta}}\left(\cdot\mid T_{\mathrm{exo}}^a\bm{x},\bm{\eta}(a)\right),\qquad \widetilde P_{\bm{\theta}}^a(\cdot\mid\bm{x})=\widehat K_{\bm{\theta}}\left(\cdot\mid\widetilde T^a\bm{x},\bm{\eta}(a)\right).
\label{eq:exact-generic-factorization}
\end{equation}
Here, $T_{\mathrm{exo}}^a$ is the exact intervention operator in Eq.~(\ref{eq:wm-factorized}), while $\widetilde T^a$ denotes an approximate intervention transformation.

\begin{theorem}[Exact intervention removes action-semantics error]
\label{thm:exact-intervention}
Suppose the learned endogenous dynamics have one-step approximation error at most $\epsilon_{\mathrm{dyn}}$. Then the exact intervention factorization in Eq.~(\ref{eq:exact-generic-factorization}) has rollout error
\begin{equation}
D_{\mathrm{TV}}\left(\mathbb{P}_{P}^{\pi},\mathbb{P}_{\widehat P_{\bm{\theta}}}^{\pi}\right)
\leq1-(1-\epsilon_{\mathrm{dyn}})^H.
\label{eq:exact-factorized-rollout}
\end{equation}
Further assume that the generic surrogate incurs intervention error $\epsilon_{\mathrm{exo}}$ and that the endogenous dynamics are $L_T$-Lipschitz with respect to this error. Then
\begin{equation}
D_{\mathrm{TV}}\left(\mathbb{P}_{P}^{\pi},\mathbb{P}_{\widetilde P_{\bm{\theta}}}^{\pi}\right)
\leq1-\left(1-\epsilon_{\mathrm{dyn}}-L_T\epsilon_{\mathrm{exo}}\right)^H,
\label{eq:generic-factorized-rollout}
\end{equation}
whenever $\epsilon_{\mathrm{dyn}}+L_T\epsilon_{\mathrm{exo}}\leq1$.
\end{theorem}

Because $T_{\mathrm{exo}}$ is applied exactly, the learned transition error contains only the endogenous dynamics term $\epsilon_{\mathrm{dyn}}$. A surrogate that must additionally infer how an action modifies the network can accumulate the extra term $L_T\epsilon_{\mathrm{exo}}$ before multi-step rollout begins.

\begin{proof}
Let $K$ denote the true endogenous transition kernel after intervention. Further assume that the learned endogenous dynamics have one-step error at most $\epsilon_{\mathrm{dyn}}$, that $K$ is $L_T$-Lipschitz over post-intervention network states under a metric $d_{\mathcal X}$, and that $d_{\mathcal X}(T_{\mathrm{exo}}^a\bm{x},\widetilde T^a\bm{x})\leq\epsilon_{\mathrm{exo}}$.

For the generic surrogate in Eq.~(\ref{eq:exact-generic-factorization}), the one-step transition error satisfies
\begin{align}
D_{\mathrm{TV}}\left(P^a(\cdot\mid\bm{x}),\widetilde P_{\bm{\theta}}^a(\cdot\mid\bm{x})\right)
&\leq D_{\mathrm{TV}}\left(K(\cdot\mid T_{\mathrm{exo}}^a\bm{x},\bm{\eta}(a)),K(\cdot\mid\widetilde T^a\bm{x},\bm{\eta}(a))\right) \nonumber\\
&\quad+D_{\mathrm{TV}}\left(K(\cdot\mid\widetilde T^a\bm{x},\bm{\eta}(a)),\widehat K_{\bm{\theta}}(\cdot\mid\widetilde T^a\bm{x},\bm{\eta}(a))\right) \nonumber\\
&\leq L_Td_{\mathcal X}\left(T_{\mathrm{exo}}^a\bm{x},\widetilde T^a\bm{x}\right)+\epsilon_{\mathrm{dyn}} \nonumber\\
&\leq\epsilon_{\mathrm{dyn}}+L_T\epsilon_{\mathrm{exo}}.
\label{eq:generic-factorized-one-step}
\end{align}
For the exact intervention factorization in Eq.~(\ref{eq:exact-generic-factorization}),
\begin{align}
D_{\mathrm{TV}}\left(P^a(\cdot\mid\bm{x}),\widehat P_{\bm{\theta}}^a(\cdot\mid\bm{x})\right)
&=D_{\mathrm{TV}}\left(K(\cdot\mid T_{\mathrm{exo}}^a\bm{x},\bm{\eta}(a)),\widehat K_{\bm{\theta}}(\cdot\mid T_{\mathrm{exo}}^a\bm{x},\bm{\eta}(a))\right) \nonumber\\
&\leq\epsilon_{\mathrm{dyn}}.
\label{eq:exact-factorized-one-step}
\end{align}

Fix any executable algorithm $\pi\in\Pi$. Conditional on two coupled trajectories agreeing through time $t$, they have the same history and therefore $\pi$ selects the same action. Applying maximal coupling at each transition and Eq.~(\ref{eq:exact-factorized-one-step}),
\begin{align}
D_{\mathrm{TV}}\left(\mathbb{P}_{P}^{\pi},\mathbb{P}_{\widehat P_{\bm{\theta}}}^{\pi}\right)
&\leq1-\prod_{t=0}^{H-1}(1-\epsilon_{\mathrm{dyn}}) \nonumber\\
&=1-(1-\epsilon_{\mathrm{dyn}})^H,
\end{align}
which proves Eq.~(\ref{eq:exact-factorized-rollout}).

Similarly, applying Eq.~(\ref{eq:generic-factorized-one-step}) and $\epsilon_{\mathrm{dyn}}+L_T\epsilon_{\mathrm{exo}}\leq1$,
\begin{align}
D_{\mathrm{TV}}\left(\mathbb{P}_{P}^{\pi},\mathbb{P}_{\widetilde P_{\bm{\theta}}}^{\pi}\right)
&\leq1-\prod_{t=0}^{H-1}\left(1-\epsilon_{\mathrm{dyn}}-L_T\epsilon_{\mathrm{exo}}\right) \nonumber\\
&=1-\left(1-\epsilon_{\mathrm{dyn}}-L_T\epsilon_{\mathrm{exo}}\right)^H,
\end{align}
which proves Eq.~(\ref{eq:generic-factorized-rollout}).
\end{proof}

\subsection{Tighter Selection Regret Under Exact Intervention}
\label{app:exact-intervention-regret-proof}

Combining Theorem~\ref{thm:exact-intervention} with Theorem~\ref{thm:selection-regret} reduces the transition-error term entering the selection-regret bound from $\epsilon_{\mathrm{dyn}}+L_T\epsilon_{\mathrm{exo}}$ to $\epsilon_{\mathrm{dyn}}$.

\begin{corollary}[Tighter selection regret under exact intervention]
\label{cor:exact-intervention-regret}
Under the conditions of Theorems~\ref{thm:selection-regret} and~\ref{thm:exact-intervention}, the exact intervention factorization satisfies
\begin{equation}
\operatorname{Reg}_P(\widehat{\pi};\Pi)\leq2B_R\left(1-(1-\epsilon_{\mathrm{dyn}})^H\right)+\eta_{\mathrm{search}}.
\label{eq:exact-factorized-regret}
\end{equation}
For the generic surrogate in Eq.~(\ref{eq:exact-generic-factorization}), the corresponding guarantee is
\begin{equation}
\operatorname{Reg}_P(\widetilde{\pi};\Pi)\leq2B_R\left[1-\left(1-\epsilon_{\mathrm{dyn}}-L_T\epsilon_{\mathrm{exo}}\right)^H\right]+\eta_{\mathrm{search}}.
\label{eq:generic-factorized-regret}
\end{equation}
\end{corollary}

Exact intervention thus removes the intervention-semantics term before it can accumulate across the rollout horizon, yielding a tighter regret upper bound whenever $L_T\epsilon_{\mathrm{exo}}>0$ under the nondegenerate conditions of Theorem~\ref{thm:exact-intervention}.

\begin{proof}
By Theorem~\ref{thm:selection-regret} and Eqs.~(\ref{eq:exact-factorized-one-step})--(\ref{eq:generic-factorized-one-step}),
\begin{equation}
\operatorname{Reg}_P(\widehat{\pi};\Pi)\leq2B_R\left(1-(1-\epsilon_{\mathrm{dyn}})^H\right)+\eta_{\mathrm{search}}.
\label{eq:proof-exact-factorized-regret}
\end{equation}
For the generic surrogate,
\begin{equation}
\operatorname{Reg}_P(\widetilde{\pi};\Pi)\leq2B_R\left[1-\left(1-\epsilon_{\mathrm{dyn}}-L_T\epsilon_{\mathrm{exo}}\right)^H\right]+\eta_{\mathrm{search}}.
\label{eq:proof-generic-factorized-regret}
\end{equation}
The difference between the two regret upper bounds satisfies
\begin{align}
&2B_R\left[1-\left(1-\epsilon_{\mathrm{dyn}}-L_T\epsilon_{\mathrm{exo}}\right)^H\right]-2B_R\left[1-(1-\epsilon_{\mathrm{dyn}})^H\right] \nonumber\\
&=2B_R\left[(1-\epsilon_{\mathrm{dyn}})^H-\left(1-\epsilon_{\mathrm{dyn}}-L_T\epsilon_{\mathrm{exo}}\right)^H\right] \nonumber\\
&\geq0.
\label{eq:proof-regret-gap}
\end{align}
The last inequality is strict when $L_T\epsilon_{\mathrm{exo}}>0$ under the nondegenerate conditions of Theorem~\ref{thm:exact-intervention}, which proves the result.
\end{proof}

\section{Network Algorithm Design Agent: Details}
\label{app:loop}

This appendix gives the parts of the design loop that Section~\ref{sec:design} summarizes: one round in pseudocode, what the agent's prompt contains, the full list of probes, the instructions the agent writes under and how one is chosen, and the checks a candidate passes before it is scored. Throughout, $\pi$ is the candidate written in the current round and $\hat{\pi}$ the current best algorithm, as in Section~\ref{sec:design}.

\begin{algorithm}[H]
\caption{One round of the design loop. $\hat{\pi}$ is the current best algorithm, $\pi$ the candidate written in this round, $\xi$ the seed drawn for it, and $f_{\bm{\theta}}$ the world model. The anchors are the library's strongest classical members scored on the instance (Appendix~\ref{app:library}), and the best of them is the bar named in every prompt.}
\label{alg:round}
\begin{algorithmic}[1]
\REQUIRE task instance; the library of classical algorithms; the world model $f_{\bm{\theta}}$; the number of rounds $G$
\STATE score the anchors on the instance; the best is the bar to beat
\STATE $\hat{\pi} \leftarrow$ none
\FOR{each of the $G$ rounds}
  \STATE choose an instruction (new mechanisms early, edits late; Appendix~\ref{app:operators})
  \STATE \textbf{probe turn:} the agent asks up to six ``what if'' questions about $\hat{\pi}$'s plan; $f_{\bm{\theta}}$ answers each with a rollout
  \STATE \textbf{write turn:} the agent writes $\pi$ with its mechanism line, given the instruction, the prompt of Table~\ref{tab:context}, and the probe answers
  \STATE validate $\pi$ (Appendix~\ref{sec:implementation}); on failure, return the error as a repair turn and \textbf{continue}
  \STATE \textbf{rollout:} draw a fresh seed $\xi$ and score $\pi$ and $\hat{\pi}$ on the same $n$ sampled trajectories of $H$ steps through $f_{\bm{\theta}}$
  \STATE feedback $\leftarrow$ rollout diagnostics, per-action credit, the realizations $\pi$ lost on (Section~\ref{sec:feedback})
  \IF{$\Delta > b$ in Eq.~(\ref{eq:band})}
    \STATE $\hat{\pi} \leftarrow \pi$
  \ENDIF
\ENDFOR
\RETURN $\hat{\pi}$
\end{algorithmic}
\end{algorithm}

\subsection{All Probes}
\label{app:probes}

\begin{table}[H]
\caption{The probe operations. Every probe is metered as evaluator work, and a turn may ask at most six. \texttt{frontier}, \texttt{robust(hidden)} and \texttt{gradient} exist on the forward-model evaluators, and \texttt{gradient} only for seeding.}
\begin{center}
\footnotesize
\setlength{\tabcolsep}{4pt}
\begin{tabular}{p{3.0cm}p{10.2cm}}
\toprule
Probe & Question it answers about the plan \\
\midrule
\texttt{drop}$(v)$ & The reward with and without $v$'s actions: $v$'s marginal contribution \\
\texttt{swap}$(u, v)$ & The reward if $v$ replaces $u$ \\
\texttt{add}$(v)$ & The marginal gain of one more budgeted action on an unselected $v$, at $k + 1$: where the headroom is \\
\texttt{best\_swap}$(v)$ & The best replacement for $v$ among the 200 highest-degree unselected nodes, with its delta and the runner-up: local search handed over as information, for the model to turn into a rule \\
\texttt{overlap}$(u, v)$ & The expected number of nodes the cascades of $u$ and $v$ both reach, $\sigma(u) + \sigma(v) - \sigma(u, v)$: whether two seeds are redundant, which \texttt{drop} cannot say \\
\texttt{horizon}$(T)$ & The spread curve runs out to $T \le 2H$: a plan that wins early and loses late against one that keeps growing \\
\texttt{frontier}$(t)$ & The expected frontier size after step $t$ and the nodes most likely spreading then: where the outbreak is when a cut lands \\
\texttt{region}$(\mathbb{U})$ & The probability mass the plan captures inside a node set $\mathbb{U}$ \\
\texttt{robust}$(\sigma)$ & The reward with every $p_{uv}$ multiplied by $\exp(\nu)$ with $\nu \sim \mathcal{N}(0, \sigma^2)$ \\
\texttt{robust(hidden)} & The reward with the transmission probabilities hidden from the model: whether the plan exploited numbers it would not have at deployment \\
\texttt{gradient}$(m)$ & The $m$ highest-gradient unselected nodes and the $m$ lowest-gradient selected ones \\
\texttt{resimulate}$(\mathbb{U}, i)$ & Source localization: the consistency of an arbitrary candidate source set $\mathbb{U}$ on episode $i$, with the nodes it over- and under-explains \\
\texttt{transmission}\newline$(u, v, t, i)$ & Cascade reconstruction: the kernel's probability that $v$ activates at $t + 1$ given the state the decoded history asserts at $t$, and whether $u$ is on that frontier \\
\bottomrule
\end{tabular}
\end{center}
\label{tab:probes}
\end{table}

The \texttt{gradient} probe asks the forward model which nodes the plan should have chosen. A sampled rollout is not differentiable, so the probe runs a mean-field one: the plan's seeds enter as a relaxed vector $\bm{x} \in [\epsilon, 1 - \epsilon]^N$ on the $t = 0$ action channel, and each step's predicted marginals $(\bm{m}_t, \bm{u}_t)$, the probabilities of being infected and of being on the frontier, are fed back as the next step's state instead of being sampled,
\begin{equation}
(\bm{m}_{t+1}, \bm{u}_{t+1}) = \sigma\bigl(f_{\bm{\theta}}(\mathcal{G}, \bm{m}_t, \bm{u}_t, \bm{1}_{t = 0}\,\bm{x})\bigr), \qquad \nabla_{\bm{x}} \sum_{v \in \mathbb{V}} m_{H,v}.
\label{eq:gradient}
\end{equation}
The sum is the expected final spread under this propagation, and one backward pass through the $H$ forward passes gives its derivative \emph{with respect to} every node's seed weight at once. For an unselected node, the derivative is the marginal value of seeding it; for a selected node, the marginal cost of dropping it. The probe returns the $m$ unselected nodes with the largest gradient and the $m$ selected nodes with the smallest, which the leave-one-out probe and credit feedback miss, since they only ablate nodes already in the plan. The seed vector is kept with $\epsilon$ inside $[0, 1]$ because the head clamps its probabilities before the logit, which zeroes the gradient of a node sitting exactly at 0 or 1. The answer is a ranking, not an optimum: the mean-field spread is a surrogate for the sampled reward and the derivative is local, so the agent is told where the headroom is and decides what to do with it. It exists only on the forward-model evaluators, since a Monte Carlo simulator has no derivative, and only under the seeding lever, whose action is the one continuous relaxation the head can take.

\subsection{Agent Context}
\label{app:context}

Table~\ref{tab:context} lists what the agent reads. The static part is sent once, in the opening turn, and the dynamic part is rebuilt for the write turn of every round. Six of its items are not described in Section~\ref{sec:design} and are defined here.

\emph{Search memory.} A list of design rules, at most fifty words, shown at the top of every prompt. After every scored round whose candidate differs from the current best, one extra call is shown the worse and the better program of the pair, ordered by the acceptance decision, with their diagnostics, and returns a hint of at most twenty words (what made the better one better, as a design rule) and the memory revised with that hint, keeping only the rules the evidence still supports~\citep{ye2024reevo}.

\emph{Attempts table.} One row per round: the instruction, the reward, the paired improvement $\Delta$, whether the candidate was accepted, its mechanism line, the reflection hint, and the agent's own forecast beside what happened. The forecast is a structured comment the agent writes under its mechanism line, stating the signed change it expects against the program it edited and its probability that the change clears the margin $b$; the run records the forecast's sign accuracy and its calibration against the realized $\Delta$, and the table shows the agent its own past forecasts next to the outcomes.

\emph{Population table.} Every scored candidate, accepted or not, with its reward on its own realization, its paired strength (the least-squares rank of Appendix~\ref{app:operators}), the seconds its plan took, its number of code lines, and its mechanism line.

\emph{Anchors and the reference comparison.} The anchors are the strongest classical members of the library for the task, marked in Table~\ref{tab:library_groups}. They are scored on the instance under the same world model, budget, horizon, and exogenous context as the candidates, and listed in the opening turn with their rewards and standard errors, best first, as the bar to beat. The best anchor's plan and rollout trace are kept, and the current best's diagnostics include a comparison against them: which actions the two plans share, where their traces diverge, and by how much.

\emph{Counterexample realizations.} When a candidate loses the paired comparison, the three sampled trajectories on which it lost by the most are described and appended to its diagnostics, so that they reach the agent as feedback and travel with the record whenever the candidate is later shown as a parent or an inspiration. The description gives the candidate's and the current best's count on that trajectory and, under seeding, the nodes the current best reached that the candidate did not and the candidate's seeds that produced no new infection there; under containment, the nodes the outbreak reached on the candidate's watch that the current best kept clean. On the inverse and forecast tasks, the counterexample is an instance, with both scores and the residual diagnostics. Acceptance still reads the next round's fresh realization, so a counterexample informs the next edit and never decides it.

\emph{Last verdict and repair.} The paired verdict on the agent's previous edit, so that a candidate that was not accepted still receives one, and the error text if the candidate failed validation, which is what the repair turn of Appendix~\ref{sec:implementation} answers.

\begin{table}[H]
\caption{The prompt of the write turn. The static part is sent once in the opening turn and the dynamic part is rebuilt every round.}
\begin{center}
\small
\setlength{\tabcolsep}{4pt}
\begin{tabular}{p{2.7cm}p{10.4cm}}
\toprule
Component & Content \\
\midrule
\multicolumn{2}{l}{\emph{Static (the opening turn)}} \\
task statement & family, objective and its direction, the diffusion model, $k$ (absolute and as a share of $N$), $H$, the allowed operations, the exogenous context (outbreak or rumor with its size, observation block, rounds of an adaptive task) \\
network profile & $N$, arc count, directedness, density, degree quantiles, $k$-core size, reciprocity, clustering, assortativity, communities and their sizes; no adjacency \\
contract and rules & the entry point and return shape, the two structured comment lines (mechanism and forecast), the import whitelist, the budget and validity rules, two worked exemplars at the level the search should start from \\
library & the callable API of the library without its simulation-based members, with the source of the classical members, and the withheld members named with the reason \\
anchor leaderboard & every anchor with its reward and standard error on this instance, the best first, ``beating the top row is the bar'' \\
probe contract & the probe operations of Table~\ref{tab:probes} and their costs \\
\midrule
\multicolumn{2}{l}{\emph{Dynamic (every round)}} \\
search memory & the design rules the evidence so far supports, at most fifty words, revised after every scored round \\
attempts table & one row per round: instruction, reward, $\Delta$, accepted, mechanism, reflection hint, and the agent's own forecast beside what happened \\
population table & every member with its reward on its own realization, its paired strength, plan seconds, code lines, and mechanism \\
instruction and programs & the instruction for this round; the current best $\hat{\pi}$ with its full diagnostics; the partner, the two strongest alternatives, or the chosen idea as the instruction needs \\
diagnostics of $\hat{\pi}$ & rollout diagnostics, paired verdict, the reference comparison against the best anchor, per-action credit, counterexample realizations (Section~\ref{sec:feedback}) \\
last verdict and repair & the paired verdict on the agent's previous edit, and the error text if it failed \\
probe answers & the probe turn's answers about $\hat{\pi}$, and any answers to questions asked beside the previous code block \\
\bottomrule
\end{tabular}
\end{center}
\label{tab:context}
\end{table}

\subsection{Operators}
\label{app:operators}

Each round the agent writes under one of the six instructions in Table~\ref{tab:operators}. Three of them explore, asking for a mechanism the population does not have, and three exploit, asking for an edit of the current best algorithm. The instruction text the agent receives for each is in Appendix~\ref{app:search}.

\begin{table}[H]
\caption{The six instructions. The two that need a partner are withheld while the population has one member; the instruction text for each is in Appendix~\ref{app:search}.}
\begin{center}
\small
\begin{tabular}{llp{8.6cm}}
\toprule
Operator & Category & Purpose \\
\midrule
\texttt{crossover} & Explore & Combine parts from a parent algorithm's mechanism with a partner algorithm's \\
\texttt{synthesize} & Explore & Write one new strategy taking the best-supported idea from each prior algorithm \\
\texttt{from\_scratch} & Explore & Implement a completely novel idea, chosen by the idea search below, as a new strategy (must not re-implement a library algorithm or something already in the population) \\
\midrule
\texttt{refine} & Exploit & One small targeted change: adjust one term or fix one weakness the diagnostics expose \\
\texttt{parameters} & Exploit & Only change numerical constants; keep the same functions and control flow \\
\texttt{simplify} & Exploit & Remove parts of the algorithm that the diagnostics do not justify; a shorter child algorithm is the goal \\
\bottomrule
\end{tabular}
\end{center}
\label{tab:operators}
\end{table}

\paragraph{Exploration and exploitation schedule.} Each round first decides whether to explore or exploit, then which instruction to use, so that exploration dominates early in the search and exploitation later. Let $r$ be the index of the round, $G$ the number of rounds, and $\mathrm{stall}_r$ the number of consecutive rounds before $r$ whose candidate was rejected under an exploit instruction or failed validation; the count resets when a candidate is accepted and after any explore move. The exploration weight $\lambda_r \in [0, 1]$ is the probability mass given to the three explore instructions as a group,
\begin{equation}
\lambda_r = \min\Bigl(1,\; \lambda_0\,\frac{G - r}{G} + \beta\,\mathrm{stall}_r\Bigr),
\label{eq:schedule}
\end{equation}
with the starting weight $\lambda_0 = 0.8$ and the mass a stall restores $\beta = 0.15$. Exploration fades as the rounds are spent, and stalling pushes back: two consecutive stalls saturate $\lambda_r$ at one and force an explore move. Within each group, an instruction $o$ is drawn in proportion to a fixed weight $w_o$, so the chance of drawing it is proportional to $w_o \lambda_r$ if it explores and to $w_o (1 - \lambda_r)$ if it exploits; \texttt{crossover} and \texttt{synthesize} are withheld while the population has fewer than two members. The decaying term is the budget-aware exploration constant of~\citet{zheng2025monte}, and the stall term is added to it.

\paragraph{Idea search before writing.} Under \texttt{from\_scratch} and \texttt{synthesize}, one extra call precedes the write turn. Its prompt asks for three distinct mechanisms, shows the task's library menu (one line per algorithm, marked as available to every candidate so that re-implementing one is not new), the mechanism line and reward of every population member, the search memory, and the attempts table, and asks for the three ideas, a novelty score from zero to ten for each (ten meaning the idea shares no mechanism with the library or the population), the chosen one, and a one-sentence reason. The chosen idea is spliced into the write turn's prompt, and the write turn implements it.

\paragraph{Parent and partner.} The program an exploit instruction edits is always the current best $\hat{\pi}$: never the latest attempt, and never the member with the highest raw reward. The distinction matters because each member's reward was measured on its own round's realization, so a rejected candidate that drew a generous realization can carry a raw reward above the current best's while having lost the paired comparison; ranking by raw reward would then hand \texttt{refine} the loser. The rest of the population is ranked by a paired strength: the weighted least-squares solution of the search's own comparisons, one equation $s(\pi) - s(\hat{\pi}) = \Delta$ per scored round with the first program's strength pinned at its reward and each equation weighted by $1 / b$, so that a member's rank rests on the differences that were measured on shared trajectories rather than on rewards measured on different ones. Under \texttt{crossover}, the partner is drawn from the rest of the population with probability proportional to $1 / (\mathrm{rank} + 2)$, the selection rule of~\citet{liu2024evolution}, and under \texttt{refine}, the two strongest other members are shown as alternatives.

\subsection{Implementation Details}
\label{sec:implementation}

For end-to-end coding-agent experiments, we fix the LLM configuration within each comparison and report it with the corresponding run; world-model-only experiments do not require an LLM. All Monte Carlo simulations use $n = 200$ episodes per candidate, and all runs using the oracle and the trained world models use $n = 200$ sampled rollouts. We use a base seed of $42$ for all experiments. All experiments were run on an NVIDIA RTX A6000 GPU with 48 GB of VRAM.

Budgets $k \in \{ 1, 5, 10, 20 \}$ percent of $N$ for all seeding, containment, and epidemic tasks; $k \in \{ 10, 20, 30, 40, 50 \}$ for influence blocking, the convention in the literature; and a single $k=10$ percent point for the recovery and forecast tasks, where $k$ is a property of the instance or unused. The outbreak is $10$ percent of $N$ for critical node detection and $1$ percent for influence blocking and epidemic control. Each execution of a generated program algorithm is limited to a variable wall-clock time, set to either $900$ seconds or $1800$ seconds on larger networks. More details can be seen at \url{https://rishabsa.github.io/NetworkWorldModel/}.

\paragraph{Validation.} Every candidate is validated before it is scored. Its source must pass an abstract syntax tree check that rejects dynamic execution, file primitives, and imports outside the fixed set of numerical and network libraries named in Appendix~\ref{app:library}. It is then executed in the sandbox of that appendix, with every call under the wall-clock limit above, and any exception, wrong return type, or timeout becomes a structured error whose text is returned to the agent as a repair turn, one further write turn that must fix the error before the round continues. Once the program returns a plan, the plan is checked against the feasible set $\mathbb{F}$ of Section~\ref{sec:problem}: it may use only the allowed operations, must return node IDs that are in range and exist, must stay within the budget, and may not remove a node of the outbreak.

\section{Callable Library and Network Primitives}
\label{app:library}

Every generated program runs in a sandbox whose namespace holds \texttt{ActionOp}, \texttt{State}, the read-only \texttt{GraphInfo} (node and arc counts, neighborhoods, degrees, and the transmission probability of every arc), the \texttt{Strategy} base class, the network primitives as \texttt{primitives}, and its own family's library as a module (\texttt{algorithms} for influence maximization, \texttt{dismantling\_algorithms} for critical node detection, and so on), with the source code of the classical members. The AST check rejects any import beyond \texttt{numpy}, \texttt{networkx}, \texttt{scipy}, and the standard math and container modules. We withhold members that select by running their own Monte Carlo simulation from generated programs, since a simulation inside the program would bypass the metered evaluator; the two simulation primitives stay callable, run the NDlib simulator~\citep{rossetti2018ndlib} inside the program, and count against the per-call wall-clock cap.

Every intervention member takes \texttt{(network, budget, diffusion\_model)} plus the family's exogenous context as a keyword (\texttt{negative\_seeds=} for blocking, \texttt{outbreak=} for containment and immunization) and returns \texttt{budget} node IDs, or \texttt{(u, v)} arcs under an edge lever; a per-round adaptive policy takes \texttt{(state, network, batch, diffusion\_model)} and returns at most \texttt{batch} seeds for one round. A localizer takes \texttt{(network, observation, budget)} and returns \texttt{budget} node IDs; a decoder takes \texttt{(network, observation, horizon)} and returns \texttt{\{node: (activation timestep, parent)\}} with \texttt{parent = None} for a source; a predictor takes \texttt{(network, observation, horizon, fit\_examples=None)} and returns the popularity at the horizon, or \texttt{None} to decline. Optional keyword arguments are the published method's own knobs at the paper's defaults. Table~\ref{tab:library_groups} lists every callable member by the mechanism it belongs to.

\begingroup\footnotesize
\setlength{\tabcolsep}{4pt}
\begin{longtable}{p{3.6cm}p{9.6cm}}
\caption{The callable library by mechanism. Counts are callable members.}
\label{tab:library_groups} \\
\toprule
Group & Members \\
\midrule
\endfirsthead
\toprule
Group & Members \\
\midrule
\endhead
\bottomrule
\endfoot
\multicolumn{2}{p{13.2cm}}{\textbf{Network primitives} (\texttt{primitives}, 16)} \\
degree and centrality (6) & \texttt{compute\_degree}, \texttt{compute\_out\_degree}, \texttt{compute\_weighted\_degree}, \texttt{get\_top\_degree\_nodes}, \texttt{compute\_pagerank}, \texttt{compute\_centrality} \\
communities and budget (2) & \texttt{detect\_communities} (label propagation), \texttt{allocate\_budget} (largest-remainder quota per community) \\
path influence (1) & \texttt{path\_influence\_scores} (truncated path products) \\
reverse influence sampling (3) & \texttt{batch\_reverse\_sample}, \texttt{ris\_select}, \texttt{estimate\_sample\_size} \\
live-edge realizations (2) & \texttt{sample\_live\_edge\_graph}, \texttt{reachable\_count} \\
in-program simulation (2) & \texttt{mc\_simulate\_spread}, \texttt{compute\_marginal\_gain} (NDlib inside the program, under the wall-clock cap) \\
\midrule
\multicolumn{2}{p{13.2cm}}{\textbf{Influence maximization} (\texttt{algorithms}, 25)} \\
centrality rankings (7) & \texttt{high\_degree}, \texttt{weighted\_degree}, \texttt{pagerank\_seeds}, \texttt{eigenvector\_seeds}, \texttt{closeness\_seeds}, \texttt{betweenness\_seeds}, \texttt{kshell\_seeds} \\
discounted and iterative rankings (4) & \texttt{degree\_discount}~\citep{chen2009efficient}, \texttt{voterank}~\citep{zhang2016identifying}, \texttt{collective\_influence}~\citep{morone2015influence}, \texttt{irie}~\citep{jung2012irie} \\
path-based influence (3) & \texttt{sp1m}~\citep{kimura2006tractable}, \texttt{mia\_pmia}~\citep{chen2010scalable} and \texttt{ldag}~\citep{chen2010scalablelt} \\
reverse influence sampling (5) & \texttt{ris\_basic}~\citep{borgs2014maximizing}, \texttt{tim} and \texttt{imm}~\citep{tang2014influence,tang2015influence}, \texttt{ssa}~\citep{nguyen2016stop}, \texttt{filtered\_ris} \\
sketches and communities (5) & \texttt{skim}~\citep{cohen2014sketch}, \texttt{community\_im}, \texttt{cofim}~\citep{shang2017cofim}, \texttt{community\_ris}, \texttt{degree\_ris\_refine} \\
floor (1) & \texttt{random\_seeds} \\
\midrule
\multicolumn{2}{p{13.2cm}}{\textbf{Adaptive influence maximization} (\texttt{adaptive\_algorithms}, 6; one round per call)} \\
sampling per round (1) & \texttt{adapt\_epic} (EPIC,~\citealp{han2018efficient}) \\
centralities per round (3) & \texttt{adapt\_degree}, \texttt{adapt\_degree\_discount}, \texttt{adapt\_pagerank}, each restricted to the susceptible nodes of the realized state \\
controls (2) & \texttt{static\_split} (one static seed set dealt out per round), \texttt{adapt\_random} \\
\midrule
\multicolumn{2}{p{13.2cm}}{\textbf{Critical node detection} (\texttt{dismantling\_algorithms}, 24)} \\
one-pass centralities (4) & \texttt{degree\_removal}, \texttt{pagerank\_removal}, \texttt{betweenness\_removal}, \texttt{kshell\_removal} \\
adaptive centralities (4) & \texttt{adaptive\_degree} (HDA,~\citealp{holme2002attack}), \texttt{iterative\_betweenness} and \texttt{approx\_iterative\_betweenness}~\citep{wandelt2018comparative}, \texttt{collective\_influence\_removal}~\citep{morone2015influence} \\
decycling and message passing (3) & \texttt{corehd}~\citep{zdeborova2016fast}, \texttt{decycling}, \texttt{bpd}~\citep{mugisha2016identifying} \\
with the published reinsertion pass (5) & \texttt{collective\_influence\_r}, \texttt{corehd\_r}, \texttt{decycling\_r}, \texttt{bpd\_r}, \texttt{gndr}; separate rows because each pair is cited under one name \\
spectral and partition cuts (4) & \texttt{gnd} and \texttt{egnd}~\citep{ren2019generalized}, \texttt{netshield}~\citep{tong2010vulnerability}, \texttt{articulation\_removal} \\
percolation (1) & \texttt{explosive\_immunization}~\citep{clusella2016immunization} \\
controls (3) & \texttt{frontier\_removal} (the susceptible ring of the outbreak, highest degree first), \texttt{acquaintance\_immunization}~\citep{cohen2003efficient}, \texttt{random\_removal} \\
\midrule
\multicolumn{2}{p{13.2cm}}{\textbf{Influence blocking} (\texttt{blocking\_algorithms}, 19; node levers return nodes, edge levers return arcs)} \\
proximity to the rumor (3) & \texttt{proximity} (the baseline of~\citealp{he2012influence}), \texttt{multi\_hop\_proximity}, \texttt{imin\_lhga} (SandIMIN,~\citealp{wang2024efficient}) \\
centralities in the reachable region (3) & \texttt{degree\_blocking} (reported to fail at this task by~\citealp{he2012influence}), \texttt{pagerank\_blocking}, \texttt{betweenness\_blocking} \\
sampling-based (4) & \texttt{forward\_blocking} and \texttt{reverse\_blocking} (TC-AIBM,~\citealp{shi2025time}), \texttt{rps}~\citep{tong2017efficient}, \texttt{imin\_lsbm} (SandIMIN,~\citealp{wang2024efficient}) \\
model-based greedy (4) & \texttt{cmia\_o}~\citep{wu2017scalable}, \texttt{cldag}~\citep{he2012influence}, \texttt{advanced\_greedy} and \texttt{greedy\_replace}~\citep{xie2025influence} \\
edge levers (4) & \texttt{kimura\_link\_blocking}~\citep{kimura2008minimizing}, \texttt{out\_edge\_blocking}, \texttt{edge\_betweenness\_blocking}, \texttt{random\_edge\_blocking} \\
floor (1) & \texttt{random\_blocking} \\
\midrule
\multicolumn{2}{p{13.2cm}}{\textbf{Epidemic control} (\texttt{immunization\_algorithms}, 22; node levers return doses, edge levers return arcs)} \\
degree and centrality doses (7) & \texttt{degree\_immunization}~\citep{pastor2002immunization}, \texttt{adaptive\_degree\_immunization}, \texttt{acquaintance\_immunization}~\citep{cohen2003efficient}, \texttt{pagerank\_immunization}, \texttt{eigenvector\_immunization}, \texttt{betweenness\_immunization}, \texttt{kshell\_immunization} \\
spectral (4) & \texttt{netshield} and \texttt{netshield\_plus}~\citep{tong2010vulnerability,chen2015node}, \texttt{greedy\_walk}~\citep{saha2015approximation}, \texttt{preciado\_allocation}~\citep{preciado2014optimal} \\
data-aware (3) & \texttt{dava} and \texttt{dava\_fast}~\citep{zhang2015data}, \texttt{frontier\_immunization} \\
edge levers (7) & \texttt{netmelt}~\citep{tong2012gelling}, \texttt{product\_degree} and \texttt{eigen\_score}~\citep{van2011decreasing}, \texttt{greedy\_walk\_edge}, \texttt{edge\_betweenness\_cut}, \texttt{frontier\_edge\_cut}, \texttt{random\_edge\_cut} \\
floor (1) & \texttt{random\_immunization} \\
\midrule
\multicolumn{2}{p{13.2cm}}{\textbf{Source localization} (\texttt{localization\_algorithms}, 13; a per-node scorer of the same name for each in \texttt{localization\_scorers})} \\
label propagation and spectral (3) & \texttt{lpsi}~\citep{wang2017multiple}, \texttt{netsleuth}~\citep{prakash2012spotting}, \texttt{dynamic\_age}~\citep{fioriti2012predicting} \\
distance and centre estimators (4) & \texttt{rumor\_centrality}~\citep{shah2011rumors}, \texttt{jordan\_center}~\citep{zhu2014information}, \texttt{ojc}~\citep{zhu2017catch}, \texttt{effective\_distance}~\citep{brockmann2013hidden} \\
message passing (1) & \texttt{dmp\_localize}~\citep{lokhov2014inferring} \\
infected-subgraph centralities (4) & \texttt{infected\_degree}, \texttt{infected\_betweenness}, \texttt{infected\_closeness}, \texttt{infected\_eigenvector}~\citep{comin2011identifying} \\
floor (1) & \texttt{random\_sources} \\
\midrule
\multicolumn{2}{p{13.2cm}}{\textbf{Cascade reconstruction} (\texttt{reconstruction\_algorithms}, 16; every decoder shares one parent rule)} \\
ordered Steiner trees (4) & \texttt{delayed\_bfs}, \texttt{ordered\_steiner\_closure}, \texttt{greedy\_ordered}~\citep{xiao2018reconstructing}, \texttt{steiner\_tree} \\
sampling and random walks (2) & \texttt{tree\_sampling}~\citep{xiao2018robust}, \texttt{personalized\_pagerank} \\
consistent trees (3) & \texttt{consistent\_tree\_wpct} and \texttt{consistent\_tree\_wbct}~\citep{zong2012inferring}, \texttt{cult}~\citep{rozenshtein2016reconstructing} \\
likelihood-based (3) & \texttt{dhrec}~\citep{sefer2016diffusion}, \texttt{cri}~\citep{chen2016detecting}, \texttt{netfill}~\citep{sundareisan2015hidden} \\
controls (4) & \texttt{jordan\_backward}, \texttt{observed\_only} (the reports alone), \texttt{one\_hop}, \texttt{random\_reconstruction} \\
\midrule
\multicolumn{2}{p{13.2cm}}{\textbf{Cascade prediction} (\texttt{prediction\_algorithms}, 16; \texttt{None} declines and is counted)} \\
feature-based (4) & \texttt{szabo\_huberman}~\citep{szabo2010predicting}, \texttt{feature\_linear}, \texttt{feature\_gbt} (the features of~\citealp{cheng2014can}), \texttt{weng\_communities}~\citep{weng2013virality} \\
generative point processes (6) & \texttt{seismic}~\citep{zhao2015seismic}, \texttt{hawkes} and \texttt{hawkes\_hybrid}~\citep{mishra2016feature}, \texttt{rpp}~\citep{shen2014modeling}, \texttt{hip}~\citep{rizoiu2017expecting}, \texttt{branching\_factor} \\
graph-only (3) & \texttt{neighborhood\_size}, \texttt{degree\_scaled}, \texttt{reachability} (the structural ceiling) \\
floors (3) & \texttt{persistence}, \texttt{mean\_size} (the geometric mean of training sizes), \texttt{random\_prediction} \\
\end{longtable}
\endgroup

\section{Prompts of the Design Agent}
\label{app:search}

Every prompt is assembled by \texttt{coding\_agent/prompts.py}; the texts below are the load-bearing parts, verbatim where short and condensed where a block runs to pages (the worked exemplars and the per-family briefs). The gateway drops the system role, so the system text is folded into the first user turn.

\paragraph{System prompt (intervention families).} The task brief, then:
\begin{quote}\small\ttfamily
OUTPUT FORMAT: the first line inside your fenced block MUST be `\# MECHANISM: <one sentence>` naming the idea of the algorithm; from the second attempt on, the second line is `\# EXPECTED: <signed change you expect against the program you edited> p=<0-1>`. Reply with exactly ONE fenced python block and nothing else, no prose before or after. The block contains import lines (if you need any) and then exactly ONE class subclassing `Strategy`. Nothing else at module level.\\[2pt]
IMPORTS: you MAY import any of numpy, networkx, scipy, math, random, statistics, heapq, bisect, collections, itertools, functools. Use numpy for anything you would otherwise write as a Python loop over all nodes. Importing anything else is rejected.\\[2pt]
AVAILABLE NAMES (already in your script's namespace: do NOT import these): `ActionOp(op, target, destination=None, weight=None)` [...]; `State` [...]; `GraphInfo`: .num\_nodes, .out\_neighbors(node), .in\_neighbors(node), .degree(node), .edge\_index, .ic\_probs; the `algorithms`, `adaptive\_algorithms` and `primitives` modules (API below).\\[2pt]
ACTION RULES: A seed is ActionOp("add\_node", node). Emit at most `budget` add\_node actions in total. Seeding the same node twice is REJECTED. Node ids must be in [0, num\_nodes). [...]\\[2pt]
METHOD: ONE-SHOT SUPER-ALGORITHM. Implement `plan\_horizon(self, graph, budget, horizon) -> list[list[ActionOp]]`. Return a list of length (horizon+1): element t is the action bag applied at timestep t. This is your whole multi-timestep plan, decided up front.
\end{quote}
followed by two worked exemplars of the family, the timing note for the family (when to spend the budget), and the remove-semantics note when the task may emit removals. Containment replaces the action rules with the removal rules (at most \texttt{budget} \texttt{remove\_node} actions, the harness expands each into a deletion, removing a source is rejected, \texttt{add\_node} is rejected); the recovery and forecast families replace the whole block with their own contract text.

\paragraph{Opening user turn.} The task line (task, objective and direction), \texttt{diffusion\_model}, \texttt{budget} (absolute and as a percentage of nodes, with its unit), \texttt{horizon}, \texttt{allowed\_ops}; the outbreak, rounds, observation, mask or cascade block the family needs; the graph profile of Table~\ref{tab:context}; the library API reference with the source of the classical members and the blocked names; the anchor leaderboard (\texttt{REFERENCE SCORES: ... run on THIS graph, under THIS evaluator, at the same budget and horizon. Beating the top row is the bar}); the probe contract; and \texttt{Write the Strategy now}.

\paragraph{Generation turn.}
\begin{quote}\small\ttfamily
You are evolving a population of strategies.\\
SEARCH MEMORY (rules the evidence so far supports): [$M_g$]\\
ATTEMPTS SO FAR (every generation of this search): [attempts table]\\
POPULATION (what survives): [population table]\\
YOUR LAST EDIT: [paired verdict]\\
PARENT: the best in the population, which is NOT necessarily your last attempt (reward=..., mechanism: ...): [script] Parent rollout diagnostics: [feedback bundle]\\
{[PARTNER / ALTERNATIVE / STRATEGY blocks, or CHOSEN IDEA, per the operator]}\\
OPERATION: [instruction below]\\
The first line inside your code block must be `\# MECHANISM: <one sentence>` naming the idea of the strategy. The second line must be `\# EXPECTED: <signed change in reward you expect against the program you edited, in the task's units> p=<your probability, 0 to 1, that the change clears the noise band>`; the realized delta is compared with it and your calibration is reported back to you. Reply with one python block.\\
{[PROBE RESULTS blocks]}
\end{quote}

\paragraph{Operator instructions.}
\begin{quote}\small\ttfamily
REFINE: make a SMALL, targeted improvement to the PARENT: adjust one term, fix one weakness the diagnostics expose, keep its overall approach.\\
PARAMETERS: change ONLY numeric constants of the PARENT (weights, thresholds, sample counts, radii). No structural change: same functions, same control flow, different numbers, chosen from what the diagnostics say.\\
SIMPLIFY: remove components of the PARENT that the diagnostics do not justify: a term, a stage, a special case, a parameter. The child must be SHORTER and score within the noise band of the parent; that counts as a success here, because an algorithm a reader can follow is the goal.\\
CROSSOVER: combine the PARENT's mechanism with the PARTNER's into one strategy that keeps the parent's contract. The child must contain an identifiable part of each; it is not a re-tune of either.\\
SYNTHESIZE: read every strategy shown and write ONE new strategy that takes the best-supported idea from each, according to their diagnostics. It must differ from every strategy shown and from every library algorithm.\\
FROM SCRATCH: do NOT edit any program shown. Implement the CHOSEN IDEA as a new strategy under the same contract. It must not be a re-implementation of a library algorithm or of a mechanism already in the population.
\end{quote}

\paragraph{Idea search.}
\begin{quote}\small\ttfamily
Before writing code, propose 3 DISTINCT mechanisms for a new strategy on this task, then choose the one most different from everything below that you still expect to score well.\\
LIBRARY ALGORITHMS (already available to everyone; re-implementing one is not new): [menu]\\
MECHANISMS ALREADY IN THE POPULATION: [one line per member with reward]\\
SEARCH MEMORY: [$M_g$] \quad ATTEMPTS SO FAR: [attempts table]\\
Reply with JSON only: \{"ideas": ["<one sentence each>", ...], "novelty": [<0-10 per idea: 10 = shares no mechanism with the library or the population>], "chosen": <index into ideas>, "why": "<one sentence>"\}
\end{quote}

\paragraph{Reflection.}
\begin{quote}\small\ttfamily
Two strategies for the same task were scored on the same realization ([direction], unit: [unit]). The second scored better.\\
{[Worse: reward, mechanism, script, diagnostics]} \quad [Better: reward, mechanism, script, diagnostics]\\
Prior memory of this search (may be empty): [$M_g$]\\
Reply with JSON only: \{"hint": "<at most 20 words: what made the better one better, as a design rule>", "memory": "<at most 50 words: the prior memory revised with this hint; keep only rules the evidence still supports>"\}
\end{quote}

\paragraph{Probe turn.}
\begin{quote}\small\ttfamily
PROBE TURN before your next [operator] edit. The program you will edit is the incumbent (iteration [g], reward [...], mechanism: [...]). Its diagnostics: [feedback bundle]\\
You may ask up to 6 what-if questions about THAT plan; the answers arrive in the prompt for this generation. Every probe costs evaluator work. Ops: [the operations of Table~\ref{tab:probes} with their JSON shapes]\\
Reply with exactly ONE fenced probes block \{"probes": [...]\} or the single word NONE if you have nothing to ask. No code in this reply.
\end{quote}
The probe contract in the opening turn lists the same operations with their JSON shapes and states that a probes block may also be placed beside a code block, answered with the next generation's feedback, and that every probe is metered like a rollout. Under the native condition, the contract is absent, and a block is answered with: \texttt{probes are unavailable under the native condition: no forward model may be consulted during this search}.

\section{Experimental Protocol}
\label{app:protocol}

\subsection{Rollout Evaluators}
\label{app:evaluators}

Every evaluator exposes one operation, $\tau = \mathrm{rollout}(\pi, H, k, \xi)$, which executes $\pi$'s plan $\bm{a}$ under seed $\xi$ and returns a trajectory $\tau$ carrying the per-trajectory scores $R_i(\pi;\xi)$.

\paragraph{Monte Carlo.} $n$ independent NDlib episodes~\citep{rossetti2018ndlib} per call (each call's seed is different per episode). Each stops early once the frontier node set and the action bag are both empty. The total cost of the Monte Carlo evaluator grows with $n$ times the number of candidates.

\paragraph{Oracle.} The environment whose per-edge transmission is pinned to the truth: $q_{uv} = p_{uv}$ under IC, the closed-form threshold hazard under LT, the simulator's own rates under the compartmental dynamics, and both campaigns' probabilities composed through the tie-break under competition. Nothing is learned in the oracle: it reproduces the NDlib process exactly, in a vectorized implementation. The oracle is the ceiling condition (an agent with perfect dynamics) and the shared referee for all final number evaluations and comparisons.

\paragraph{Network World Model.} One rollout advances $n$ sampled states. Each timestep, the algorithm's policy is queried once per active sample, so one forward pass yields activation probabilities. One uniform draw per (sample, node) against the \emph{frontier} channel decides the new activations. A sample deactivates when its frontier node set and action bag are empty, and the call ends when all are inactive; dead samples hold their final count so per-sample curves stay aligned. Its online cost is a batched learned forward pass per timestep.

\subsection{Tasks}
\label{app:tasks}

\begin{table}[H]
\caption{The eight tasks on complex systems, summarized.}
\label{tab:tasks}
\centering
\small
\setlength{\tabcolsep}{4pt}
\begin{tabularx}{\textwidth}{@{}>{\raggedright\arraybackslash}p{3.1cm}l>{\raggedright\arraybackslash}X>{\raggedright\arraybackslash}p{2.6cm}@{}}
\toprule
Task & Dynamics & Task output & Reported metric \\
\midrule
Influence maximization & IC, LT & Seed set & Final spread ($\uparrow$) \\
Adaptive influence maximization & IC, LT & Per-round seeding policy & Final spread ($\uparrow$) \\
\addlinespace
Critical node detection & IC, LT & Node removals against an outbreak & Final infected ($\downarrow$) \\
Influence blocking & IC, CLT & Counter-seeds, node or arc blocks, or weight cuts against a rumor & Final rumor size ($\downarrow$) \\
Epidemic control & SIR, SIS & Vaccinations, quarantines, arc cuts, or contact reductions & Attack rate ($\downarrow$) \\
\addlinespace
Source localization & IC, LT & Source set & Consistency with the observation ($\uparrow$) \\
Cascade reconstruction & IC, LT & Trajectory with parents & Referee reward of the history ($\uparrow$) \\
\addlinespace
Cascade prediction & Logged cascades & Popularity at the prediction horizon & MSLE ($\downarrow$) \\
\bottomrule
\end{tabularx}
\end{table}

Every task shares the transition of Eq.~(\ref{eq:wm-factorized}) and the evaluators of Appendix~\ref{app:evaluators}; what differs is the contract a program implements, what the harness injects as $\Omega$, and the reward. In every case, the reward is computable without a label, and the label metrics of the recovery and forecast families are computed only after the search, on the winner.

\paragraph{Influence maximization (maximize).} $\pi$ implements \texttt{plan\_horizon}$(\mathcal{G}, k, H)$ and returns the plan of actions $\bm{a}$; the budget buys $k$ \texttt{add\_node} seeds; the reward is the final spread $|\mathbb{I}_H|$.

\paragraph{Adaptive influence maximization (maximize).} Same dynamics, same data; the algorithm is a per-round policy \texttt{act}$(\bm{s}_t, \mathcal{G}, t)$ called at each of $\rho$ round boundaries with the state the previous batch's diffusion produced, committing batches that sum to $k$, with $g$ timesteps of diffusion between rounds and a feedback model (full adoption or myopic) deciding what the policy sees. The registry pairs every adaptive arm with the matching static arm so the adaptivity gap, $\mathrm{spread}(\text{adaptive}) / \mathrm{spread}(\text{static})$ at matched $k$ on the referee, is always computable.

\paragraph{Critical node detection (minimize).} The harness seeds an outbreak $\Omega = \mathbb{S}$ of $10$ percent of $N$ at $t = 0$; the budget buys $k$ \texttt{remove\_node} deletions, each expanded by the harness into the node and its incident arcs; the reward is the final infected count. Removing a source is rejected (deleting patient zero ends an outbreak rather than containing it), and because the planner is told the outbreak, the budget must sit below the one-hop ring $|N_1(\mathbb{S}) \backslash \mathbb{S}|$, at which every outbreak-aware plan scores exactly $|\mathbb{S}|$.

\paragraph{Influence blocking (minimize).} A rumor $\mathbb{S}_N$ is committed at $t = 0$ and already spreading on a two-cascade simulator; the lever decides what the budget buys: counter-seeds (\texttt{add\_node} seeds the positive cascade), node blocks, arc blocks, or weight cuts. The reward is the rumor's final size, and the report also carries prevented influence $\mathrm{spread}(\mathbb{S}_N, \emptyset) - \mathrm{spread}(\mathbb{S}_N \mid \text{blockers})$. The tie-break between the two cascades, the positive campaign's probabilities, and the detection delay are protocol parameters recorded with every run. 

\paragraph{Epidemic control (minimize).} Index cases are seeded at $t = 0$ under SIR and SIS on a stepper with a per-arc transmission probability; the lever maps the budget onto vaccination (immune, uncounted), quarantine (isolated, still counted), arc cuts, or graded contact reduction through \texttt{set\_edge\_weight}. The reward is the attack rate, the number of nodes ever infected; the report carries the prevalence curve and the spectral eigendrop as context.

\paragraph{Source localization (recover).} $\pi$ implements \texttt{localize}$(\mathcal{G}, \bm{y}, k)$ and names the $k$ nodes it believes started the cascade whose observed state $\bm{y} \in [0, 1]^N$ it is shown; $k$ is a property of the instance. The reward is label-free consistency; we compute precision, recall, $F_1$, and AUC against the stored sources only after the search, on both the selection and a disjoint evaluation pool.

\paragraph{Cascade reconstruction (recover).} $\pi$ implements \texttt{reconstruct}$(\mathcal{G}, o, H)$ and returns $\{v \mapsto (t_v, \mathrm{parent}_v)\}$ for every node it believes was infected, from a partial observation $o$ of a diffusion that already happened (partial times, partial nodes, final snapshot, or hidden nodes). The reward is the arm's kernel log-likelihood of the asserted history per node minus the observation it contradicts; after the search, we compute path precision, event $F_1$, and timing error against the stored history.

\paragraph{Cascade prediction (forecast).} The only task with no simulator: transitions are replayed from a real corpus, $\pi$ implements \texttt{predict}$(\mathcal{G}, o, t_p)$ and returns the popularity a real cascade will have reached by $t_p$ from its first $t_o$ steps, and may return \texttt{None} to decline. The reward minimizes mean squared log error; declines are counted but never scored. The split is chronological by default, because the field's random split leaks.

\subsection{Datasets}
\label{app:datasets}

\begin{table}[H]
\centering
\small
\setlength{\tabcolsep}{5pt}
\begin{tabularx}{\textwidth}{@{}>{\raggedright\arraybackslash\hsize=1.1\hsize}Xrr>{\raggedright\arraybackslash\hsize=0.9\hsize}X@{}}
\toprule
Dataset & Nodes & Edges & Tasks \\ \midrule
Email-EU~\citep{yin2017local, leskovec2007graph}                     & 1,005     & 24,929     & Influence Blocking \\
UCI Students~\citep{opsahl2009clustering}                            & 1,266     & 6,451      & Cascade Reconstruction \\
Network Science~\citep{newman2006finding}                            & 1,589     & 2,742      & Influence Maximization, Adaptive Influence Maximization \\
Cora-ML~\citep{mccallum2000automating, bojchevski2018deep}           & 2,810     & 7,981      & Source Localization \\
Power Grid~\citep{watts1998collective}                               & 4,941     & 6,594      & Critical Node Detection, Source Localization \\
CA-GrQc~\citep{leskovec2007graph}                                    & 5,242     & 14,484     & Cascade Reconstruction \\
Oregon1~\citep{leskovec2005graphs}                                   & 10,670    & 22,002     & Epidemic Control \\
PGP~\citep{boguna2004models}                                         & 10,680    & 24,316     & Critical Node Detection \\
Infectious SocioPatterns~\citep{isella2011whats}                     & 10,972    & 44,517     & Epidemic Control \\
NetHEPT~\citep{chen2009efficient}                                    & 15,229    & 62,752     & Influence Maximization, Adaptive Influence Maximization \\
RT-Pol~\citep{conover2011political, rossi2015network}               & 18,470    & 48,053     & Cascade Reconstruction \\
Gnutella24~\citep{ripeanu2002mapping, leskovec2007graph}             & 26,518    & 65,369     & Influence Blocking \\
Cit-HepTh~\citep{leskovec2005graphs, gehrke2003overview}             & 27,769    & 352,768    & Influence Blocking \\
Taoke~\citep{peng2026leakage}                                        & 29,711    & 95,012     & Cascade Prediction \\
Deezer~\citep{rozemberczki2019gemsec}                                & 47,538    & 222,887    & Source Localization \\
Brightkite~\citep{cho2011friendship}                                 & 58,228    & 214,078    & Epidemic Control \\
Gnutella31~\citep{ripeanu2002mapping, leskovec2007graph}             & 62,561    & 147,878    & Critical Node Detection \\
Digg~\citep{zafarani2009social}                                      & 116,893   & 2,011,447  & Influence Maximization, Adaptive Influence Maximization \\
Digg (cascades)~\citep{lerman2010information}                        & 279,630   & 1,731,653  & Cascade Prediction \\
APS~\citep{shen2014modeling, xu2021casflow}                          & 616,316   & 3,304,400  & Cascade Prediction \\
\bottomrule
\end{tabularx}
\caption{The datasets used across the eight tasks, ordered by node count. For the three cascade-prediction corpora, the counts describe the full underlying network, while the runs replay a subsample restricted to the most active participants.}
\label{tab:datasets}
\end{table}

\subsection{Baselines}
\label{app:baselines}

Table~\ref{tab:baselines} lists every baseline that we compare against: five per task, and four for adaptive influence maximization. Entries marked library are our implementations of the published algorithms, and entries marked published code are the authors' released implementations, run in their own environment. Every baseline runs on the same network, budget, and rumor or outbreak, and its returned solution is scored on the shared referee. Learned approaches to these tasks that are not run here include graph neural influence maximization~\citep{panagopoulos2023maximizing} and deep cascade popularity models~\citep{cao2017deephawkes,xu2021casflow}.

\begin{small}
\setlength{\tabcolsep}{3pt}
\begin{longtable}{@{}>{\raggedright\arraybackslash}p{1.9cm}>{\raggedright\arraybackslash}p{2.5cm}>{\raggedright\arraybackslash}p{6.2cm}>{\raggedright\arraybackslash}p{2.4cm}@{}}
\caption{The baseline methods we compare against, what each does, and where it comes from.}
\label{tab:baselines} \\
\toprule
Task & Method & What it does & Reference \\
\midrule
\endfirsthead
\toprule
Task & Method & What it does & Reference \\
\midrule
\endhead
\bottomrule
\endlastfoot
Influence maximization & IMM & Reverse influence sampling with a martingale stopping rule; the standard RIS reference (library) & \citep{tang2015influence} \\
 & OPIM & Online-processing RIS that tightens its bound as samples arrive (published code) & \citep{tang2018online} \\
 & SubSIM & Sublinear-time reverse reachable set generation with tightened bounds (published code) & \citep{guo2020influence} \\
 & DeepIM & Learned seed-set generator with a graph encoder and a spread predictor, trained per network (published code) & \citep{ling2023deep} \\
 & DegreeDiscount & Degree ranking with a discount for neighbors already seeded; the anchor (library) & \citep{chen2009efficient} \\
\midrule
Adaptive influence maximization & EPIC & AdaptGreedy instantiated with reverse influence sampling per round; the anchor (library) & \citep{han2018efficient} \\
 & Adaptive DegreeDiscount & DegreeDiscount recomputed each round on the nodes still susceptible (library) & \citep{chen2009efficient} \\
 & IMM & The non-adaptive seed set, committed in full at $t = 0$ (library) & \citep{tang2015influence} \\
 & Static-Split & One static seed set dealt out over the rounds; the control for the round machinery & this work \\
\midrule
Critical node detection & HDA & Remove the highest-degree node, recompute degrees, repeat (library) & \citep{holme2002attack} \\
 & BPD+R & Belief-propagation-guided decycling with the published reinsertion pass (library) & \citep{mugisha2016identifying} \\
 & CI+R & Collective influence, optimal percolation, with the reinsertion pass (library) & \citep{morone2015influence} \\
 & EI & Explosive immunization, percolation-based targeted removal (library) & \citep{clusella2016immunization} \\
 & Frontier & Delete the susceptible ring of the observed outbreak, highest degree first; the outbreak-aware control & this work \\
\midrule
Epidemic control & DAVA & Data-aware vaccine allocation by a dominator tree on the observed infection state (library) & \citep{zhang2015data} \\
 & NetShield+ & Batched spectral-radius reduction by NetShield's shield value (library) & \citep{chen2015node} \\
 & GreedyWalk & Greedy spectral-radius reduction through closed-walk counts (library) & \citep{saha2015approximation} \\
 & EI & Explosive immunization (published code) & \citep{clusella2016immunization} \\
 & CI & Collective influence (published code) & \citep{morone2015influence} \\
\midrule
Influence blocking & RPS & Greedy maximum coverage over reverse prevention sets (library) & \citep{tong2017efficient} \\
 & Reverse & Reverse reachable set coverage restricted to the rumor's reachable region, TC-AIBM's baseline (library) & \citep{shi2025time} \\
 & Proximity & Seed the rumor's out-neighbors, highest degree first, CLDAG's baseline; the anchor (library) & \citep{he2012influence} \\
 & GreedyReplace & Node blocking by greedy replacement, node lever (published code) & \citep{xie2025influence} \\
 & SandIMIN & Sandwich-bounded node blocking, node lever (published code) & \citep{wang2024efficient} \\
\midrule
Source localization & LPSI & Label propagation; the sources are the local maxima of the converged field; the anchor (library) & \citep{wang2017multiple} \\
 & Rumor Centrality & The founding single-source estimator, applied per infected component (library) & \citep{shah2011rumors} \\
 & Dynamic Age & Spectral multi-source estimator from the dynamical age of nodes (library) & \citep{fioriti2012predicting} \\
 & Infected-Degree & Degree within the infected subgraph (library) & \citep{comin2011identifying} \\
 & SL-VAE & Variational autoencoder over source sets with a learned forward model, run through the GraphSL library (published code) & \citep{ling2022source, wang2024graphsl} \\
\midrule
Cascade reconstruction & Reports & The observed nodes alone with no inferred hidden ones; precision one by construction (library) & \citep{rozenshtein2016reconstructing} \\
 & Steiner Tree & Minimum Steiner tree over the reported nodes, the control of the ordered variants (library) & \citep{xiao2018reconstructing} \\
 & Jordan-Backward & The Jordan center of each observed component as its source, then a greedy backward decode (library) & \citep{zhu2014information} \\
 & DHREC & Diffusion archaeology, maximum-likelihood history reconstruction (library) & \citep{sefer2016diffusion} \\
 & CRI & Cluster the infected subgraph, then reverse-infect to the sources (library) & \citep{chen2016detecting} \\
\midrule
Cascade prediction & Persistence & Predicts no growth after the observation window; the floor & this work \\
 & Szabo-Huberman & Log-linear regression from early to final popularity; the anchor (library) & \citep{szabo2010predicting} \\
 & Hawkes & A marked Hawkes process fitted per cascade (library) & \citep{mishra2016feature} \\
 & RPP & Reinforced Poisson process with fitness, aging and rich-get-richer terms (library) & \citep{shen2014modeling} \\
 & Weng-Communities & Early diffusion across communities as the virality signal (library) & \citep{weng2013virality} \\
\end{longtable}
\end{small}

\section{Choice of Coding Model}
\label{app:coding-model}

We test whether the results depend on the LLM used as the coding agent by repeating the search on Network Science and Power Grid with GPT-5.6 Sol, Terra, and Luna. Table~\ref{tab:change_llm} reports the reward of the algorithm returned by the design loop under each model at every budget. The three models converge to similar performance, and each outperforms the strongest classical baseline, indicating that the gains are not specific to a particular coding model.

\begin{table}[H]
\caption{Algorithm design repeated on GPT-5.6 Sol, Terra, and Luna.}
\label{tab:change_llm}
\centering
\small
\setlength{\tabcolsep}{4pt}
\begin{tabular}{@{}lcccccccc@{}}
\toprule
 & \multicolumn{4}{c}{Network Science (1,589 nodes) ($\uparrow$)} & \multicolumn{4}{c}{Power Grid (4,941 nodes) ($\downarrow$)} \\
\cmidrule(lr){2-5}\cmidrule(lr){6-9}
Model & 1\% & 5\% & 10\% & 20\% & 1\% & 5\% & 10\% & 20\% \\
\midrule
GPT-5.6 Sol & 8.92 & 25.27 & 39.06 & 60.06 & 26.18 & 22.04 & 17.26 & 10.24 \\
GPT-5.6 Terra & 8.89 & 25.23 & 38.98 & 60.01 & 26.19 & 22.05 & 17.28 & 10.25 \\
GPT-5.6 Luna & 8.81 & 25.21 & 38.95 & 60.00 & 26.32 & 22.34 & 17.82 & 10.49 \\
\bottomrule
\end{tabular}
\end{table}

\section{Comparison with LLM Algorithm Discovery Systems}
\label{app:discovery}

The design loop is one of several systems that use an LLM to search over programs. We compare it with four published algorithm discovery systems: EoH~\citep{liu2024evolution}, OpenEvolve, the open-source implementation of AlphaEvolve~\citep{novikov2025alphaevolve}, LLaMEA~\citep{van2024llamea}, and ReEvo~\citep{ye2024reevo}. Each system runs its own search loop, prompts, and published defaults, and is given the same problem as the design loop: the task statement and a fitness function that scores a candidate on the exact simulator. No system sees the Network World Model. All four and the design loop use the same coding model, GPT-5.6 Sol.

\begin{table}[H]
\centering
\caption{Comparison with LLM algorithm discovery systems on influence maximization on Network Science (spread, \% of nodes activated, $\uparrow$) and critical node detection on Power Grid (remaining spread, \% of nodes infected, $\downarrow$), under IC. All rows use GPT-5.6 Sol as the coding model.}
\label{tab:discovery}
\small
\begin{tabular}{@{}l*{8}{c}@{}}
\toprule
 & \multicolumn{4}{c}{Influence maximization ($\uparrow$)} & \multicolumn{4}{c}{Critical node detection ($\downarrow$)} \\
 & \multicolumn{4}{c}{Network Science (1,589 nodes)} & \multicolumn{4}{c}{Power Grid (4,941 nodes)} \\
\cmidrule(lr){2-5}\cmidrule(lr){6-9}
Method & 1\% & 5\% & 10\% & 20\% & 1\% & 5\% & 10\% & 20\% \\
\midrule
EoH        & 5.5 & 19.5 & 32.5 & 54.6 & 26.6 & 23.4 & 19.4 & 14.3 \\
OpenEvolve & 5.1 & 21.9 & 36.8 & 50.8 & 26.6 & 23.8 & 21.4 & 17.7 \\
LLaMEA     & 6.1 & 21.1 & 34.9 & 56.2 & 26.6 & 23.4 & 19.5 & 14.3 \\
ReEvo      & 6.6 & 22.3 & 35.5 & 57.0 & 26.6 & 23.5 & 19.6 & 15.0 \\
\midrule
\textbf{Network World Model} & \textbf{8.9} & \textbf{25.3} & \textbf{39.0} & \textbf{60.0} & \textbf{26.2} & \textbf{22.0} & \textbf{17.3} & \textbf{10.2} \\
\bottomrule
\end{tabular}
\end{table}

Table~\ref{tab:discovery} reports the results. For both influence maximization and critical node detection, the design loop performs best at every budget. None of the four discovery systems reaches the strongest classical baseline on Network Science at any budget (IMM or OPIM, 8.8, 24.8, 38.0, and 59.2 in Table~\ref{tab:results_main}), although EoH and LLaMEA both converged to reverse-reachable-set sampling, the mechanism behind IMM. Therefore, the contributions of the algorithm design loop and the Network World Model are necessary to achieve considerable gains over the baselines.

\section{Full Results}
\label{app:full-results}

\begin{figure}[H]
\centering
\captionsetup[subfloat]{position=bottom,font=scriptsize,farskip=0pt,captionskip=1pt}
\subfloat[IM, Digg (IC)]{\includegraphics[width=0.24\textwidth]{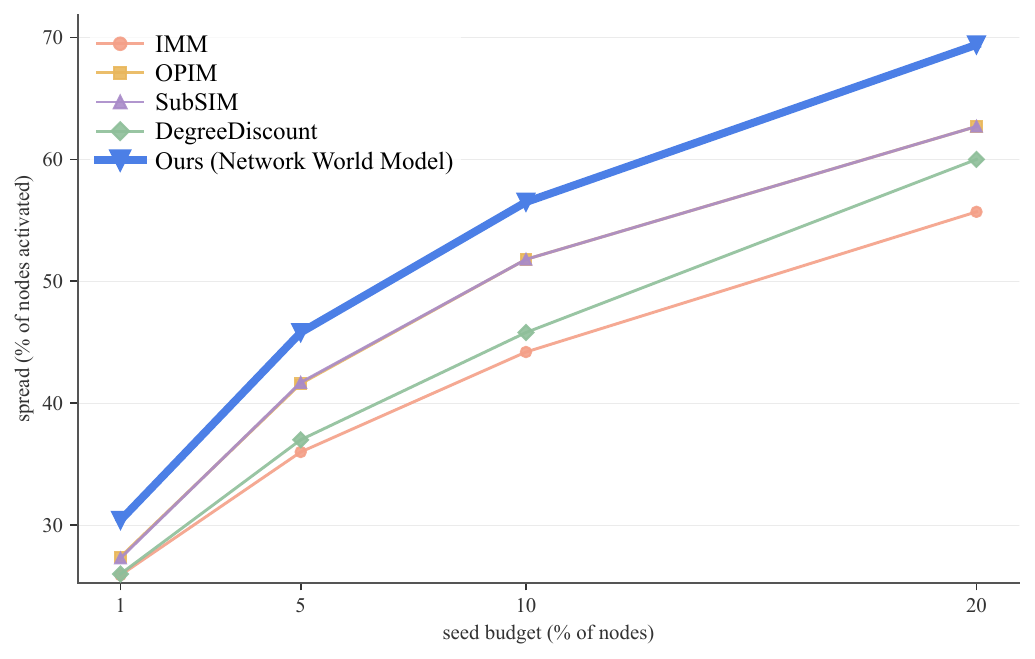}}\hfill
\subfloat[AIM, Digg (IC)]{\includegraphics[width=0.24\textwidth]{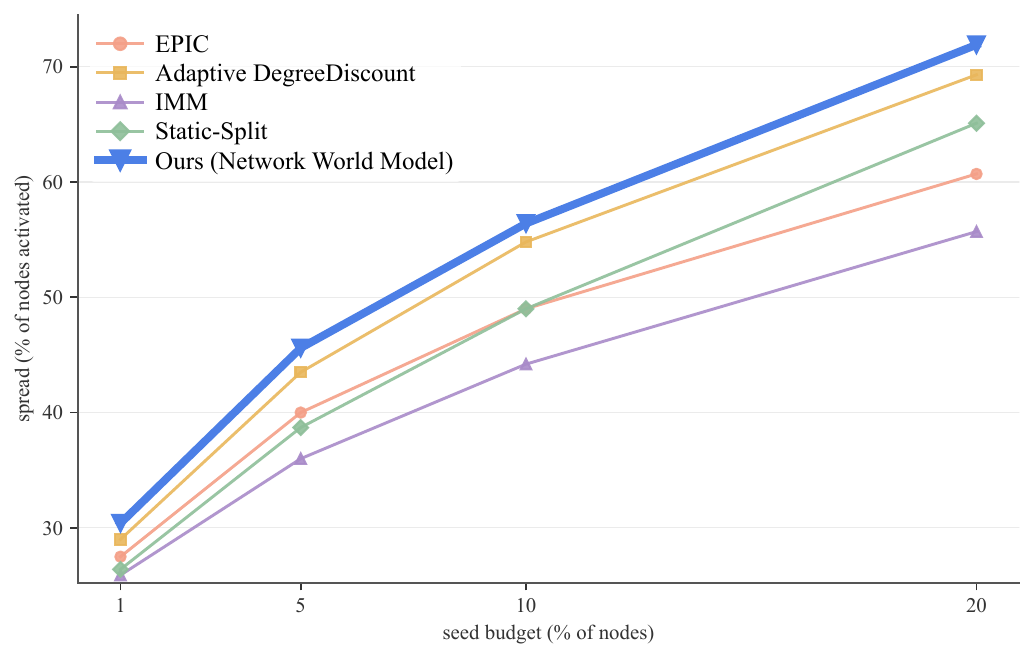}}\hfill
\subfloat[CND, PGP (IC)]{\includegraphics[width=0.24\textwidth]{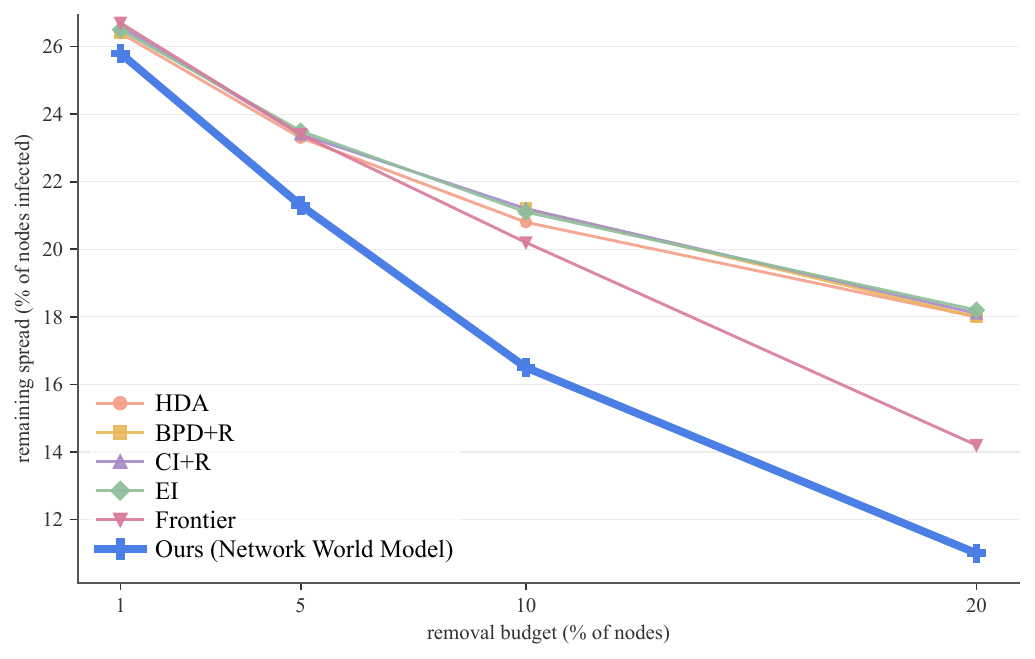}}\hfill
\subfloat[IB, Gnutella24 (IC)]{\includegraphics[width=0.24\textwidth]{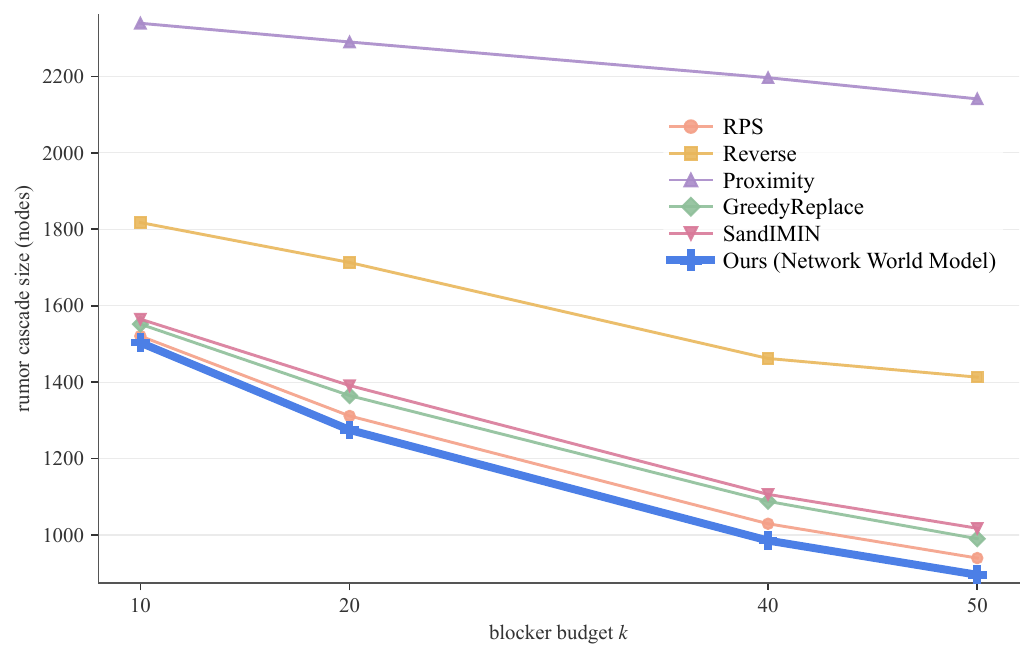}}
\par\vspace{4pt}
\caption{Performance across intervention budgets under IC for influence maximization (IM), adaptive influence maximization (AIM), critical node detection (CND), and influence blocking (IB).}
\label{fig:budget_curves}
\end{figure}

Table~\ref{tab:results_full} reports every end-to-end result behind Section~\ref{sec:exp-design}: all eight tasks and every dataset that was run. The two-dataset view of Table~\ref{tab:results_main} is a subset of it.

\begin{table}[H]
\centering
\caption{Designed algorithms results for all eight tasks and every dataset. Every cell is the IC\,/\,LT pair (IC\,/\,CLT for influence blocking, SIR\,/\,SIS for epidemic control) obtained by replaying the returned solution on the shared oracle referee, so all methods are scored on the same referee. A dash in a cell indicates no result because of either an Out of Memory (OOM) error or a timeout that exceeds the 1,800 s wall-clock limit.}
\label{tab:results_full}
\fontsize{5.6}{6.4}\selectfont
\setlength{\tabcolsep}{1.5pt}
\setlength{\aboverulesep}{0.6pt}
\setlength{\belowrulesep}{0.9pt}
\newlength{\resultsmethodw}
\settowidth{\resultsmethodw}{\textbf{Adaptive DegreeDiscount}}
\newcommand{\resultsgap}{\par\vspace{4pt}}

\begin{tabular*}{\textwidth}{@{}p{\resultsmethodw}@{\extracolsep{\fill}}*{12}{c}@{}}
\toprule
\multicolumn{13}{@{}l}{\textbf{Influence maximization} (IC\,/\,LT): spread, \% of nodes activated ($\uparrow$)} \\
\addlinespace[1pt]
 & \multicolumn{4}{c}{Network Science (1,589 nodes)} & \multicolumn{4}{c}{Digg (116,893 nodes)} & \multicolumn{4}{c}{NetHEPT (15,229 nodes)} \\
\cmidrule(lr){2-5}\cmidrule(lr){6-9}\cmidrule(lr){10-13}
Method & 1\% & 5\% & 10\% & 20\% & 1\% & 5\% & 10\% & 20\% & 1\% & 5\% & 10\% & 20\% \\
\midrule
IMM                               & 8.7\,/\,10.9 & 24.8\,/\,30.2 & 38.0\,/\,45.1 & 59.2\,/\,68.5 & 25.9\,/\,47.0 & 36.0\,/\,60.9 & 44.2\,/\,69.7 & 55.7\,/\,79.6 & 11.6\,/\,14.8 & 28.9\,/\,36.2 & 42.4\,/\,51.6 & 61.7\,/\,72.5 \\
OPIM                              & 8.8\,/\,10.8 & 24.2\,/\,29.5 & 37.7\,/\,45.3 & 58.3\,/\,67.4 & 27.4\,/\,50.5 & 41.6\,/\,68.4 & 51.8\,/\,77.6 & 62.7\,/\,85.5 & 12.1\,/\,15.5 & 31.4\,/\,39.2 & 44.7\,/\,54.3 & 64.5\,/\,75.7 \\
SubSIM                            & 8.7\,/\,10.7 & 24.3\,/\,29.1 & 37.3\,/\,45.0 & 57.8\,/\,68.5 & 27.3\,/\,50.1 & 41.7\,/\,68.4 & 51.8\,/\,77.7 & 62.7\,/\,85.4 & 12.6\,/\,16.2 & 31.3\,/\,39.0 & 44.8\,/\,54.5 & 64.3\,/\,75.4 \\
DeepIM                            & 4.8\,/\,5.3 & 15.3\,/\,20.0 & 27.7\,/\,32.9 & 45.8\,/\,53.0 & \na\,/\,\na & \na\,/\,\na & \na\,/\,\na & \na\,/\,\na & \na\,/\,\na & \na\,/\,\na & \na\,/\,\na & \na\,/\,\na \\
DegreeDiscount                    & 8.3\,/\,10.7 & 23.7\,/\,29.7 & 36.3\,/\,44.5 & 55.0\,/\,65.5 & 26.0\,/\,49.6 & 37.0\,/\,62.7 & 45.8\,/\,70.0 & 60.0\,/\,81.6 & 11.9\,/\,15.9 & 30.3\,/\,38.8 & 43.0\,/\,53.5 & 61.1\,/\,73.2 \\
\midrule
\textbf{Network World Model} & \textbf{8.9}\,/\,\textbf{11.3} & \textbf{25.3}\,/\,\textbf{30.9} & \textbf{39.1}\,/\,\textbf{46.8} & \textbf{60.1}\,/\,\textbf{69.3} & \textbf{30.4}\,/\,\textbf{56.8} & \textbf{45.8}\,/\,\textbf{73.2} & \textbf{56.5}\,/\,\textbf{83.1} & \textbf{69.4}\,/\,\textbf{93.5} & \textbf{12.9}\,/\,\textbf{16.9} & \textbf{33.0}\,/\,\textbf{41.4} & \textbf{47.8}\,/\,\textbf{58.1} & \textbf{68.5}\,/\,\textbf{80.0} \\
\bottomrule
\end{tabular*}
\resultsgap

\begin{tabular*}{\textwidth}{@{}p{\resultsmethodw}@{\extracolsep{\fill}}*{12}{c}@{}}
\toprule
\multicolumn{13}{@{}l}{\textbf{Adaptive influence maximization} (IC\,/\,LT): spread, \% of nodes activated ($\uparrow$)} \\
\addlinespace[1pt]
 & \multicolumn{4}{c}{Network Science (1,589 nodes)} & \multicolumn{4}{c}{Digg (116,893 nodes)} & \multicolumn{4}{c}{NetHEPT (15,229 nodes)} \\
\cmidrule(lr){2-5}\cmidrule(lr){6-9}\cmidrule(lr){10-13}
Method & 1\% & 5\% & 10\% & 20\% & 1\% & 5\% & 10\% & 20\% & 1\% & 5\% & 10\% & 20\% \\
\midrule
EPIC                  & 8.4\,/\,10.8 & 23.1\,/\,28.7 & 34.6\,/\,41.5 & 52.9\,/\,62.3 & 27.5\,/\,48.4 & 40.0\,/\,65.8 & 49.0\,/\,74.0 & 60.7\,/\,82.4 & 11.4\,/\,14.6 & 28.9\,/\,36.5 & 41.2\,/\,51.3 & 60.4\,/\,70.6 \\
Adaptive DegreeDiscount           & 8.2\,/\,10.8 & 23.6\,/\,30.2 & 36.5\,/\,45.8 & 56.9\,/\,69.3 & 29.0\,/\,54.2 & 43.5\,/\,68.9 & 54.8\,/\,79.4 & 69.3\,/\,93.6 & 11.9\,/\,16.4 & 31.6\,/\,41.3 & 46.7\,/\,58.2 & 67.9\,/\,80.1 \\
IMM                               & 8.7\,/\,10.9 & 24.8\,/\,30.2 & 38.0\,/\,45.1 & 59.2\,/\,68.5 & 25.9\,/\,47.0 & 36.0\,/\,60.9 & 44.2\,/\,69.7 & 55.7\,/\,79.6 & 11.6\,/\,14.8 & 28.9\,/\,36.2 & 42.4\,/\,51.6 & 61.7\,/\,72.5 \\
Static-Split                      & 8.3\,/\,10.7 & 24.1\,/\,30.3 & 37.5\,/\,45.9 & 58.1\,/\,69.3 & 26.4\,/\,49.6 & 38.7\,/\,64.5 & 49.0\,/\,73.8 & 65.1\,/\,87.5 & 12.1\,/\,16.0 & 31.0\,/\,39.8 & 44.7\,/\,55.8 & 63.8\,/\,76.8 \\
\midrule
\textbf{Network World Model} & \textbf{8.9}\,/\,\textbf{11.3} & \textbf{25.2}\,/\,\textbf{30.9} & \textbf{39.1}\,/\,\textbf{46.8} & \textbf{60.0}\,/\,\textbf{69.4} & \textbf{30.4}\,/\,\textbf{56.8} & \textbf{45.6}\,/\,\textbf{73.2} & \textbf{56.4}\,/\,\textbf{83.0} & \textbf{71.9}\,/\,\textbf{94.0} & \textbf{12.9}\,/\,\textbf{16.9} & \textbf{33.0}\,/\,\textbf{41.5} & \textbf{47.8}\,/\,\textbf{58.7} & \textbf{68.6}\,/\,\textbf{82.0} \\
\bottomrule
\end{tabular*}
\resultsgap

\begin{tabular*}{\textwidth}{@{}p{\resultsmethodw}@{\extracolsep{\fill}}*{12}{c}@{}}
\toprule
\multicolumn{13}{@{}l}{\textbf{Critical node detection} (IC\,/\,LT): remaining spread, \% of nodes infected ($\downarrow$)} \\
\addlinespace[1pt]
 & \multicolumn{4}{c}{Power Grid (4,941 nodes)} & \multicolumn{4}{c}{PGP (10,680 nodes)} & \multicolumn{4}{c}{Gnutella31 (62,561 nodes)} \\
\cmidrule(lr){2-5}\cmidrule(lr){6-9}\cmidrule(lr){10-13}
Method & 1\% & 5\% & 10\% & 20\% & 1\% & 5\% & 10\% & 20\% & 1\% & 5\% & 10\% & 20\% \\
\midrule
HDA                               & 26.5\,/\,28.9 & 24.0\,/\,27.1 & 21.8\,/\,25.4 & 18.6\,/\,22.3 & 26.4\,/\,31.9 & 23.3\,/\,28.5 & 20.8\,/\,25.6 & 18.0\,/\,21.6 & 31.3\,/\,39.0 & 26.0\,/\,35.0 & 21.9\,/\,30.8 & 17.4\,/\,22.9 \\
BPD+R                             & 26.7\,/\,29.1 & 24.1\,/\,27.5 & 21.9\,/\,25.8 & 18.5\,/\,22.4 & 26.4\,/\,32.0 & \na\,/\,\na & 21.2\,/\,25.8 & 18.0\,/\,21.6 & 31.4\,/\,39.1 & 26.4\,/\,35.5 & 22.5\,/\,31.4 & \na\,/\,\na \\
CI+R                              & 26.7\,/\,29.0 & 24.1\,/\,27.3 & 21.9\,/\,25.7 & 18.5\,/\,22.3 & 26.6\,/\,32.0 & 23.4\,/\,28.5 & 21.2\,/\,25.9 & 18.1\,/\,21.6 & 31.3\,/\,39.1 & \na\,/\,\na & \na\,/\,\na & \na\,/\,\na \\
EI                                & 26.8\,/\,29.2 & 24.5\,/\,27.9 & 22.1\,/\,26.0 & 19.1\,/\,23.2 & 26.5\,/\,32.0 & 23.5\,/\,28.8 & 21.1\,/\,25.9 & 18.2\,/\,22.2 & 32.0\,/\,39.4 & 26.4\,/\,35.3 & 22.1\,/\,31.1 & 17.9\,/\,23.7 \\
Frontier                          & 26.3\,/\,28.6 & 22.7\,/\,25.0 & 18.7\,/\,20.4 & 10.7\,/\,\textbf{10.7} & 26.7\,/\,32.0 & 23.4\,/\,27.9 & 20.2\,/\,23.6 & 14.2\,/\,14.5 & 31.3\,/\,39.0 & 26.0\,/\,34.2 & 21.6\,/\,28.6 & 16.0\,/\,17.6 \\
\midrule
\textbf{Network World Model} & \textbf{26.2}\,/\,\textbf{28.2} & \textbf{22.0}\,/\,\textbf{24.1} & \textbf{17.3}\,/\,\textbf{19.6} & \textbf{10.2}\,/\,\textbf{10.7} & \textbf{25.8}\,/\,\textbf{31.1} & \textbf{21.3}\,/\,\textbf{26.7} & \textbf{16.5}\,/\,\textbf{21.5} & \textbf{11.0}\,/\,\textbf{13.0} & \textbf{31.1}\,/\,\textbf{38.2} & \textbf{24.7}\,/\,\textbf{32.0} & \textbf{19.5}\,/\,\textbf{26.1} & \textbf{12.3}\,/\,\textbf{16.6} \\
\bottomrule
\end{tabular*}
\resultsgap

\begin{tabular*}{\textwidth}{@{}p{\resultsmethodw}@{\extracolsep{\fill}}*{10}{c}@{}}
\toprule
\multicolumn{11}{@{}l}{\textbf{Influence blocking} (IC\,/\,CLT): rumor cascade size (nodes) ($\downarrow$)} \\
\addlinespace[1pt]
 & \multicolumn{5}{c}{Email-EU (1,005 nodes)} & \multicolumn{5}{c}{Gnutella24 (26,518 nodes)} \\
\cmidrule(lr){2-6}\cmidrule(lr){7-11}
Method & 10 & 20 & 30 & 40 & 50 & 10 & 20 & 30 & 40 & 50 \\
\midrule
RPS                               & 60.2\,/\,\textbf{70.9} & 47.4\,/\,56.7 & 39.8\,/\,48.2 & 35.0\,/\,42.8 & 31.9\,/\,39.1 & 1520.6\,/\,1625.3 & 1311.6\,/\,1380.6 & 1151.5\,/\,1206.3 & 1029.6\,/\,1076.5 & 939.8\,/\,980.4 \\
Reverse                           & 61.0\,/\,73.1 & 49.2\,/\,57.9 & 42.1\,/\,49.9 & 37.2\,/\,44.8 & 34.5\,/\,42.1 & 1817.6\,/\,1961.8 & 1713.1\,/\,1856.4 & 1540.1\,/\,1676.8 & 1462.0\,/\,1587.7 & 1413.1\,/\,1535.7 \\
Proximity                         & 63.1\,/\,74.9 & 52.4\,/\,61.2 & 44.6\,/\,53.2 & 39.6\,/\,47.7 & 36.3\,/\,44.2 & 2338.9\,/\,2531.6 & 2289.9\,/\,2481.1 & 2239.7\,/\,2424.9 & 2196.5\,/\,2379.0 & 2141.1\,/\,2319.7 \\
GreedyReplace                     & 92.2\,/\,162.3 & 77.2\,/\,163.5 & 65.9\,/\,158.7 & 60.1\,/\,157.3 & 54.8\,/\,154.9 & 1551.8\,/\,1640.5 & 1364.9\,/\,1449.2 & 1216.8\,/\,1342.3 & 1088.2\,/\,1211.8 & 990.4\,/\,1072.4 \\
SandIMIN                          & 92.5\,/\,160.0 & 77.4\,/\,158.7 & 66.5\,/\,157.4 & 63.7\,/\,157.8 & 54.4\,/\,150.9 & 1564.9\,/\,1633.3 & 1391.2\,/\,1456.3 & 1228.9\,/\,1300.2 & 1106.4\,/\,1173.9 & 1017.8\,/\,1096.9 \\
\midrule
\textbf{Network World Model} & \textbf{59.4}\,/\,71.7 & \textbf{46.3}\,/\,\textbf{55.4} & \textbf{39.4}\,/\,\textbf{47.1} & \textbf{34.5}\,/\,\textbf{42.2} & \textbf{31.0}\,/\,\textbf{38.7} & \textbf{1503.3}\,/\,\textbf{1582.3} & \textbf{1275.0}\,/\,\textbf{1333.8} & \textbf{1113.6}\,/\,\textbf{1154.8} & \textbf{985.6}\,/\,\textbf{1023.0} & \textbf{896.0}\,/\,\textbf{924.7} \\
\bottomrule
\end{tabular*}
\resultsgap

\begin{tabular*}{\textwidth}{@{}p{\resultsmethodw}@{\extracolsep{\fill}}*{5}{c}@{}}
\toprule
\multicolumn{6}{@{}l}{\textbf{Influence blocking}, continued (IC\,/\,CLT): rumor cascade size (nodes) ($\downarrow$)} \\
\addlinespace[1pt]
 & \multicolumn{5}{c}{Cit-HepTh (27,769 nodes)} \\
\cmidrule(lr){2-6}
Method & 10 & 20 & 30 & 40 & 50 \\
\midrule
RPS                               & 882.0\,/\,947.0 & 815.3\,/\,868.8 & 769.5\,/\,819.7 & 737.5\,/\,784.7 & 714.2\,/\,759.9 \\
Reverse                           & 953.0\,/\,1048.3 & 936.6\,/\,1030.6 & 908.8\,/\,997.2 & 867.4\,/\,941.3 & 857.0\,/\,929.2 \\
Proximity                         & 1109.1\,/\,1292.2 & 1107.8\,/\,1290.8 & 1105.6\,/\,1288.1 & 1103.8\,/\,1286.0 & 1101.4\,/\,1282.9 \\
GreedyReplace                     & 897.6\,/\,971.7 & 835.2\,/\,905.0 & 795.8\,/\,857.1 & 766.9\,/\,822.4 & 742.6\,/\,797.5 \\
SandIMIN                          & 903.9\,/\,978.0 & 847.4\,/\,905.0 & 796.4\,/\,872.3 & 767.4\,/\,825.1 & 743.0\,/\,799.2 \\
\midrule
\textbf{Network World Model} & \textbf{879.1}\,/\,\textbf{942.2} & \textbf{810.7}\,/\,\textbf{862.7} & \textbf{767.5}\,/\,\textbf{815.9} & \textbf{736.3}\,/\,\textbf{783.5} & \textbf{712.2}\,/\,\textbf{755.7} \\
\bottomrule
\end{tabular*}
\resultsgap

\end{table}

\begin{table}[H]
\centering
\addtocounter{table}{-1}
\caption{Designed algorithms results for all eight tasks and every dataset (continued).}
\fontsize{5.6}{6.4}\selectfont
\setlength{\tabcolsep}{1.5pt}
\setlength{\aboverulesep}{0.6pt}
\setlength{\belowrulesep}{0.9pt}
\settowidth{\resultsmethodw}{\textbf{Adaptive DegreeDiscount}}
\newcommand{\resultsgap}{\par\vspace{4pt}}

\begin{tabular*}{\textwidth}{@{}p{\resultsmethodw}@{\extracolsep{\fill}}*{12}{c}@{}}
\toprule
\multicolumn{13}{@{}l}{\textbf{Epidemic control} (SIR\,/\,SIS): attack rate, \% of nodes ever infected ($\downarrow$)} \\
\addlinespace[1pt]
 & \multicolumn{4}{c}{Infectious SocioPatterns (10,972 nodes)} & \multicolumn{4}{c}{Oregon1 (10,670 nodes)} & \multicolumn{4}{c}{Brightkite (58,228 nodes)} \\
\cmidrule(lr){2-5}\cmidrule(lr){6-9}\cmidrule(lr){10-13}
Method & 1\% & 5\% & 10\% & 20\% & 1\% & 5\% & 10\% & 20\% & 1\% & 5\% & 10\% & 20\% \\
\midrule
DAVA                              & \textbf{19.9}\,/\,23.0 & \textbf{7.4}\,/\,8.3 & \textbf{1.0}\,/\,\textbf{1.0} & \textbf{1.0}\,/\,\textbf{1.0} & 7.9\,/\,9.6 & \textbf{1.0}\,/\,\textbf{1.0} & \textbf{1.0}\,/\,\textbf{1.0} & \textbf{1.0}\,/\,\textbf{1.0} & 40.9\,/\,47.1 & \na\,/\,\na & \na\,/\,\na & \na\,/\,\na \\
NetShield+                        & 21.2\,/\,24.9 & 17.3\,/\,20.6 & 12.8\,/\,15.6 & 7.4\,/\,8.9 & 4.3\,/\,5.4 & 2.1\,/\,2.3 & 1.8\,/\,1.9 & 1.6\,/\,1.7 & 24.1\,/\,29.4 & 8.5\,/\,10.5 & 4.5\,/\,5.4 & 2.7\,/\,3.0 \\
GreedyWalk                        & 20.9\,/\,24.6 & 16.5\,/\,19.8 & 12.3\,/\,14.9 & 7.0\,/\,8.5 & 4.8\,/\,5.9 & 2.5\,/\,2.8 & 1.9\,/\,2.1 & 1.6\,/\,1.7 & \textbf{23.2}\,/\,28.4 & 8.4\,/\,10.5 & 4.6\,/\,5.5 & 2.8\,/\,3.1 \\
EI                                & 20.7\,/\,24.4 & 16.3\,/\,19.3 & 11.5\,/\,13.6 & 8.1\,/\,9.5 & 4.2\,/\,5.2 & 2.0\,/\,2.1 & 1.7\,/\,1.8 & 1.6\,/\,1.6 & 23.5\,/\,28.7 & 8.7\,/\,10.9 & 4.5\,/\,5.3 & 2.8\,/\,3.1 \\
CI                                & 20.7\,/\,24.3 & 16.1\,/\,18.9 & 12.4\,/\,14.6 & 6.9\,/\,8.1 & 4.2\,/\,5.2 & 2.0\,/\,2.1 & 1.7\,/\,1.8 & 1.6\,/\,1.6 & 26.0\,/\,31.1 & 10.2\,/\,12.3 & 4.7\,/\,5.5 & 2.7\,/\,3.0 \\
\midrule
\textbf{Network World Model} & 20.4\,/\,\textbf{19.5} & 12.2\,/\,\textbf{4.1} & \textbf{1.0}\,/\,\textbf{1.0} & \textbf{1.0}\,/\,\textbf{1.0} & \textbf{2.4}\,/\,\textbf{4.3} & \textbf{1.0}\,/\,\textbf{1.0} & \textbf{1.0}\,/\,\textbf{1.0} & \textbf{1.0}\,/\,\textbf{1.0} & \textbf{23.2}\,/\,\textbf{24.5} & \textbf{1.4}\,/\,\textbf{1.5} & \textbf{1.0}\,/\,\textbf{1.0} & \textbf{1.0}\,/\,\textbf{1.0} \\
\bottomrule
\end{tabular*}
\resultsgap

\begin{tabular*}{\textwidth}{@{}p{\resultsmethodw}@{\extracolsep{\fill}}*{3}{c}@{}}
\toprule
\multicolumn{4}{@{}l}{\textbf{Source localization} (IC\,/\,LT): consistency of the recovered sources with the observation, 0 is perfect ($\uparrow$)} \\
\addlinespace[1pt]
Method & Cora-ML (2,810 nodes) & Power Grid (4,941 nodes) & Deezer (47,538 nodes) \\
\midrule
LPSI                              & \textminus{}0.173\,/\,\textminus{}0.130 & \textminus{}0.106\,/\,\textminus{}0.100 & \textminus{}0.176\,/\,\textminus{}0.174 \\
Rumor Centrality                  & \textminus{}0.237\,/\,\textminus{}0.257 & \textminus{}0.242\,/\,\textminus{}0.260 & \textminus{}0.190\,/\,\textminus{}0.200 \\
Dynamic Age                       & \textminus{}0.200\,/\,\textminus{}0.183 & \textminus{}0.234\,/\,\textminus{}0.274 & \textminus{}0.178\,/\,\textminus{}0.183 \\
Infected-Degree                   & \textminus{}0.173\,/\,\textminus{}0.154 & \textminus{}0.116\,/\,\textminus{}0.119 & \textminus{}0.178\,/\,\textminus{}0.186 \\
SL-VAE                            & \textminus{}0.370\,/\,\textminus{}0.437 & \textminus{}0.306\,/\,\textminus{}0.323 & \textminus{}0.314\,/\,\textminus{}0.240 \\
\midrule
\textbf{Network World Model} & \textbf{\textminus{}0.139}\,/\,\textbf{\textminus{}0.106} & \textbf{\textminus{}0.080}\,/\,\textbf{\textminus{}0.073} & \textbf{\textminus{}0.147}\,/\,\textbf{\textminus{}0.145} \\
\bottomrule
\end{tabular*}
\resultsgap

\begin{tabular*}{\textwidth}{@{}p{\resultsmethodw}@{\extracolsep{\fill}}*{3}{c}@{}}
\toprule
\multicolumn{4}{@{}l}{\textbf{Cascade reconstruction} (IC\,/\,LT): referee reward of the decoded history, 0 is perfect ($\uparrow$)} \\
\addlinespace[1pt]
Method & UCI Students (1,266 nodes) & CA-GrQc (5,242 nodes) & RT-Pol (18,470 nodes) \\
\midrule
Reports                           & \textminus{}0.852\,/\,\textminus{}0.946 & \textminus{}0.791\,/\,\textminus{}0.692 & \textminus{}1.252\,/\,\textminus{}1.068 \\
Steiner Tree                      & \textminus{}0.901\,/\,\textminus{}0.978 & \textminus{}0.798\,/\,\textminus{}0.753 & \textminus{}2.618\,/\,\textminus{}2.076 \\
Jordan-Backward                   & \textminus{}0.852\,/\,\textminus{}0.946 & \textminus{}0.791\,/\,\textminus{}0.692 & \textminus{}1.252\,/\,\textminus{}1.068 \\
DHREC                             & \textminus{}1.464\,/\,\textminus{}1.525 & \textminus{}1.175\,/\,\textminus{}1.176 & \textminus{}2.471\,/\,\textminus{}2.594 \\
CRI                               & \textminus{}0.885\,/\,\textminus{}0.995 & \textminus{}1.106\,/\,\textminus{}1.151 & \textminus{}0.785\,/\,\textminus{}0.782 \\
\midrule
\textbf{Network World Model} & \textbf{\textminus{}0.675}\,/\,\textbf{\textminus{}0.717} & \textbf{\textminus{}0.566}\,/\,\textbf{\textminus{}0.566} & \textbf{\textminus{}0.391}\,/\,\textbf{\textminus{}0.373} \\
\bottomrule
\end{tabular*}
\resultsgap

\begin{tabular*}{\textwidth}{@{}p{\resultsmethodw}@{\extracolsep{\fill}}*{3}{c}@{}}
\toprule
\multicolumn{4}{@{}l}{\textbf{Cascade prediction}: MSLE at the prediction horizon ($\downarrow$)} \\
\addlinespace[1pt]
Method & APS (30,000 nodes) & Taoke (29,711 nodes) & Digg (5,000 nodes) \\
\midrule
Persistence                       & 0.288 & 1.101 & 4.222 \\
Szabo-Huberman                    & 0.098 & 0.716 & 0.834 \\
Hawkes                            & 0.141 & 0.957 & 2.744 \\
RPP                               & 0.116 & 1.009 & 3.283 \\
Weng-Communities                  & 0.196 & 0.850 & 1.142 \\
\midrule
\textbf{Network World Model} & \textbf{0.078} & \textbf{0.544} & \textbf{0.774} \\
\bottomrule
\end{tabular*}
\end{table}

\subsection{Cross-Network Transfer}
\label{app:transfer}

Table~\ref{tab:transfer} gives the per-budget results behind Figure~\ref{fig:transfer}, together with the influence maximization and adaptive influence maximization transfers. For each task, we replay the algorithm designed on one network unmodified on the task's other networks and compare it with the algorithm designed directly on each target; the two agree within $0.1$ percentage points in every completed cell.

\begin{table}[H]
\centering
\caption{Generalizability of the designed algorithms. The algorithm designed on one graph is replayed unmodified on the task's other graphs. Designed on the target is that graph's own returned program, as in Tables~\ref{tab:results_main} and~\ref{tab:results_full}. The NetHEPT adaptive program exceeded the 1,800 s wall-clock limit for one call on Digg at every budget.}
\label{tab:transfer}
\small
\begin{tabular}{@{}l*{8}{c}@{}}
\toprule
\multicolumn{9}{@{}l}{\textbf{Influence maximization} (IC), program designed on Digg ($\uparrow$)} \\
\addlinespace[2pt]
 & \multicolumn{4}{c}{Network Science (1,589)} & \multicolumn{4}{c}{NetHEPT (15,229)} \\
\cmidrule(lr){2-5}\cmidrule(lr){6-9}
Algorithm & 1\% & 5\% & 10\% & 20\% & 1\% & 5\% & 10\% & 20\% \\
\midrule
Designed on the target       & 8.9 & 25.3 & 39.1 & 60.1 & 12.9 & 33.0 & 47.8 & 68.5 \\
Transferred from Digg        & 8.9 & 25.3 & 39.1 & 60.0 & 12.9 & 33.0 & 47.8 & 68.6 \\
\bottomrule
\addlinespace[0.8em]
\toprule
\multicolumn{9}{@{}l}{\textbf{Adaptive influence maximization} (IC), program designed on NetHEPT ($\uparrow$)} \\
\addlinespace[2pt]
 & \multicolumn{4}{c}{Network Science (1,589)} & \multicolumn{4}{c}{Digg (116,893)} \\
\cmidrule(lr){2-5}\cmidrule(lr){6-9}
Algorithm & 1\% & 5\% & 10\% & 20\% & 1\% & 5\% & 10\% & 20\% \\
\midrule
Designed on the target       & 8.9 & 25.2 & 39.1 & 60.0 & 30.4 & 45.6 & 56.4 & 71.9 \\
Transferred from NetHEPT     & 8.9 & 25.3 & 39.1 & 60.1 & \na & \na & \na & \na \\
\bottomrule
\addlinespace[0.8em]
\toprule
\multicolumn{9}{@{}l}{\textbf{Critical node detection} (IC), program designed on PGP ($\downarrow$)} \\
\addlinespace[2pt]
 & \multicolumn{4}{c}{Power Grid (4,941)} & \multicolumn{4}{c}{Gnutella31 (62,561)} \\
\cmidrule(lr){2-5}\cmidrule(lr){6-9}
Algorithm & 1\% & 5\% & 10\% & 20\% & 1\% & 5\% & 10\% & 20\% \\
\midrule
Designed on the target       & 26.2 & 22.1 & 17.4 & 10.2 & 31.1 & 24.7 & 19.5 & 12.3 \\
Transferred from PGP         & 26.2 & 22.1 & 17.3 & 10.2 & 31.1 & 24.7 & 19.5 & 12.3 \\
\bottomrule
\end{tabular}
\end{table}

The one incomplete transfer is the adaptive influence maximization algorithm designed on NetHEPT, which exceeds the $1{,}800$ s per-call limit on the $116{,}893$-node Digg network; its cost, not its decision rule, fails to transfer.

\section{Additional Network World Model Experiments}
\label{app:wm-study}

\subsection{World Model Setup}
\label{app:wm-setup}

\paragraph{Data and model.} The IC training suites are network-disjoint: BA-24 holds 24 Barab\'asi-Albert networks with $N=100$, and SBM-24 holds 24 stochastic block model networks with $N=120$; LT is trained on BA-24. The default setting uses $w_{uv}=1/\mathrm{indeg}(v)$. A random-weight setting draws $\textnormal{w}_{uv}\sim U(0.02,0.4)$, so that the transmission probability cannot be reconstructed from the network structure, and is evaluated both with the weights visible to the model and with them hidden. Held-out BA and SBM networks are used for the decision-value tests, WS, ER, and SBM networks for topology shift, and BA networks up to $N=1000$ for size transfer. Training and evaluation networks are disjoint whenever we make a cross-network claim. The model is the GraphSAGE encoder with the structured transition head of Section~\ref{sec:world-model}, at $34{,}113$ parameters.

\paragraph{Targets and metrics.} For IC, each recorded transition stores a soft next-state target estimated from 30 simulator draws, together with one realized next state; for LT, the soft target is the analytic one-step conditional hazard. Probabilistic metrics are computed on at-risk nodes only: the KL divergence to the soft target and the Brier score $(y-\hat y)^2$, and, on realized states, NLL and UROC. Multi-step quality uses the sampled rollout of Eq.~(\ref{eq:wm-rollout}) and reports the rollout Brier score and the relative bias of the final spread, $(\hat{R}-R)/R$, at $H=10$. To measure whether the model preserves the effect of an intervention, and not only the marginal state distribution, we compare counterfactual effects with the exact transition mechanism,
\begin{equation}
d^{\mathrm{pred}}_v=
\hat p_v(\bm{s},a)-\hat p_v(\bm{s},a'),
\qquad
d^{\mathrm{exact}}_v=
p_v^{\mathrm{exact}}(\bm{s},a)-p_v^{\mathrm{exact}}(\bm{s},a'),
\label{eq:exact-effect}
\end{equation}
and report the normalized MAE and the Pearson correlation on nodes with a non-zero exact effect. Using the exact mechanism avoids treating Monte Carlo noise as ground truth.

\subsubsection{Scale and Topology Transfer}
\label{app:wm-transfer}

Table~\ref{tab:wm_ood} freezes the BA-100 IC checkpoint and evaluates it without retraining, changing only the evaluation network. The model transfers to networks up to $10\times$ larger and to WS, ER, and SBM topologies: one-step KL stays between $0.0036$ and $0.0063$, and degradation appears in long-horizon spread calibration, not in one-step discrimination. These runs keep the edge weights visible to the model, so the structured head can read the transmission probability from its input.

\begin{table}[H]
\caption{Frozen-checkpoint generalization across network size and topology. ``Exact Brier'' compares the predicted one-step probability with the analytic IC probability. The rollout column reports the model's Brier score and the noise floor of the exact process at $H=10$, and spread bias reports the same pair. The counterfactual effect is measured against the analytic intervention effect. Arrows give the better direction; spread bias is signed, so closer to zero is better.}
\label{tab:wm_ood}
\centering
\scriptsize
\resizebox{\textwidth}{!}{%
\begin{tabular}{@{}llcccccc@{}}
\toprule
Shift & Target & KL ($\downarrow$) & Exact Brier ($\downarrow$) & AUROC ($\uparrow$) & $H{=}10$ Brier ($\downarrow$) & Spread bias ($\rightarrow 0$) & Effect MAE$_{\mathrm{norm}}$ ($\downarrow$) / Pearson ($\uparrow$) \\
\midrule
None & BA-100 & 0.0059 & 0.00011 & 0.923 & 0.0083 / 0.0078 & $+0.2\%$ / $-0.4\%$ & 0.0095 / 1.0000 \\
Size $2\times$ & BA-200 & 0.0048 & 0.00002 & 0.934 & 0.0076 / 0.0072 & $+1.6\%$ / $-0.5\%$ & 0.0103 / 0.9999 \\
Size $5\times$ & BA-500 & 0.0043 & 0.00002 & 0.943 & 0.0079 / 0.0064 & $+3.0\%$ / $0.0\%$ & 0.0098 / 0.9999 \\
Size $10\times$ & BA-1000 & 0.0041 & 0.00003 & 0.948 & 0.0075 / 0.0066 & $+2.7\%$ / $-0.3\%$ & 0.0129 / 0.9999 \\
Topology & WS-100 & 0.0036 & 0.00000 & 0.932 & 0.0074 / 0.0073 & $-0.9\%$ / $-1.2\%$ & 0.0088 / 1.0000 \\
Topology & ER-100 & 0.0051 & 0.00023 & 0.930 & 0.0093 / 0.0081 & $-2.5\%$ / $-0.8\%$ & 0.0131 / 0.9995 \\
Topology & SBM-100 & 0.0063 & 0.00084 & 0.946 & 0.0113 / 0.0087 & $-5.0\%$ / $-0.9\%$ & 0.0299 / 0.9951 \\
\bottomrule
\end{tabular}%
}
\end{table}

\subsubsection{Transition Head and Action Conditioning}
\label{app:wm-ablation}

Comparing the structured transition head with a linear one in the world model reveals that encoding the known form of the process in the readout matters most over long rollouts: the linear head leaves one-step KL almost unchanged ($0.0061$ against $0.0059$), but raises the spread bias at $H=10$ from $+1.7\%$ to $+8.7\%$ and the normalized error of the predicted action effect from $0.0105$ to $0.0398$, so encoding the known form of the process matters over a horizon and not at one step. By freezing a checkpoint and changing only the evaluation network, the model transfers to networks ten times larger and to WS, ER, and SBM topologies with one-step KL between $0.0036$ and $0.0063$, and what degrades is long-horizon spread calibration.

We also evaluate the world model on two benchmark networks: Network Science under IC and Power Grid under IC. The metrics are those of Table~\ref{tab:wm_main}, with two additions from the action-conditioning evaluation: the drop in one-step $\Delta$F1 when the action channels are zeroed at evaluation, which is zero for a model that ignores its actions, and the Pearson correlation between the predicted effect of an action and the effect measured on the Monte Carlo simulator's counterfactual forks.

\subsection{World Model Encoder Ablation}
\label{app:wm-encoder}

Table~\ref{tab:wm_variants} compares three encoders against a GraphSAGE checkpoint trained on the data and with the settings of the end-to-end runs. Every row keeps the data, the structured head, and the training budget, and changes the backbone; because only GraphSAGE supports per-edge message conditioning, every row in this table is trained without it, so the GraphSAGE row is a retrained reference, not the checkpoint from the end-to-end runs. On both networks, the four encoders are within $0.002$ $\Delta$F1 of one another, and the final count bias of every rollout stays within $3\%$, so the structured head, not the encoder, carries the result.

\begin{table}[H]
\caption{World-model encoders on Network Science. Each row changes only the encoder against the first row. Count bias is the relative bias of the final count of a free-running sampled rollout at the horizon; ``$\Delta$F1 drop'' is the loss in one-step $\Delta$F1 when the action channels are zeroed at evaluation; ``Effect corr.'' is the Pearson correlation between the predicted and the measured effect of an action.}
\label{tab:wm_variants}
\centering
\small
\begin{tabular}{@{}lcccccc@{}}
\toprule
Variant & $\Delta$F1 ($\uparrow$) & Brier ($\downarrow$) & Count bias & $\Delta$F1 drop & Effect corr. ($\uparrow$) & Train (min) \\
\midrule
GraphSAGE & 0.8765 & 0.0003 & $-0.5\%$ & 0.792 & 0.849 & 6.5 \\
GCN & 0.8750 & 0.0003 & $-1.4\%$ & 0.810 & 0.849 & 8.8 \\
GATv2 & 0.8756 & 0.0004 & $-2.3\%$ & 0.812 & 0.849 & 11.4 \\
GCNII & 0.8745 & 0.0004 & $-1.6\%$ & 0.796 & 0.848 & 20.8 \\
\bottomrule
\end{tabular}
\end{table}

\begin{figure}[H]
\centering
\includegraphics[width=0.49\textwidth]{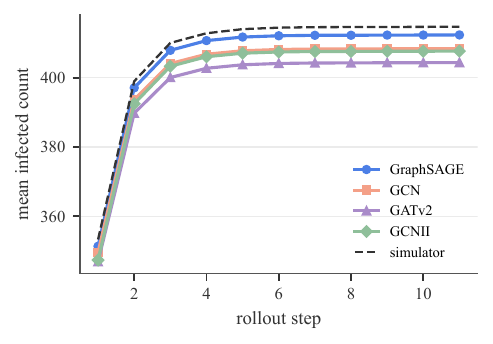}
\includegraphics[width=0.49\textwidth]{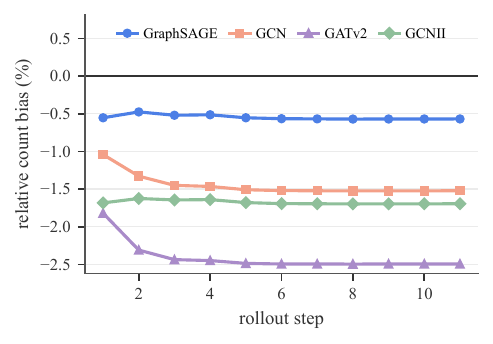}
\caption{The four encoders on Network Science. Left: mean infected count of the free-running sampled rollout at each step, against the simulator (dashed). Right: the same rollouts as relative count bias by step; the bias settles within the first five steps on every backbone and does not grow with the horizon.}
\label{fig:wm_backbone_bias}
\end{figure}

\subsection{Error Against Horizon}
\label{app:error-horizon}

Table~\ref{tab:wm_horizon} and Figure~\ref{fig:wm_horizon} report the relative count bias of the free-running rollout of the end-to-end checkpoints by step, on the 50 held-out episodes of the main runs. On both networks, the bias settles within the first two to five steps and then stays put, so one-step error does not compound over a rollout. Power Grid's rollout overcounts by about 3\%, a constant offset the search tolerates because it shifts every candidate's score equally.

\begin{table}[H]
\caption{Relative count bias of the free-running sampled rollout against the Monte Carlo simulator, by rollout step, for the checkpoints used in the end-to-end runs, on the 50 held-out episodes of each run. Zero means no bias; a magnitude that grows with the step would mean compounding error.}
\label{tab:wm_horizon}
\centering
\small
\begin{tabular}{@{}lcccc@{}}
\toprule
Network & Step 1 & Step 2 & Step 5 & Step 10 \\
\midrule
Network Science & $+0.0\%$ & $+0.1\%$ & $+0.1\%$ & $+0.1\%$ \\
Power Grid      & $+1.6\%$ & $+2.4\%$ & $+3.2\%$ & $+3.3\%$ \\
\bottomrule
\end{tabular}
\end{table}

\begin{figure}[H]
\centering
\includegraphics[width=\textwidth]{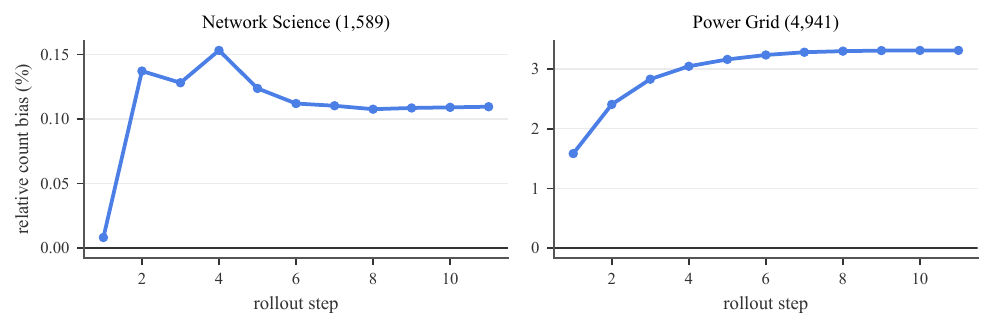}
\caption{Relative count bias of the free-running sampled rollout by step on the two benchmark networks, on the 50 held-out episodes of the main runs. The bias settles within the first few steps on both networks and does not grow.}
\label{fig:wm_horizon}
\end{figure}

\section{Examples of Designed Algorithms}
\label{app:examples}

We highlight four algorithms returned by the algorithm design loop for four different tasks: critical node detection, influence maximization, source localization, and influence blocking. Each is the program returned for its task and network under IC, and it is reported in Tables~\ref{tab:results_main} and~\ref{tab:results_full}. For each program, we provide a short description, pseudocode, results, and a Python source-code excerpt. Shaded lines in the pseudocode mark the steps that set the program apart from the classical methods in the library (Appendix~\ref{app:library}), the description sets the same steps in bold, and the excerpt shades the code lines that implement them.

\subsection{Critical Node Detection: Ensemble Cavity-Gradient Removal}
\label{app:example-cnd}

\paragraph{Description.} The program removes nodes greedily, but it ranks them by a quantity the dismantling methods in the library do not compute: a gradient conditioned on the outbreak, which estimates to first order how many infections at the horizon the removal of each node prevents. The forward part \textbf{runs non-backtracking cavity messages from the outbreak sources for $H$ steps}, where the message on $u\to v$ is the probability that $u$ is infected without help from $v$. The backward part \textbf{propagates adjoints back through the same $H$ steps}, so a node's score counts both its own infection and the infections its outgoing messages cause downstream. HDA, CI and EI rank nodes without regard to where the outbreak starts, and Frontier looks only at the ring around the sources; the gradient accounts for both the source location and the horizon. The program \textbf{computes the gradient under two models}, the given transmission probabilities and a copy capped at $0.9$ that tempers edges transmitting almost surely, and ranks nodes by an \textbf{ensemble score that weights the capped model more and adds a bonus for disagreement between the two}, so a node one model undervalues is not averaged away. Recomputing the gradient after every removal would take $k$ passes, so the program \textbf{removes six nodes per gradient pass and shrinks the scores of each removed node's neighbors} in between, which keeps one batch from spending two removals on the same route.

\begin{algorithm}[H]
\caption{Ensemble cavity-gradient removal, designed for critical node detection on PGP under IC.}
\label{alg:example-cnd}
\begin{algorithmic}[1]
\REQUIRE network $\mathcal{G}$ with transmission probabilities $p_{uv}$, outbreak sources $\mathbb{S}$, budget $k$, horizon $H$
\STATE $\bar p_{uv}\leftarrow\min(p_{uv},0.9)$; \ $\mathbb{R}\leftarrow\emptyset$ \COMMENT{capped probabilities; removed nodes}
\WHILE{$|\mathbb{R}|<k$}
\STATE \keystep{$\bm{g}^{\mathrm{cap}}\leftarrow\textsc{CavityGradient}(\bar p,\mathbb{R})$, \ $\bm{g}^{\mathrm{orig}}\leftarrow\textsc{CavityGradient}(p,\mathbb{R})$}
\FOR{up to six removals}
\STATE \keystep{$\mathrm{score}_v\leftarrow0.65\,g^{\mathrm{cap}}_v+0.35\,g^{\mathrm{orig}}_v+0.10\,\lvert g^{\mathrm{cap}}_v-g^{\mathrm{orig}}_v\rvert$ for every $v\notin\mathbb{S}\cup\mathbb{R}$}
\STATE $v^\star\leftarrow\operatorname*{arg\,max}_v\mathrm{score}_v$; \ $\mathbb{R}\leftarrow\mathbb{R}\cup\{v^\star\}$
\STATE \keystep{in both models, shrink the neighbors' scores: $g_w\leftarrow g_w(1-0.55\,p_{v^\star w})$ for out-neighbors and $g_w\leftarrow g_w(1-0.35\,p_{wv^\star})$ for in-neighbors}
\ENDFOR
\ENDWHILE
\RETURN remove every $v\in\mathbb{R}$ at $t=0$
\STATE \textbf{function} \textsc{CavityGradient}$(p,\mathbb{R})$, on $\mathcal{G}$ with the nodes in $\mathbb{R}$ deleted:
\STATE \keystep{forward: run non-backtracking cavity messages from $\mathbb{S}$ for $H$ steps and read each node's probability of being infected by step $H$}
\STATE \keystep{backward: propagate adjoints through the $H$ steps; $g_v$ is $v$'s infection probability plus the downstream infections its outgoing messages cause}
\STATE \textbf{return} $\bm{g}$
\end{algorithmic}
\end{algorithm}

\paragraph{Results.} At the $10\%$ budget on PGP, the program leaves $16.5\%$ of the nodes infected, against $20.2\%$ for Frontier, the strongest baseline, and the same program gives the $5\%$ and $20\%$ cells ($21.3\%$ and $11.0\%$). Replayed unchanged, it leaves $17.3\%$ infected on Power Grid and $19.5\%$ on Gnutella31 at the $10\%$ budget, matching the programs designed on those networks (Table~\ref{tab:transfer}).

\paragraph{Excerpt.} The selection loop, with the ensemble score and the neighbor discount:
\codeblockstart
\codeplain{while~len(chosen)~<~count:}
\codekey{~~~~capped\codeus{}scores~=~cavity\codeus{}gradient(capped\codeus{}p)}
\codekey{~~~~actual\codeus{}scores~=~cavity\codeus{}gradient(actual\codeus{}p)}
\codekey{~~~~batch~=~min(6,~count~-~len(chosen))}
\codeplain{~~~~for~\codeus{}~in~range(batch):}
\codeplain{~~~~~~~~\char35{}~Modestly~favor~cuts~whose~value~differs~between~models,}
\codeplain{~~~~~~~~\char35{}~rather~than~allowing~averaging~to~conceal~a~useful~gateway.}
\codekey{~~~~~~~~scores~=~(}
\codekey{~~~~~~~~~~~~0.65~*~capped\codeus{}scores}
\codekey{~~~~~~~~~~~~+~0.35~*~actual\codeus{}scores}
\codekey{~~~~~~~~~~~~+~0.10~*~np.abs(capped\codeus{}scores~-~actual\codeus{}scores)}
\codekey{~~~~~~~~)}
\codeplain{~~~~~~~~scores[\char126{}eligible]~=~-np.inf}
\codeplain{~~~~~~~~node~=~int(np.argmax(scores))}
\codeplain{~~~~~~~~if~not~np.isfinite(scores[node]):}
\codeplain{~~~~~~~~~~~~break}
\codeplain{~~~~~~~~chosen.append(node)}
\codeplain{~~~~~~~~alive[node]~=~0.0}
\codeplain{~~~~~~~~eligible[node]~=~False}
\codeplain{~~~~~~~~out\codeus{}ids~=~outgoing\codeus{}order[}
\codeplain{~~~~~~~~~~~~outgoing\codeus{}ptr[node]:outgoing\codeus{}ptr[node~+~1]}
\codeplain{~~~~~~~~]}
\codeplain{~~~~~~~~in\codeus{}ids~=~incoming\codeus{}order[}
\codeplain{~~~~~~~~~~~~incoming\codeus{}ptr[node]:incoming\codeus{}ptr[node~+~1]}
\codeplain{~~~~~~~~]}
\codeplain{~~~~~~~~for~model\codeus{}scores,~model\codeus{}p~in~(}
\codeplain{~~~~~~~~~~~~(capped\codeus{}scores,~capped\codeus{}p),}
\codeplain{~~~~~~~~~~~~(actual\codeus{}scores,~actual\codeus{}p),}
\codeplain{~~~~~~~~):}
\codekey{~~~~~~~~~~~~np.multiply.at(}
\codekey{~~~~~~~~~~~~~~~~model\codeus{}scores,}
\codekey{~~~~~~~~~~~~~~~~v[out\codeus{}ids],}
\codekey{~~~~~~~~~~~~~~~~1.0~-~0.55~*~model\codeus{}p[out\codeus{}ids],}
\codekey{~~~~~~~~~~~~)}
\codekey{~~~~~~~~~~~~np.multiply.at(}
\codekey{~~~~~~~~~~~~~~~~model\codeus{}scores,}
\codekey{~~~~~~~~~~~~~~~~u[in\codeus{}ids],}
\codekey{~~~~~~~~~~~~~~~~1.0~-~0.35~*~model\codeus{}p[in\codeus{}ids],}
\codekey{~~~~~~~~~~~~)}
\codeblockend

\subsection{Influence Maximization: Stratified Horizon-Limited Coverage with Exchange Sweeps}
\label{app:example-im}

\paragraph{Description.} The program builds on reverse-reachable sampling, the mechanism behind IMM, OPIM and SubSIM, and changes it in three ways. Its reverse-reachable sets are \textbf{horizon-limited}: each keeps only the nodes that reach its root within $H$ reverse steps of sampled live edges, matching the reward, which counts activations within $H$ steps, whereas the classical estimators count activations at any time. Its roots are \textbf{stratified}: every node is the root of exactly $r$ sets instead of a random number of them, which removes the variance that uniformly drawn roots add to the coverage estimate. After greedy maximum coverage, an \textbf{exchange sweep} revisits the seeds in increasing order of \textbf{unique coverage}, the number of sets only that seed covers, and \textbf{replaces a seed only when a non-seed covers more by a fixed margin}. A second sweep, bounded to $4{,}096$ seeds, recomputes unique coverage after the first sweep's swaps and revisits the seeds that are now weakest.

\begin{algorithm}[H]
\caption{Stratified horizon-limited coverage with exchange sweeps, designed for influence maximization on Digg under IC.}
\label{alg:example-im}
\begin{algorithmic}[1]
\REQUIRE network $\mathcal{G}$ with $N$ nodes and transmission probabilities $p_{uv}$, budget $k$, horizon $H$
\STATE sample $2{,}048$ pilot sets for the mean set size $\bar m$; \ $r\leftarrow20\cdot\min(40,\max(1,\lfloor3\times10^{7}/(N\bar m)\rfloor))$
\STATE \keystep{stratified roots: make every node the root of exactly $r$ reverse-reachable sets, $rN$ sets in total}
\STATE \keystep{horizon-limited sets: for each root, keep the nodes that reach it within $H$ reverse steps of sampled live edges}
\STATE $\mathbb{A}\leftarrow\emptyset$; repeat $k$ times: add the node covering the most uncovered sets (ties by weighted out-degree)
\FOR{sweep limit $L\in(k,\ \min(4096,k))$}
\STATE \keystep{order the seeds by unique coverage, the number of sets only that seed covers, and take the $L$ lowest}
\FOR{each such seed $s$}
\STATE remove $s$; \ $s'\leftarrow$ the non-seed with the largest marginal coverage
\STATE \keystep{keep $s'$ only if its marginal coverage exceeds that of $s$ by more than $\max(1,\lfloor0.005\,r\rfloor)$ sets, otherwise restore $s$}
\ENDFOR
\STATE \textbf{if} no seed changed \textbf{then} stop
\ENDFOR
\RETURN seed every $v\in\mathbb{A}$ at $t=0$
\end{algorithmic}
\end{algorithm}

\paragraph{Results.} At the $10\%$ budget on Digg, the program activates $56.5\%$ of the nodes, against $51.8\%$ for OPIM and SubSIM, and the same program gives all four budget cells ($30.4\%$, $45.8\%$, $56.5\%$ and $69.4\%$). Replayed unchanged on Network Science and NetHEPT, it matches the programs designed on those networks to within $0.1$ points (Table~\ref{tab:transfer}).

\paragraph{Excerpt.} The two exchange sweeps:
\codeblockstart
\codekey{margin~=~max(1,~int(0.005~*~repetitions))}
\codeplain{\char35{}~Preserve~the~parent's~full~sweep~and~its~exact~acceptance~threshold.}
\codeplain{\char35{}~Recompute~removal~costs~once~to~revisit~opportunities~created~later.}
\codeplain{for~sweep\codeus{}limit~in~(k,~min(4096,~k)):}
\codeplain{~~~~losses~=~np.empty(k,~dtype=np.int64)}
\codeplain{~~~~for~i,~node~in~enumerate(seeds):}
\codeplain{~~~~~~~~rr~=~cover\codeus{}ids[cover\codeus{}ptr[node]:cover\codeus{}ptr[node~+~1]]}
\codekey{~~~~~~~~losses[i]~=~np.count\codeus{}nonzero(multiplicity[rr]~==~1)}
\codekey{~~~~reconsider~=~np.argsort(losses,~kind="stable")[:sweep\codeus{}limit]}
\codeplain{~~~~changes~=~0}
\codeplain{~~~~for~position~in~reconsider:}
\codeplain{~~~~~~~~old~=~seeds[int(position)]}
\codeplain{~~~~~~~~rr~=~cover\codeus{}ids[cover\codeus{}ptr[old]:cover\codeus{}ptr[old~+~1]]}
\codeplain{~~~~~~~~multiplicity[rr]~-=~1}
\codeplain{~~~~~~~~exposed~=~rr[multiplicity[rr]~==~0]}
\codeplain{~~~~~~~~self.\codeus{}update\codeus{}scores(exposed,~1,~scores,~rr\codeus{}ptr,~members,~n)}
\codeplain{~~~~~~~~selected[old]~=~False}
\codeplain{~~~~~~~~candidates~=~np.where(selected[priority],~-1,~scores[priority])}
\codeplain{~~~~~~~~new~=~int(priority[np.argmax(candidates)])}
\codekey{~~~~~~~~if~scores[new]~<=~scores[old]~+~margin:}
\codekey{~~~~~~~~~~~~new~=~old}
\codeplain{~~~~~~~~rr~=~cover\codeus{}ids[cover\codeus{}ptr[new]:cover\codeus{}ptr[new~+~1]]}
\codeplain{~~~~~~~~fresh~=~rr[multiplicity[rr]~==~0]}
\codeplain{~~~~~~~~self.\codeus{}update\codeus{}scores(fresh,~-1,~scores,~rr\codeus{}ptr,~members,~n)}
\codeplain{~~~~~~~~multiplicity[rr]~+=~1}
\codeplain{~~~~~~~~selected[new]~=~True}
\codeplain{~~~~~~~~seeds[int(position)]~=~new}
\codeplain{~~~~~~~~changes~+=~int(new~!=~old)}
\codeplain{~~~~if~changes~==~0:}
\codeplain{~~~~~~~~break}
\codeblockend

\subsection{Source Localization: Cavity Consistency Optimization with Component Repair}
\label{app:example-sl}

\paragraph{Description.} Most localization methods in the library score each infected node by a centrality or a propagated label and return the top $k$. The program instead fits a source set to the observation through \textbf{a differentiable forward model}: $F(\bm{x})$ runs non-backtracking cavity messages for ten steps from source probabilities $\bm{x}$ and returns each node's infection probability, and the loss is the squared error between $F(\bm{x})$ and the observation. Starting from the LPSI answer, it \textbf{relaxes the source indicator to fractions that sum to $k$, descending the loss with an integrality penalty that grows over three stages}. The binary set is then improved by an \textbf{exchange search in two branches, ranked by gradient and curvature, whose disagreements are recombined}, and a \textbf{component repair} step targets observed infections the current sources leave unexplained.

\begin{algorithm}[H]
\caption{Cavity consistency optimization with component repair, designed for source localization on Power Grid under IC.}
\label{alg:example-sl}
\begin{algorithmic}[1]
\REQUIRE network $\mathcal{G}$ with transmission probabilities $p_{uv}$, observed infection indicator $\bm{y}$, number of sources $k$
\STATE $\mathbb{E}\leftarrow\{v:y_v=1\}$; \ $\bm{x}\leftarrow0.8\,\tfrac{k}{|\mathbb{E}|}\bm{1}_{\mathbb{E}}+0.2\,\bm{1}_{\mathrm{LPSI}}$ \COMMENT{start from LPSI's $k$ sources}
\STATE \keystep{$F(\bm{x})$: infection probabilities after ten steps of non-backtracking cavity messages from source probabilities $\bm{x}$; \ $L(\bm{x})=\lVert F(\bm{x})-\bm{y}\rVert^{2}$}
\FOR{$\lambda\in(0,\ 0.04,\ 0.15)$}
\STATE \keystep{projected gradient descent on $L(\bm{x})+\lambda\sum_v x_v(1-x_v)$ over $\{\bm{x}\in[0,1]^{\mathbb{E}}:\sum_v x_v=k\}$, with the gradient from an adjoint pass through $F$ and a backtracking line search}
\STATE every seven steps, round $\bm{x}$ to its $k$ largest entries and keep the best binary set $\bm{b}$
\ENDFOR
\STATE \keystep{exchange search in two branches, adjacency-aware and global: swap $1$ to $32$ sources ranked by gradient plus a stochastic curvature estimate, accepting only if $L$ decreases; recombine the branches' disagreements}
\STATE \keystep{component repair: group the observed-infected nodes that $F(\bm{b})$ predicts uninfected into connected components, and move low-value sources to one representative per component, accepting only if $L$ decreases}
\STATE a short adjacency-aware exchange search
\RETURN the $k$ sources in $\bm{b}$
\end{algorithmic}
\end{algorithm}

\paragraph{Results.} On Power Grid the recovered sources reach a consistency of $-0.080$, against $-0.106$ for LPSI, the strongest baseline (Table~\ref{tab:results_full}). Replayed unchanged, the program reaches $-0.138$ on Cora-ML and $-0.148$ on Deezer, against $-0.173$ and $-0.176$ for LPSI, and within $0.001$ of the programs designed on those networks.

\paragraph{Excerpt.} The relaxation with the growing integrality penalty:
\codeblockstart
\codeplain{seeds~=~0.8~*~uniform~+~0.2~*~best}
\codeplain{step~=~0.35}
\codekey{for~penalty~in~(0.0,~0.04,~0.15):}
\codeplain{~~~~for~iteration~in~range(28):}
\codekey{~~~~~~~~predicted,~context~=~self.\codeus{}forward(seeds,~trace=True)}
\codeplain{~~~~~~~~residual~=~predicted~-~y}
\codekey{~~~~~~~~objective~=~float(np.dot(residual,~residual))~+~penalty~*~float(}
\codekey{~~~~~~~~~~~~np.sum(seeds~*~(1.0~-~seeds))}
\codekey{~~~~~~~~)}
\codekey{~~~~~~~~gradient~=~self.\codeus{}adjoint(seeds,~context,~2.0~*~residual)}
\codekey{~~~~~~~~gradient~+=~penalty~*~(1.0~-~2.0~*~seeds)}
\codeplain{~~~~~~~~accepted,~trial\codeus{}step~=~False,~step}
\codeplain{~~~~~~~~for~\codeus{}~in~range(10):}
\codeplain{~~~~~~~~~~~~proposal~=~np.zeros(self.n)}
\codekey{~~~~~~~~~~~~proposal[eligible]~=~self.\codeus{}project(}
\codekey{~~~~~~~~~~~~~~~~seeds[eligible]~-~trial\codeus{}step~*~gradient[eligible],~count}
\codekey{~~~~~~~~~~~~)}
\codeplain{~~~~~~~~~~~~value~=~self.\codeus{}loss(proposal,~y)~+~penalty~*~float(}
\codeplain{~~~~~~~~~~~~~~~~np.sum(proposal~*~(1.0~-~proposal))}
\codeplain{~~~~~~~~~~~~)}
\codeplain{~~~~~~~~~~~~directional~=~float(np.dot(gradient,~proposal~-~seeds))}
\codeplain{~~~~~~~~~~~~if~value~<=~objective~+~1e-4~*~directional~+~1e-9:}
\codeplain{~~~~~~~~~~~~~~~~seeds~=~proposal}
\codeplain{~~~~~~~~~~~~~~~~step~=~min(trial\codeus{}step~*~1.2,~3.0)}
\codeplain{~~~~~~~~~~~~~~~~accepted~=~True}
\codeplain{~~~~~~~~~~~~~~~~break}
\codeplain{~~~~~~~~~~~~trial\codeus{}step~*=~0.5}
\codeplain{~~~~~~~~if~not~accepted:}
\codeplain{~~~~~~~~~~~~step~=~max(trial\codeus{}step,~1e-5)}
\codeplain{~~~~~~~~if~iteration~\char37{}~7~==~6~or~not~accepted:}
\codeplain{~~~~~~~~~~~~rounded~=~self.\codeus{}binary(seeds,~eligible,~count,~tie)}
\codeplain{~~~~~~~~~~~~loss~=~self.\codeus{}loss(rounded,~y)}
\codeplain{~~~~~~~~~~~~if~loss~<~best\codeus{}loss:}
\codeplain{~~~~~~~~~~~~~~~~best,~best\codeus{}loss~=~rounded,~loss}
\codeplain{~~~~~~~~if~not~accepted:}
\codeplain{~~~~~~~~~~~~break}
\codeblockend

\subsection{Influence Blocking: Differentiable Competitive-Cavity Blocking}
\label{app:example-ib}

\paragraph{Description.} RPS and the other sampling-based blockers in the library (Table~\ref{tab:library_groups}) sample prevention sets and cover them greedily, and the first candidate of this search did the same. The returned program replaces sampling with a \textbf{competitive mean-field recursion} that tracks, for every arc, the probability that it has not yet transmitted, and lets the counter cascade take a node first when both cascades reach it in the same step, the positive-dominance tie-break of our IC runs (Appendix~\ref{app:tasks}). Because the recursion is differentiable, the program \textbf{computes the gradient of the rumor size with respect to every counter-seed probability in one backward pass} and chooses all $k$ seeds jointly by \textbf{projected gradient descent on a relaxation, with an integrality penalty that grows over three stages}. Taking the $k$ largest entries of the relaxed solution can lose value, so it also rounds by \textbf{moving mass between pairs of fractional entries in whichever direction lowers the rumor size}, from two starting points. Last, a \textbf{swap repair recomputes the gradient after vacating a seed's seat}. A replacement can look redundant while the outgoing seed still protects its route, so the gradient taken with the seat empty ranks candidates that the original gradient misses.

\begin{algorithm}[H]
\caption{Differentiable competitive-cavity blocking, designed for influence blocking on Gnutella24 under IC.}
\label{alg:example-ib}
\begin{algorithmic}[1]
\REQUIRE network $\mathcal{G}$ with transmission probabilities $p_{uv}$, rumor sources $\mathbb{S}_N$, $k$ counter-seeds, horizon $H$
\STATE \keystep{$J(\bm{x})$: expected rumor size after $H$ steps of a competitive mean-field recursion in which the counter cascade, seeded with probabilities $\bm{x}$, takes a node first when both cascades reach it in the same step}
\STATE \keystep{$\nabla J(\bm{x})$ from one adjoint pass back through the $H$ steps}
\STATE $\mathbb{E}\leftarrow\mathbb{V}\setminus\mathbb{S}_N$; \ $x_v\leftarrow k/|\mathbb{E}|$ for $v\in\mathbb{E}$
\FOR{$(\lambda,T)\in\bigl((0,44),\ (0.4,12),\ (1.5,12)\bigr)$}
\STATE \keystep{$T$ steps of projected gradient descent on $J(\bm{x})+\lambda\,(k-\lVert\bm{x}\rVert^{2})$ over $\{\bm{x}\in[0,1]^{\mathbb{E}}:\sum_v x_v=k\}$ with a backtracking line search}
\STATE periodically round $\bm{x}$ to its $k$ largest entries and keep the best binary set $\bm{b}$
\ENDFOR
\FOR{the relaxed solutions after the first and the last stage}
\STATE \keystep{restrict to the $3k$ largest entries; move mass between the largest and the smallest fractional entry in whichever direction lowers $J$, until the solution is binary; keep it if it beats $\bm{b}$}
\ENDFOR
\FOR{two passes}
\STATE \keystep{for each of the three seeds with the least benefit under $\nabla J$: vacate its seat, recompute $\nabla J$ with the seat empty, and try the six best additions under each of the two gradients}
\STATE apply the best swap that lowers $J$; stop if none does
\ENDFOR
\RETURN seed every $v$ with $b_v=1$ in the counter cascade at $t=0$
\end{algorithmic}
\end{algorithm}

\paragraph{Results.} With $k=30$ counter-seeds on Gnutella24 the rumor reaches $1{,}113.6$ nodes, against $1{,}151.5$ for RPS, the strongest baseline. The same program gives all five budget cells and is ahead of RPS at every budget (Table~\ref{tab:results_full}). Replayed unchanged on Email-EU, it is also ahead of RPS at every budget, with rumor sizes from $59.4$ nodes at $k=10$ to $30.8$ at $k=50$, against $60.2$ to $31.9$.

\paragraph{Excerpt.} The competitive recursion, in which the counter cascade claims a node before the rumor; the lines that record values for the backward pass are omitted:
\codeblockstart
\codeplain{for~\codeus{}~in~range(self.steps):}
\codekey{~~~~middle~=~np.maximum(}
\codekey{~~~~~~~~theta~-~self.weight~*~positive\codeus{}front[self.tail],~eps}
\codekey{~~~~)}
\codekey{~~~~following~=~np.maximum(}
\codekey{~~~~~~~~middle~-~self.weight~*~negative\codeus{}front[self.tail],~eps}
\codekey{~~~~)}
\codeplain{~~~~middle\codeus{}product~=~self.\codeus{}products(middle)}
\codeplain{~~~~following\codeus{}product~=~self.\codeus{}products(following)}
\codeplain{~~~~new\codeus{}positive~=~susceptible\codeus{}initial~*~(}
\codeplain{~~~~~~~~old\codeus{}product~-~middle\codeus{}product}
\codeplain{~~~~)}
\codeplain{~~~~new\codeus{}negative~=~susceptible\codeus{}initial~*~(}
\codeplain{~~~~~~~~middle\codeus{}product~-~following\codeus{}product}
\codeplain{~~~~)}
\codekey{~~~~objective~+=~float(new\codeus{}negative.sum())}
\codeplain{~~~~theta~=~following}
\codeplain{~~~~old\codeus{}product~=~following\codeus{}product}
\codeplain{~~~~positive\codeus{}front~=~new\codeus{}positive}
\codeplain{~~~~negative\codeus{}front~=~new\codeus{}negative}
\codeblockend